\documentclass[12pt]{article}
\usepackage{amsmath}
\usepackage{amssymb}
\usepackage{graphicx,psfrag,epsf}
\usepackage{enumerate}

\usepackage{url} % not crucial - just used below for the URL 
\usepackage{float}
\usepackage[table,xcdraw,dvipsnames]{xcolor}

\newcommand\cb{\color{blue}}
\newcommand\ck{\color{black}}

\usepackage{cite}  
\usepackage{multirow}
\usepackage{lipsum}
\usepackage{amsfonts}
\usepackage{cite}      % Written by Donald Arseneau
                  
\usepackage{graphicx}   % Written by David Carlisle and Sebastian Rahtz

\usepackage{amsthm}
\RequirePackage[colorlinks,citecolor=blue,urlcolor=blue]{hyperref}
\usepackage{amssymb,calc}
\usepackage{thmtools}
\usepackage{mathrsfs}
\usepackage{mathtools}
\usepackage{bm}
\usepackage{bbm}
\usepackage{caption}
\usepackage{subcaption}

\usepackage{enumerate}
\usepackage{rotating}
\RequirePackage{cleveref}
\usepackage{hyphenat}
\usepackage[numbers]{natbib}
\usepackage{caption,subcaption}
\usepackage{tikz}
\usepackage{pgfplots}

\usepackage[T1]{fontenc}
\usepackage{textcomp} % for \textquotesingle
\usepackage{lmodern}

\usepackage{algorithm}
\usepackage[noend]{algpseudocode}
\makeatletter
\renewcommand{\ALG@name}{Algorithm}
\makeatother

\usepackage{titletoc}

\pgfplotsset{width=10cm}

\tikzset{declare function={gamma(\x)=sqrt(2*pi)*\x^(\x-0.5)*exp(-\x)*exp(1/(12*\x));}}
\tikzset{declare function={tpdf(\x,\nu)=gamma(0.5*(\nu+1))/(sqrt(pi*\nu)*gamma(\nu/2))*(1+\x^2/\nu)^(-(\nu+1)/2);}}
\tikzset{declare function={invgampdf(\x,\a,\b)=(\b/\x)^\a/\x/gamma(\a)*exp(-\b/\x);}}

\newcommand{\nhphantom}[1]{\ifmmode\settowidth{\dimen0}{$#1$}\else\settowidth{\dimen0}{#1}\fi\hspace*{-\dimen0}}

\usetikzlibrary{patterns}
\tikzset{
	hatch distance/.store in=\hatchdistance,
	hatch distance=5pt,
	hatch thickness/.store in=\hatchthickness,
	hatch thickness=0.5pt,
}

\newcommand{\wh}[1]{\widehat{#1}}

\definecolor{pink}{rgb}{0.9, 0.17, 0.31}

\def\C {\,|\:}

\newcommand\E{\mathbb E}

\newcommand\R{\mathbb R}
\renewcommand\b{\bm{\beta}}

\newcommand{\wt}[1]{\widetilde{#1}}

\newtheorem{assumption}{Assumption}
\newtheorem{lemma}{Lemma}

\renewcommand{\nhphantom}[1]{\ifmmode\settowidth{\dimen0}{$#1$}\else\settowidth{\dimen0}{#1}\fi\hspace*{-\dimen0}}

\numberwithin{equation}{section}

\newtheorem{thm}{Theorem}[section]

\newtheorem{exa}{Example}
\newtheorem{remark}{Remark}

\crefname{thm}{Theorem}{Theorems}
\crefname{prop}{Proposition}{Propositions}
\crefname{lem}{Lemma}{Lemmas}
\crefname{coro}{Corollary}{Corollaries}
\crefname{add}{Addendum}{Addendums}
\crefname{asm}{Assumption}{Assumptions}
\crefname{alg}{Algorithm}{Algorithms}
\crefname{proc}{Procedure}{Procedures}
\crefname{exe}{Exercise}{Exercises}
\crefname{exa}{Example}{Examples}
\crefname{prob}{Problem}{Problems}
\crefname{section}{Section}{Sections}
\crefname{subsection}{Section}{Sections}
\crefname{appendix}{Appendix}{Appendices}

\def\argmax{\mathop{\arg\max}}
\def\argmin{\mathop{\arg\min}}

\allowdisplaybreaks[3]

\begin{document}

\def\spacingset#1{\renewcommand{\baselinestretch}%
{#1}\small\normalsize} \spacingset{1.1}

	\title{\sf  Tree Bandits for Generative Bayes}

\author{Sean O'Hagan, Jungeum Kim, and Veronika Ro\v{c}kov\'{a}}

%\author{Jungeum Kim\thanks{Jungeum Kim is a postdoctoral researcher at the Booth School of Business of the University of Chicago (\texttt{jungeum.kim@chicagobooth.edu})}\, }
%\author{Veronika Ro{\v{c}}kov{\'a}\thanks{Veronika Ro{\v{c}}kov{\'a} is  Professor in Econometrics and Statistics and James S. Kemper Faculty Scholar at the Booth School of Business of the University of Chicago (\texttt{veronika.rockova@chicagobooth.edu}). The author gratefully acknowledges support from the James S. Kemper Foundation Faculty Research Fund at the Booth School of Business as well as the National Science Foundation (DMS:1944740). }}
%\author[1]{Author C\thanks{C.C@university.edu}}
%\author[2]{Author E\thanks{E.E@university.edu}}

	\maketitle
%} \fi

\bigskip
\begin{abstract}
In generative models with  obscured  likelihood, Approximate Bayesian Computation (ABC) is often the tool of last resort for inference. 
However, ABC  demands many prior parameter trials to keep only a small fraction that passes an acceptance test.
To accelerate ABC rejection sampling, this paper develops a self-aware framework that learns from past trials and errors. We apply recursive partitioning classifiers on the ABC lookup table  to sequentially refine high-likelihood regions into boxes. Each box is regarded as an arm in a binary bandit problem treating ABC acceptance as a reward. Each arm has a proclivity for being chosen for the next ABC evaluation, depending on the prior distribution and past rejections. The method places more splits in those areas where the likelihood resides, shying away from low-probability regions destined for ABC rejections. We provide two versions: (1) ABC-Tree for posterior sampling, and (2) ABC-MAP for maximum a posteriori estimation. We demonstrate accurate ABC approximability at much lower simulation cost. We justify the use of  our tree-based bandit algorithms with nearly optimal regret bounds.
{Finally, we  successfully apply our approach to the problem of masked image classification using deep generative models.}
\end{abstract}

\noindent%

\spacingset{1.45} % DON'T change the spacing!

\section{Introduction}
Many posterior sampling strategies (such as the Metropolis-Hasting algorithm or Approximate Bayesian Computation (ABC)) randomly propose a parameter value (trial) and then reject it (error) if not supported by data. This accept/reject mechanism may require exceedingly many trials to obtain only a few accepted posterior samples. These shortcomings can be overcome by allowing  Bayesian algorithms to learn from their trials and errors.

In this work, we develop active learning versions of ABC for posterior sampling, as well well as maximum a-posteriori (MAP) estimation,  focusing on {\em generative} (or simulation-based) models, where  users can access model likelihood only indirectly through simulation. These models are commonly employed in population ecology \citep{beaumont2010approximate}, astrophysics \citep{schafer2012likelihood}, and finance \citep{hore2010bayesian}, among other fields. ABC starts off by  simulating a reference table consisting of pairs of parameters and fake data sets simulated from the prior and the likelihood, respectively.
 The default ABC rejection sampling  then proceeds by filtering the reference table where the only kept pairs are those for which the   observed and fake data  (or summary statistics thereof) are sufficiently close.   The accepted parameters form a set of independent samples from an approximate posterior, and thus provide means for posterior inference in the absence of a tractable likelihood function. However, ABC can be very inefficient even in simple problems \citep{meeds2014gps}. Rejection sampling algorithms require a large number of simulations where most of the computational energy will be wasted on simulating implausible parameters only to discard them during the acceptance test. This has motivated query-efficient ABC techniques which intelligently decide where to propose next \citep{marjoram2003markov, sisson2007sequential, beaumont2009adaptive, gutmann2016bayesian, jarvenpaa2018gaussian}.  

Algorithms based on MCMC and sequential MCMC have been used to improve the efficiency of ABC as compared to the basic rejection sampler.
For example, \citet{marjoram2003markov} developed a variant of the classical Metropolis-Hastings algorithm that proposes new ABC samples through perturbation of the most recent accepted state. \citet{sisson2007sequential} proposed a sequential Monte Carlo ABC (SMC-ABC) variant by providing better proposals based on past iterations. \citet{beaumont2009adaptive} introduce a population Monte Carlo correction to the algorithm based on  importance sampling arguments. These algorithms have enjoyed widespread usage and can accurately reconstruct posteriors in simulator-based models with much fewer simulations required when compared to rejection ABC, which always proposes parameters from the prior. \citet{alsing2018optimal} suggest applying an optimality criterion to SMC-ABC by computing a particle-based estimate of the posterior and subsequently computing the optimal proposal density. However, sampling from this optimal density is nontrivial and requires the use of MCMC or other techniques.

Bayesian optimization and surrogate models (Gaussian processes) have been proposed for active learning from past unsuccessful ABC attempts
 \citep{gutmann2016bayesian,jarvenpaa2018gaussian}.
 % employ a Gaussian process (GP) to approximate the distribution of the discrepancy (between summary statistics) as a function of parameters, and Bayesian optimization to propose new parameters. 
%However, these approaches still require summary statistics,  the  distance function, and  a homoscedastic GP {\cb Don't we also need summary statistics and distance function?}. 
%Active learning has been incorporated in Bayesian simulation before. 
For example, \citet{wilkinson2014accelerating} modeled the uncertainty in the likelihood to rule out regions with negligible posterior. 
%\citet{rasmussenBayesianMonteCarlo2002} used GP regression to accelerate hybrid Monte Carlo. \citet{osborneActiveLearningModel2012} developed an active learning scheme to select evaluations to estimate integrals.  
More recently, \citet{jarvenpaa2019efficient} propose an uncertainty measure for the ABC posterior (due to the lack of simulations to estimate ABC accurately) and define a loss function that gauges such uncertainty. %They then propose to select the next evaluation location to minimize the expected loss. 
The premise of such active learning methods is constructing a utility function based on the posterior conditioned on the queries so far and then maximize the utility to determine where to draw next. However, maximizing the typically multimodal utility can be a hard problem and this approach will only be beneficial when determining the next experiment is not as costly. \citet{kandasamy2017query} propose another query efficient method for estimating posteriors when the likelihood is expensive to evaluate. 
%They adopt a Bayesian active regression approach on the log-likelihood using the samples it has already computed. They propose two myopic query strategies on the uncertainty regression.
% Bayesian optimization enables

One way or another, building query functions opens up  tremendous  possibilities for computational gains  by retaining (rejected) draws and by using regression techniques to build  proposal distributions. We do exploit flexible regression techniques here as well. One of the key innovation in ABC has been to use a post-sampling regression adjustment \citep{beaumont2002approximate}, allowing larger rejection tolerance values and thereby shrinking computation time to realistic orders of magnitude. Here, we use flexible regression techniques for efficient proposals as opposed to post-sampling adjustment (debiasing). 
Conditional density estimation techniques have been proposed to approximate the likelihood in the ABC setting. \citet{fan2013approximate} construct an approximation based on samples $\theta$ (not necessarily from the prior) and the corresponding summaries simulated form the intractable likelihood. The concept of approximating intractable likelihood using simulated data goes back to \citet{diggle1984monte}. \citet{lueckmann2019likelihood} propose ABC which uses probabilistic neural emulator networks to learn synthetic likelihoods on simulated data. Simulations are chosen adaptively using an acquisition function which takes into account uncertainty about either the posterior distribution of interest or the parameters of the emulator.

Our approach is more directly connected to rejection ABC rather than the more recent sequential neural likelihood approaches  \citep{lueckmann2019likelihood, papamakarios2019sequential}. We focus on improving the efficacy of ABC rejection sampling which is still widely used and very easy to implement by even the least computationally savvy practitioners. We view the proposal of the subsequent parameter value at which to collect new data as an active (reinforcement) learning problem, where ABC acceptance is regarded as a binary bandit reward. Being able to learn from past mistakes is the key to steering the search towards high-posterior regions where ABC is more likely to accept.
We consider two major inferential objectives: (1) building a fast approximation to ABC posterior which requires much smaller reference tables, (2) building a stochastic optimization method for rapid \textit{maximum a-posteriori} estimation. 
%Maximizing intractable likelihoods via simulation has also been investigated. \citet{rubioSimpleApproachMaximum2013} propose a simple method of maximizing a kernel density estimate of an ABC sample in order to maximize the posterior in the presence of an intractable likelihood. Bayesian Optimization can also be easily applied to likelihood-free MAP estimation by using a surrogate model for the discrepancy or the acceptance rate. However, choosing such a model may be difficult in likeliho. 
%These manifest as assumptions on the noise distribution of the objective function, which is dependent on the actual probabilistic process being simulated, as well as the choice of discrepancy function.
Both  our approaches (for sampling and MAP estimation) rest on recursive partitioning by fitting a classification tree (or a forest) to ABC acceptances and rejections in order to refine the proposal for the next step. 
Each tree partitions the parameter space into boxes and we regard each box as an arm in a binary bandit problem. This framework allows us to be selective about which of the boxes we choose next, prioritizing those which have provided more rewards in the past. Our algorithm has two loops: (1) an outer loop refines the partition, and (2) an inner loop refines the ABC proposal for a given partition by playing a bandit game.
%Our approaches are related to \cite{bubeckXArmedBandits2010} who partition the measurable space of arms with a tree structure for bandit optimization. A variant of the upper confidence bound (UCB) algorithm, this approach is deterministic after conditioning on the observed rewards, while our approach embraces randomness through a Thompson sampling style approach. Our work bridges the gap between bandit optimization and the likelihood free MAP problem.

Modeling the ABC proposal (or the ABC posterior) via  recursive partitioning of the parameter space into hyper-rectangles  carries some advantages when compared to other similar approaches like Gaussian processes. In particular, the resulting model for the ABC acceptance rate may be tractably normalized. This property is highly amenable to sequentially updating of the ABC proposal, making it tractable even in parameter spaces of reasonably large dimension. {Trees have also been successful at modeling functions with varying local smoothness \citep{van2017bayesian}, and can scale well in higher dimensions when compared to Gaussian processes or kernel density estimation.
If not for ABC posterior approximation, our method can be applied to generate high-quality training data for generative Bayes algorithms \citep{wang2022adversarial, polson2023generative, li2020deep} that contain observations that are very similar to the observed data. 
% Despite restricting our proposals to a family of {\cp prior-weighted} step-functions, we can easily obtain smooth posterior approximations from an importance weighted kernel density estimate on accepted parameters.} %which may not approximate smooth likelihoods as well as GPs. However, as opposed to {\cp $\leftarrow$ Sean, \cb this sentence looks contradicting to the previous sentence.}
%We justify our algorithms theoretically in a special case, extend them to more general parameter spaces using an adaptive partitioning scheme

%In this paper, we propose query-efficient sample-based ABC strategies for posterior sampling and MAP estimation, where inference does not directly rely on a model-based estimate of the likelihood. The overarching theme of these strategies is to adaptively partition the parameter space and iteratively focus on regions that are more valuable for the task at hand.
{The paper is structured as follows. In Section \ref{sec:abc}, we briefly review ABC and reinforcement learning in the ABC context and provide a high-level description of our approach. The ABC-Tree is introduced in Section \ref{sec:abctree}, and the MAP-Tree in Section \ref{sec:MAPTREE}. The numerical study in Section \ref{sec:experiemnt} demonstrates the efficiency of our methods on both simulation and real datasets. The paper concludes with Section \ref{sec:concluding}.} 

\vspace{-0.5cm}
\section{Approximate Bayesian Computation}\label{sec:abc}
%We assume that the observed data $\bm X=(X_1,\dots,X_n)'\in\mathcal X$ arrived from a simulator-based model governed by parameters $\bm\theta_0\in\R^d$. Such a model may be provided by a probabilistic program that   implicitly defines a likelihood $p(\bm X\mid\bm\theta)$ which  we cannot evaluate. We have access to the likelihood only through a simulation process $\bm\theta\mapsto \widetilde X_{\bm\theta}$, where $\widetilde X_{\bm\theta}$ has a density $p(\bm X\mid\bm\theta)$. Prior beliefs about $\bm\theta_0$ are expressed through the prior $\pi(\boldsymbol\theta)$. For a given choice of an acceptance threshold $\varepsilon>0$, a summary statistic $s:\mathcal{X}\to\mathcal{S}$, and a distance metric $d:\mathcal{S}\times\mathcal{S}\to\mathbb{R}_{\geq 0}$, we define the $\varepsilon$-likelihood as the convolution of the intractable likelihood and the acceptance kernel

We assume that the observed data $\bm X=(X_1,\dots,X_n)'\in\mathcal X$ arrived from a simulator-based model governed by parameters $\bm\theta_0\in\Theta\subset \R^d$. Such a model may be provided by a probabilistic program that   implicitly defines a likelihood $p_{\bm\theta}(\bm X)$ which  we cannot evaluate. We have access to the likelihood only through simulation $\bm\theta\mapsto \widetilde {\bm X}_{\bm\theta}$, where $\widetilde {\bm X}_{\bm\theta}$ has a density $p_{\bm\theta}(\cdot)$. Prior beliefs about $\bm\theta_0$ are expressed through the prior $\pi(\boldsymbol\theta)$. For a given choice of an acceptance threshold $\varepsilon>0$, a summary statistic $s:\mathcal{X}\to\mathcal{S}$, and a distance metric $d:\mathcal{S}\times\mathcal{S}\to\mathbb{R}_{\geq 0}$, we define the $\varepsilon$-likelihood as the convolution of the intractable likelihood and the acceptance kernel
\begin{equation}\label{eq:conv_likelihood}
L_\varepsilon(\bm X\mid\bm\theta) \propto \int_\mathcal{X} \mathbb{I}\{d(s(\wt{\bm{X}}_{{\bm\theta}}),s(\bm X))<\varepsilon\}p_{\bm\theta}(\wt{\bm{X}}_{{\bm\theta}})\,\mathrm{d}\wt{\bm{X}}_{{\bm\theta}}\,.
\end{equation}
%{where $C  = \int_{\mathcal{X}}\int_\mathcal{X} \mathbb{I}\{d(s(\wt{\bm{X}}),s(\bm X))<\varepsilon\}p_{\bm\theta}(\wt{\bm{X}})\,\mathrm{d}\wt{\bm{X}}\mathrm{d}{\bm{X}}$.
% If $s(\bm X) = \bm X$, we can safely regard $C_{\bm\theta}$ as independent from $\bm\theta$ because then $C_{\bm\theta} =\int_\mathcal{X}  \int_{1\{d(\wt{\bm{X}},\bm X)<\varepsilon\}}\mathrm{d}\bm X\,p_{\bm\theta}(\wt{\bm{X}})\,\mathrm{d}\wt{\bm{X}}$, which is the volume of the $\varepsilon$-ball in $\mathcal{X}$. Otherwise, we can treat the data sampling process and likelihoods to be defined based on the sufficient statistic $\bm Z := S(\bm X)$, which is not an issue in the ABC and MAP contexts.} 
With such consideration, %we can treat $C_{\bm\theta}$ as a constant,% 
we define the $\varepsilon$-posterior  to be  proportional to the product of the prior and the $\varepsilon$-likelihood,
\begin{equation}\label{eq:abc_post}
p_\varepsilon(\bm\theta\mid\bm{X})\propto \pi(\bm\theta) L_\varepsilon(\bm{X}\mid\bm\theta)\,.
\end{equation}

Standard ABC involves picking a candidate parameter value $\wt{\bm\theta}\sim\pi(\boldsymbol\theta)$ and simulating data $\wt {\bm X}_{\wt {\bm\theta}}$. The prior guess $\wt{\bm\theta}$ is rewarded with acceptance when  $d(s(\wt {\bm X}_{\wt{\bm\theta}}),s(\bm X))<\varepsilon$. Repeating this procedure until $B$ parameter values have been accepted, one obtains a sample $\{\wt{\boldsymbol\theta}_i\}_{i=1}^B$. 
%This accept-reject mechanism emulates the prior weight multiplication $L_\varepsilon(\bm{X}\mid\bm\theta)$ in \eqref{eq:abc_post}. 
If $s$ is a sufficient statistic for $\boldsymbol\theta$, then $\{\wt{\boldsymbol\theta}_i\}_{i=1}^B$ is a set of i.i.d. samples from the true posterior in the limit as $\varepsilon\rightarrow 0$.

%Standard ABC involves sampling potential parameter values $\wt{\bm\theta}\sim\pi(\boldsymbol\theta)$ and simulating data  {\color{red} change the notation below according to above}$\wt {\bm X}_{\wt{\bm\theta}}$. We accept $\wt{\boldsymbol\theta}$ if $d(s(\mathbf{X}_{\tilde{\boldsymbol\theta}}),s(\mathbf{x}))<\varepsilon$. Repeating this procedure until $N$ parameter values have been accepted, one obtains a sample $\{\wt{\boldsymbol\theta}_i\}_{i=1}^N$. If $s$ is a sufficient statistic for $\boldsymbol\theta$, then $\{\wt{\boldsymbol\theta}_i\}_{i=1}^N$ is a set of i.i.d. samples from the true posterior in the limit as $\varepsilon\rightarrow 0$. 

In the absence of subjective prior information, the vast majority of parameter guesses $\wt{\bm\theta}$ will be implausible and have a negligible probability of being accepted. For this reason, it is common practice  \citep{sisson2007sequential, beaumont2009adaptive} in ABC to draw from a proposal distribution $q(\boldsymbol\theta)$ instead of the prior $\pi(\boldsymbol\theta)$, which ideally has more mass on plausible regions of the parameter space. After collecting samples, one can use importance sampling to correct for the fact that   $q$ has been used  instead of the prior, weighting each accepted $\wt{\boldsymbol\theta}$ by the density ratio $\pi(\wt{\boldsymbol\theta})/q(\wt{\boldsymbol\theta})$. Using an informative proposal distribution helps  increase the ABC acceptance rate, but introduces a trade-off in the importance weight variance, affecting the effective sample size.\footnote{Intuitively, proposing from a delta measure at the maximum likelihood parameter value would result in the highest ABC acceptance rate (see, \eqref{eq:conv_likelihood}), but would never yield a sample from the posterior distribution.} %{\cp Yes, because the likelihood is proportional to the acceptance rate, so proposing the parameter with maximal likelihood will have the largest acceptance rate}Cotcha!
Therefore, it is important to control such a trade-off through some optimality criteria. In sequential Monte Carlo methods \citep{sisson2007sequential, beaumont2009adaptive}, discussions of the optimality of proposals often involve choosing the bandwidth of a Gaussian kernel to perturb parameter values from the previous population. The $\mathrm{KL}$ divergence between the proposal density and the target is adopted by \citet{beaumont2009adaptive} to choose the kernel bandwidth in a Gaussian KDE proposal for sequential Monte Carlo. This is later refined in \citet{filippi2013optimality} who discussed maximization of the sum of the log acceptance ratio and the negative $\mathrm{KL}$ divergence between the proposal and target. \citet{alsing2018optimal} alternatively define an optimal proposal density for ABC in terms of the \emph{effective acceptance rate}. 

\vspace{-0.5cm}
\subsection{Reinforcement Learning for ABC}

Reinforcement learning (RL) is a branch of artificial intelligence that studies what actions to take so as to minimize regret in uncertain environments. Here, we investigate reinforcement learning strategies that sequentially process the ABC reference table and zoom onto those areas of the parameters space where the prior parameter guesses are more likely to be accepted.
%Reinforcement learning (RL) is a branch of artificial intelligence that studies what actions to take so as to minimize regret in uncertain environments. 

In a similar vein, sequential ABC sampling strategies have been proposed where, at each iteration, one decides the next evaluation location. For instance, \citep{jarvenpaa2019efficient} {model the uncertainty in the ABC likelihood in \eqref{eq:abc_post} by $L_\varepsilon(\bm X\mid\bm\theta)\propto\Phi\left(\frac{\varepsilon-f(\bm\theta)}{\sigma}\right)$, where $f$ is a Gaussian process regression function that models the distance metric $d(s(\wt{\bm{X}}_{\bm \theta}),s(\bm X))$}. They recompute the uncertainty in the ABC likelihood after each iteration and fit a model that predicts the next iteration output, given all data available at the time. The aim is to  choose the next evaluation location such that the expected loss, after having simulated the model at this location, is minimized.  Their approach is a one-step ahead Bayes optimal solution to a decision problem of minimizing the expected uncertainty. Similarly, the entropy search of \citet{hennig2012entropy} is designed for query-efficient global optimization and aims to find a parameter value that maximizes the objective function and to minimize the uncertainty related to this maximizer. 
\citet{lueckmann2019likelihood} use active learning to selectively acquire new samples using  local and global emulators.  %They consider an acquisition rule which targets the region of maximum variance in the predicted (unnormalized) posterior. 
%While these models can model complex likelihoods, inference in the emulator model can be computationally intensive, making sequential updating only feasible after large batches of simulator evaluations. 

Query-efficient ABC methods will only be beneficial when evaluating the query function (after each ABC step) is not too prohibitive.
Our approach is conceptually simple and computationally feasible. We focus on recursive partitioning schemes where the parameter space will be chopped into boxes depending on their probability of yielding a parameter value accepted by ABC. In direct analogy with pulling arms in a bandit problem, for generating ABC parameter guesses we prioritize boxes that have yielded fewer rejections in the past. In what follows, we survey Thompson Sampling (TS), which is the basis of our bandit-like ABC partitioning approach. %{\cb I don't see a literally overlapping part. I am not sure where to comment out.}

\vspace{-0.5cm}
\subsection{Thompson Sampling Revisited}
In the multi-armed bandit (MAB) problem, a slot-machine player needs to decide which of the $K$ arms  to pull. Each arm has a reward distribution  that is unknown to the gambler who can only learn about it by playing. In doing so, the player faces a well-known exploration/exploitation dilemma. A similar dilemma is also present in  ABC proposals, where we want to propose parameter values where the posterior density is thought to be high, to increase our acceptance rate, but must also explore enough to learn the entire posterior and control the importance sampling variance. 

Consider the Bernoulli bandit game and denote with $ I^{(t)}$ the random arm played at time $t$. Pulling the $i^{th}$ arm at time $t$ gives a random payout $Y^{(t)}\in \{0,1\}$ with $  P(Y^{(t)}=1|I^{(t)} = i)=\mu_i$. The mean payout $\mu_i$ of each arm is unknown. An agent must decide which of the $K$ arms to play at time $t$, given the outcome of the previous $t-1$ plays. A natural goal in the MAB game is to minimize the {\sl cumulative regret}, i.e. the amount of money one loses by not playing the optimal arm at each step. Denoting with $\mu^\star=\max\limits_{1\leq i\leq K}\mu_i$ the largest average reward, define a {cumulative} regret of agent $\mathcal{A}$ after $T$ plays by 
\begin{equation}\label{eq:thompson}
    R_T(\bm{\mu},\mathcal{A}) = T\mu^*-\E\left[\sum_{t=1}^TY^{(t)}\right],
\end{equation}
{where the expectation is over the agent's choices $I^{(t)}\mid \{I^{(1)},Y^{(1)},\ldots,I^{(t-1)},Y^{(t-1)}\}$ and the stochastic payouts $Y^{(t)}\mid I^{(t)}$.}
Thompson Sampling is a heuristic algorithm with a Bayesian flavor that achieves a logarithmic expected {cumulative} regret \citep{agrawal2017near} in the Bernoulli bandit problem. It \emph{explores} by {sampling a random instance $\eta_j \sim P(\mu_j\mid \{Y^{(m)}, I^{(m)}\}_{m=1}^{t-1}) $ from the posterior knowledge} on the mean reward $ \mu_j$ for each arm $j$. It \emph{exploits} by selecting the arm with the highest $\eta_j$ for the next play. The posterior distribution is updated to reflect the new random payout $Y$. {When the rewards are Bernoulli distributed, endowing the mean acceptance rates $\mu_j$ with Beta priors yields Beta posteriors on the acceptance rates $\mu_j$. This facilitates tractable sampling of $\eta_j$ from the posterior on each $\mu_j$ after every play.  

Another common objective in the multi-armed bandit problem is to simply learn which arm is optimal at the end of the game, as opposed to maximizing the total reward over the course of playing. Known as  best arm identification \citep{audibert2010best} (or simple regret minimization \cite{bubeck2009pure}), {the agent must make a choice $\hat{k}^{(T)}$ after $T$ plays that minimizes $P(\hat{k}^{(T)}\neq k^*)$,
%r_T = \mu^* - \mu_{J_T}\,,
where $k^*=\argmax_k \mu_k$ is the arm with the maximal mean reward.} \citet{russo2016simple} introduces Top-two Thompson Sampling, which oversamples nearly-optimal arms, trading off some exploitation of the best arm for greater certainty about the mean rewards of almost optimal arms. The procedures of Thompson Sampling and the top-two variant are detailed in Algorithm \ref{alg:ts} in Section \ref{sec:ts}. In our work, we will think of the ABC acceptance of a parameter proposal as a binary reward $Y$.

\iffalse
Thompson Sampling \cite{thompsonLikelihoodThatOne1933} is a heuristic for choosing the next action in uncertain environments, where the player chooses an action with probability proportional to the posterior probability that it is optimal. In many settings, it has been shown to be empirically powerful 
\cite{chapelleEmpiricalEvaluationThompson2011} \cite{russoLearningOptimizePosterior2014}, and it maintains strong theoretical guarantees as well \cite{agrawalFurtherOptimalRegret2012} \cite{grantThompsonSamplingSmootherthanLipschitz2020} \cite{russoInformationTheoreticAnalysisThompson2015}.
%Thompson Sample for Bernoulli bandits leverages independent Beta prior  distributions $\mathcal{B}(a_i,b_i)$ on each $\theta_i$ for $1\leq i\leq N$.  The beta prior is a natural conjugate choice after observing a binary reward. It is customary to initialize the algorithm with  $a_i=b_i=1$ for $1\leq i\leq N$ to signal that  none of the arms is a priori dominant  (\cite{thompson_likelihood_1933}). At time $t$, each arm has been played $N_i(t)=a_i(t)+b_i(t)$ times, where $a_i(t)$ is the number of successes and $b_i(t)$ is the number of failures. The algorithm updates the distribution  on $\theta_i$ as $\texttt{Beta}(a_i(t)+1,b_i(t)+1)$. The algorithm then samples $\hat{\theta}_i(t)$ from these posterior distributions and plays the arm with the highest  $\hat{\theta}_i(t)$.
 \cite{agrawal_analysis_2012} extended this algorithm to the general case where  rewards are not necessarily Bernoulli but general random variables on the interva $[0,1]$. 
 
\fi

{Bandit optimization algorithms have also been constructed over a continuous arm domain \citep{agrawal1995continuum, kleinberg2004nearly, bubeck2011x}. %, slivkins2011multi}. 
% \citet{kleinberg2004nearly} proposed an algorithm for continuum-armed bandits with a Lipschitz cost function and sub-Gaussian noise based on a modification of the UCB algorithm applied to a discretization of the arm space. 
Notably, hierarchical online optimization proposed by \citet{bubeck2011x} is an optimization algorithm on continuous spaces based on a tree-based hierarchical partitioning.  Our approach is based on discretization and sequential refinement of a step-function approximation to the posterior. Unlike \citet{bubeck2011x}, however, we rely on Thompson Sampling as opposed to an upper confidence bound approach.}

\vspace{-0.5cm}
\subsection{Our Approach:  ABC-Tree}
The central theme of our work is viewing the parameter space from the perspective of partitions. We render learning the optimal proposal distribution as learning (1) the optimal partition of the parameter space and (2) the optimal\footnote{Our chosen notion of optimality will be introduced in Section \ref{sec:online}.} proposal distribution on the partition. These two objectives are folded into a single algorithm through a nested structure, where the \emph{outer} loop achieves goal (1) via  recursive partitioning, and where the \emph{inner} loop  achieves goal (2) using bandits.

\begin{figure}[!t]
    \centering
    \includegraphics[width=\linewidth,height=7cm]{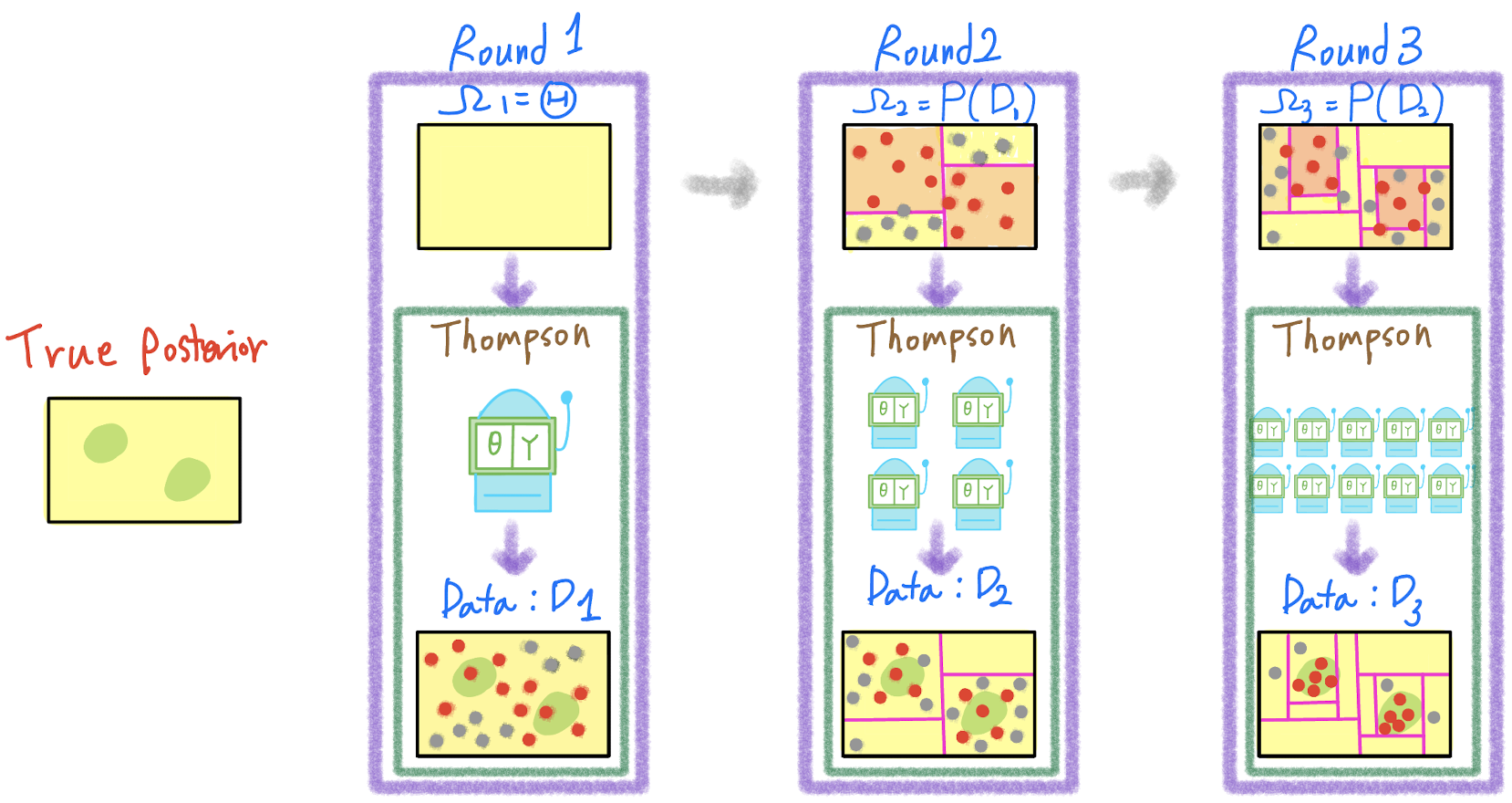}
    \caption{The nested structure of ABC-Tree. The outer loop actively learns the optimal partition and the inner loop actively learns the optimal proposal distribution  on the current partition. Red dots are accepted proposals.}
    \label{fig:abstract_algorithm}
\end{figure}

The inner loop (Algorithm \ref{alg:qe_abc} in Section \ref{sec:innter_loop}) is conditional on a given partition of the parameter space. Regarding each box as a bandit arm, the loop refines an ABC proposal by a Thompson Sampling strategy that  updates the posterior probability of obtaining a reward (ABC acceptance). Due to the discretization of the parameter space, proposing a parameter value {can be simplified}. First, we randomly select a box according to the current state of reward posteriors in a manner that very much resembles Thompson Sampling. Next, we choose a value from a prior restricted to the box. {When the prior is uniform, this process is extremely simple, and for general priors we can re-weight (Remark \ref{rmk:rm1})}. The inner loop continues until enough acceptances have been accumulated {or when the number of inner iterations has exceeded the inner loop budget.}

  \begin{figure}[!t]
    \centering
    \includegraphics[width=0.6\textwidth]{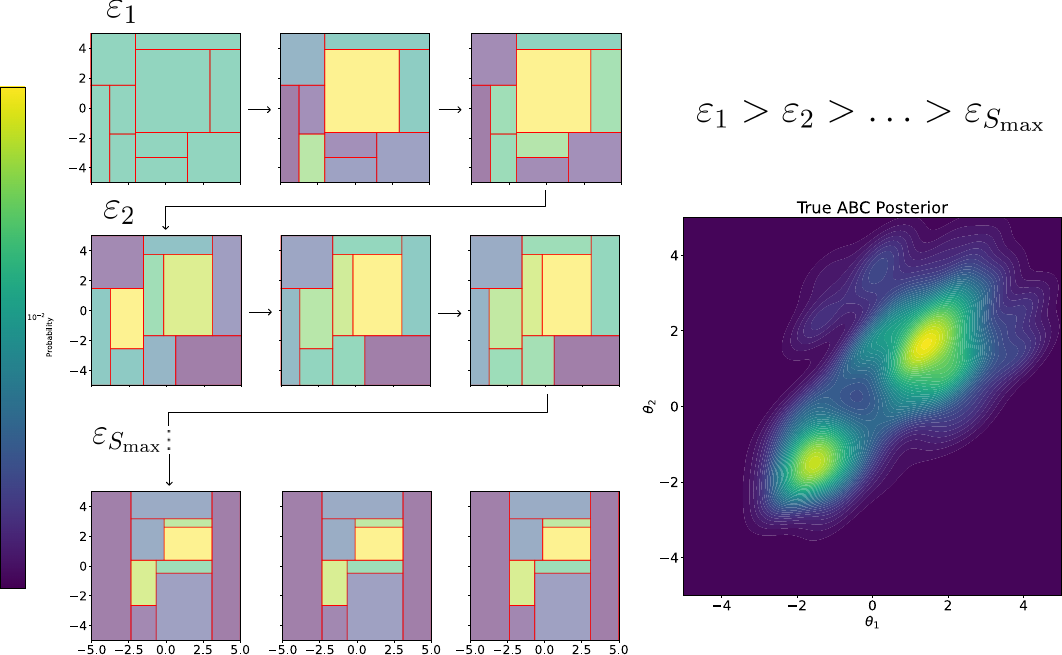}
    \caption{Evolution of the proposal distributions in Algorithm \ref{alg:seqtreeabc} (partitioning with CART). The algorithm decreases $\varepsilon$ and re-partitions $\Theta$. Within outer loops (each row), the partition and $\varepsilon$ stay fixed, while the weight for each box is updated after every inner loop iteration.}
    \label{fig:progression-ex}
\end{figure}

In the outer loop { (Algorithm \ref{alg:seqtreeabc} in Section \ref{sec:outer})} we refine the partition. This is accomplished by simply fitting a classification tree (forest) to the ABC acceptances obtained from past proposals. The acceptance is re-calculated for each outer loop with a decreasing ABC threshold $\varepsilon$. Therefore, the outer loop develops an adaptive partition with a finer resolution that wraps the  $\varepsilon$-posterior into boxes. Note that the number of arms (partition boxes) provided by the outer loop changes as the algorithm progresses, which has proven useful in multi-armed bandit problems over continuous arm domains \citep{kleinberg2008multi, bubeck2011x}. 

A cartoon of our ABC-Tree algorithm, blending the outer and inner loops into a nested structure,  is depicted in Figure \ref{fig:abstract_algorithm}. Intuitively, the outer loop engulfs the posterior modes through recursive partitioning, and the inner loop pulls  boxes with higher reward rates (or higher posterior) more often. {The progression of ABC-Tree is visualized in Figure \ref{fig:progression-ex} on a two-dimensional example.} In a nutshell, active learning algorithms in ABC utilize observations, including proposed parameters and ABC acceptance results, to update the subsequent proposal distribution. In this regard, both the outer and inner loop can be seen as active learning, making the nested algorithm \emph{doubly active}.

\vspace{-0.5cm}
\section{ABC Tree}
\label{sec:abctree}

This section unveils our bandit-style ABC approach to learning proposal distributions (and ABC posterior approximations) supported on a finite partition of the parameter space.
%We start off by focusing on the proposal distribution, characterizing several properties which a good proposal ought to satisfy.

%In this section, we develop our adaptive ABC with the recursive partitioning approach. First, we present an online learning framework for ABC sampling. Then we develop our nested algorithm for this online learning. More specifically, for the \emph{inner loop}, we develop a query-efficient ABC method on a fixed partition of the parameter space, sequentially updating the proposal distribution iteration. Then, we present our \emph{outer loop}, where we sequentially decrease the ABC acceptance threshold $\varepsilon$ and construct informative partitions through a tree-based partitioning. 
\vspace{-0.5cm}
\subsection{Inner Loop: Bandit  ABC}\label{sec:innter_loop}
Suppose that the parameter space\footnote{We assume that after parameter rescaling, the parameter space is contained in a $d$-dimensional unit cube. Our plots will often be transformed back to the original scale.}  $\Theta\subset [0,1]^d$  has  been partitioned into  (box-shaped) regions $\Theta=\cup_{k=1}^K\Omega_k\subset\mathbb{R}^d$.  Each region $\Omega_k$ will be regarded as an arm in the multi-arm bandit setting. Playing an arm is equivalent to first drawing a parameter value $\wt{\bm\theta}$ from inside the box $\Omega_k$ and then generating a fake data vector $\wt{\bm X}_{\wt{\bm\theta}}\sim p_{\wt{\bm{\theta}}}(\cdot)$.
Each box $\Omega_k$ is chosen with a probability $q_k$ where $\sum_{k=1}^Kq_k=1$ and, given $\Omega_k$, the trial $\wt{\bm\theta}$ is selected according to $ \pi(\bm\theta\C\Omega_k)$. These two steps  constitute an ABC  proposal distribution 
\begin{equation}
q(\bm\theta)= \sum_{k=1}^K q_k \times \pi(\bm\theta\C\Omega_k)\mathbb{I} {\{\bm\theta\in\Omega_k\}}\label{eq:proposal}
\end{equation}
 defined as the prior which has been  rescaled  for each bin.  The primary goal of the inner loop is to obtain a query-efficient proposal distribution $\bm q=(q_1,\dots, q_K)'$ on the arms, given the partition $\bm\Omega=\{\Omega_k\}_{k=1}^K$. 

Consider a fixed acceptance threshold $\varepsilon>0$ and a simulation budget $T_{max}$. With the objective of   sampling from the ABC posterior $p_\varepsilon(\bm\theta\mid\bm{X})$, at each occasion $t=1,\ldots,T_{max}$ we select a region $\Omega_k$  and  a parameter guess $\bm\theta^{(t)}$ within $\Omega_k$  to observe an ABC acceptance reward 
\begin{equation}
Y^{(t)}={\rm ABC}(\bm X, \wt{\bm{X}}_{\bm{\theta}^{(t)}}, \varepsilon)\equiv \mathbb{I}\{d(s(\wt{\bm{X}}_{\bm{\theta}^{(t)}}),s(\bm X))<\varepsilon\}.\label{eq:mean_reward}
\end{equation}
 We denote the unknown probability of the $k^{th}$ region $\Omega_k$ (i.e. arm) to yield the reward \eqref{eq:mean_reward} by $\mu_k$. At the commencement of the simulation process, we have no information about how likely each region is to yield a reward and we express our uncertainty about $\mu_k$ through a beta distribution $\mathrm{Beta}(\alpha^{(0)}_k,\beta^{(0)}_k)$.   As the simulation progresses, we learn about $\mu_k$'s by collecting rewards. We leverage this knowledge by updating $\alpha^{(0)}_k$ and $\beta^{(0)}_k$ as well as bin proposal probabilities $q_k^{(t)}$ for every $t$  by prioritizing bins that have a higher success rate. The
updating  of $q_k^{(t)}$ will be achieved by a Thompson Sampling style routine for binary bandits. Before explaining the  updating procedure, it is instrumental to understand how  the ABC reward means $\mu_k$'s relate to the projection of the ABC likelihood onto discrete distributions supported on $\bm\Omega$. This crucial insight,  presented in the Lemma below, provides a justification for our bandit-style approach.
%because $\pi(\bm\theta^{(t)})= \pi(\bm\theta^{(t)}|\Omega_k)\pi(\Omega_k)$ and $q(\bm\theta^{(t)})=\pi(\bm\theta^{(t)}|\Omega_k)q_k$.\footnote{As an alternative, we can also use a multiple importance sampling scheme, assigning the importance weight $\pi({\Omega_k})/\bar{q}_k$, where $\bar{q}_k=\frac{1}{T}\sum_{t=1}^T q_{ k}^{(T)}$ is the average proposal mass over the course of $T$ samples. While \cite{elviraGeneralizedMultipleImportance2019} showed that the latter has lower variance, we use the former for the simple implementation of the nested algorithm.} 
%Throughout this work, we denote by ${\rm ABC}(\bm X,\Omega_k,\varepsilon)= \mathbb{I}\{d(s(\wt{\bm{X}}_{\bm{\theta}}),s(\bm X))<\varepsilon\}$ the binary random variable of ABC acceptance at a threshold $\varepsilon$, where $\bm{\theta}\sim\pi(\bm\theta|\Omega_k)$ and $\wt{\bm{X}}_{\bm\theta}\sim p(\bm{X}\mid \bm{\theta})$. With a slight abuse of notation, we denote the same indicator by ${\rm ABC}(\bm X,{\bm\theta},\varepsilon)$ when $\bm\theta$ is given, and by ${\rm ABC}(\bm X,\wt{\bm X}_{\bm\theta},\varepsilon)$ when $\wt{\bm X}_{\bm\theta}$ is given. 
%  Thanks to \eqref{eq:abc_post}, we have a crucial interpretation that connects a block-wise  marginal posterior to   expected acceptance rates  in the following Lemma.
\begin{lemma} \label{lem:crucial}
For any region $\Omega_k\subset\Theta$ we have
\begin{align}
    p_\varepsilon(\Omega_k\mid\bm X)&\propto\pi(\Omega_k) \mathbb{E}\left[{ {\rm ABC}(\bm X,\wt{\bm{X}}_{\wt{\bm{\theta}}},\varepsilon)}\right].\label{eq:intuition_2}
\end{align}
where the expectation $\mathbb E$ is taken over the distribution of  $\wt{\bm X}_{\wt{\bm{\theta}}}$ for $\wt{\bm{\theta}}\sim \pi(\bm\theta\C\Omega_k)$.
\end{lemma}
\proof This is due to the definition 
\begin{align*}    p_\varepsilon(\Omega_k\mid\bm{X}) &\propto \int_{\Omega_k} \pi(\bm \theta) L_\varepsilon(\bm{X}\mid\bm{\theta})\,\mathrm{d}\bm{\theta} \\
&\propto \pi(\Omega_k)\int_{\Omega_k}\pi(\bm \theta\C \Omega_k)\int_\mathcal{X} \mathbb I\{d(s(\wt{\bm{X}}),s(\bm X))<\varepsilon\}p_{\bm\theta}(\wt{\bm{X}})\,\mathrm{d}\wt{\bm{X}}  \,\mathrm{d}\bm{\theta}.
\end{align*}

This Lemma suggests that one could build a histogram reconstruction of the posterior supported on $\bm\Omega$ by averaging over many sampled rewards from each bin and then normalizing them.
Sampling each bin separately is inefficient if there are many bins that do not contain a lot of posterior mass. This is why we employ Thompson Sampling style algorithm which focuses on the bins that
have enough support. Rather than directly reporting a step function approximation to the ABC posterior, we turn this approximation into an ABC proposal distribution \eqref{eq:proposal} built on a partition $\bm\Omega$.   

\begin{algorithm}[!t]
\caption{Inner Loop}\label{alg:qe_abc}
\spacingset{1.2}
\begin{algorithmic}[1]
\Statex \textbf{Input:} Prior $\bm\pi(\cdot)$, observed data $\bm{X}$, tolerance $\varepsilon$, partition $\bm\Omega=\{\Omega_k\}_{k=1}^K$, budget $T_\mathrm{max}$,
\Statex {\color{white}\textbf{Input:}}   acceptance quota $B$ and    $\bm{\alpha}=(\alpha_1,\dots, \alpha_K)'$ and
$\bm{\beta}=(\beta_1,\dots,\beta_{K})'$
\State \textbf{Initialize} $t=0$ and $\alpha_k^{(0)}=\alpha_k$ and $\beta_k^{(0)}=\beta_k$
\While{$N \leq B$ and $t \leq T_\mathrm{max}$}
\State $t\gets t+1$
% \State Set $\hat{{\bm {p}}}^{(t)} \propto \bm\alpha^{(t-1)}/(\bm\alpha^{(t-1)}+\bm\beta^{(t-1)})$
\State (Binned Likelihood)  ${\eta}^{(t)}_k=\alpha_k^{(t)}/(\alpha_k^{(t)}+\beta_k^{(t)})$ for $k=1,\ldots,K$
\State \label{ln:explo}(Binned Posterior) $\hat{{ {p}}}_k^{(t)}\propto \eta^{(t)}_k\times \pi(\Omega_k)$ for $k=1,\ldots,K$
\State \label{ln:reg_exp}(Regularization)  $\hat{\bm{q}}^{(t)} \gets \bm{q}^*(\hat{\bm{p}}^{(t)})$ where $\hat{\bm{p}}^{(t)}=(\hat{{ {p}}}_1^{(t)},\dots, \hat{{ {p}}}_K^{(t)})'$
%$I =I^{(t)}\sim \mathrm{Cat}({\bm{q}}^{(t)})$, where ${\bm{q}}^{(t)} \gets \bm{q}^*(\hat{{\bm {p}}}^{(t)})$ 
\State {(Arm pulling)} $I=I^{(t)}\sim \mathrm{Cat}(\hat{\bm{q}}^{(t)})$
\State {(Simulate data)} $\wt {\bm X}_{\bm\theta^{(t)}}\sim p_{\bm\theta^{(t)}}(\cdot)$ where $\bm\theta^{(t)}\sim \pi\left(\bm\theta\C \Omega_{I}\right)$
\State {(ABC reward)} $Y^{(t)} \gets {\rm ABC}(\bm X,\wt{\bm X}_{\bm\theta^{(t)}},\varepsilon)$%, where $d^{(t)} \gets d(s(\bm{X}),s(\wt{\bm{X}}_{\bm{\theta}^{(t)}}))$, 
\State {(Importance weight)}\label{ln:i_weight} $u^{(t)} =\pi(\Omega_{I})/\hat{q}^{(t)}_{I }$
\State \label{ln:info}{(Posterior update)} $(\alpha_{I }^{(t+1)}, \beta_{I }^{(t+1)})\gets (\alpha_{I }^{(t)}+Y^{(t)}, \beta_{I }^{(t)}+1-Y^{(t)})$
\State \textbf{if} {$Y^{(t)}=1$} \textbf{then} $N\gets N+1$ 
\EndWhile
\State $T\gets t$
\State \Return $\mathcal{D}(\varepsilon)=\{(\bm\theta^{(t)},\wt{\bm{X}}_{\bm{\theta}^{(t)}}, Y^{(t)}(\varepsilon))\}_{t=1}^T $ and $\{u^{(t)}\}_{t=1}^T$ %accepted parameters and importance weights $\{(\theta_{I^{(m)}},u^{(m)}):Y^{(m)}=1\}$ 
\end{algorithmic}
\end{algorithm}

Algorithm \ref{alg:qe_abc} summarizes the inner loop of our ABC-Tree method.
%learn the optimal ABC proposal $q^{(t)}_k$ on $\Omega_k$ for $1\leq k\leq K$
For each cell $\Omega_k$,
% In the effort to learn about the (marginal) likelihood function on $\Omega_k$, first sample $\eta^{(t)}_k$ from $\mathrm{Beta}\left(\alpha_k^{(t-1)},\beta_k^{(t-1)}\right)$  for each $1\leq k\leq K$. 
we employ a beta distribution $\mathrm{Beta}\left(\alpha_k^{(t)},\beta_k^{(t)}\right)$ which reflects uncertainty about the value of the marginal likelihood $L_\epsilon(\bm X \C\bm\theta)$ in \eqref{eq:conv_likelihood} restricted to  $\Omega_k$ after $t$ iterations. 
At the beginning of the simulation, we have no information about the likelihood and  thereby  set $\alpha_k^{(0)}=\beta_k^{(0)}=1$. Later, as we obtain a new partition from the outer loop in Section \ref{sec:outer}, we can choose a different initialization $\alpha_k^{(0)}$ and $\beta_k^{(0)}$ which is based ABC lookup tables from previous loops. After collecting $t$ ABC rewards (Algorithm \ref{alg:qe_abc} line 9), we can update our beliefs about the likelihood inside $\Omega_k$ by increasing either $\alpha_k^{(t+1)}$ or $\beta_k^{(t+1)}$ (Algorithm \ref{alg:qe_abc} line 11).  Similarly to other query-efficient ABC techniques, it is important to model the uncertainty in the underlying likelihood \citep{gutmann2016bayesian, jarvenpaa2018gaussian} and to sample new parameters to reduce this uncertainty (exploration) while, at the same time, increasing the acceptance rate (exploitation).  We achieve this balance here though a Thompson Sampling (TS) style procedure using Beta distribution updating.

Recall that, in each round $t$, TS chooses an arm to play based on the largest sampled $\eta^{(t)}_k\sim\mathrm{Beta}\left(\alpha_k^{(t)},\beta_k^{(t)}\right)$.  The question remains whether such a strategy is compatible with the goal of exploring the entire ABC posterior. Intuitively, for the purpose of ABC posterior sampling, we might want to allow more exploration and choose an arm according to  {\em sampling} as opposed to {\em optimization}. 
In the effort to learn about the (marginal) likelihood function on $\Omega_k$, we might want to compute the {\em mean}  $\eta^{(t)}_k=\alpha_k^{(t)}/(\alpha_k^{(t)}+\beta_k^{(t)})$ of $\mathrm{Beta}\left(\alpha_k^{(t)},\beta_k^{(t)}\right)$, as opposed to a random sample.
Such  $\eta^{(t)}_k$ can be regarded as an approximation  to $L_\epsilon(\bm X \C\bm\theta)$ restricted to $\Omega_k$. 
By weighting $\eta^{(t)}_k$ by the prior $\pi(\Omega_k)$, Lemma \ref{lem:crucial} implies that $\hat{p}^{(t)}_k\propto \eta^{(t)}_k\times\pi(\Omega_k)$ constitutes an approximation to the ABC posterior  $p_\varepsilon(\Omega_k\C \bm X)$ on $\Omega_k$\footnote{Random sampling from the beta distribution as opposed to computing the mean would yield a more noisy approximation. These two strategies, however, were found to perform comparably well.}.  %This approximation is the key to effective {\em exploration} as they reflect uncertainty and  prevent from oversampling  arms that are promising too early. 
While the   approximation to the discretized posterior encapsulated in $\hat{\bm{p}}^{(t)}=(\hat{p}^{(t)}_1,\dots, \hat{p}^{(t)}_K)'$  alone can be attempted as a proposal for the next step, it might be advisable (as discussed in Section \ref{sec:online}) to further refine this proposal through regularization. Such a refinement, denoted by $\hat{\bm{q}}^{(t)}=(\hat{q}^{(t)}_1,\dots, \hat{q}^{(t)}_K)'$, is then used to pull an arm through {\em sampling} (not optimization) each $k$ with a probability $\hat{q}^{(t)}_k$. 

Since the sampled parameters $\{\bm\theta^{(t)}\}_{t=1}^T$ have been obtained under {\em varying proposals $q^{(t)}(\bm\theta)$} that are different from the prior, we need to revert the samples back to the original prior by importance weighting.  
The parameter value $\bm\theta^{(t)}\in\Omega_k$ can be re-weighted by the importance weight which boils down to $u ^{(t)}= \pi(\bm\theta\in \Omega_k)/q_k^{(t)}$. For a uniform prior over a compact domain, $\pi(\bm\theta\in \Omega_k)$ is simply the reciprocal volume of the box $\Omega_k$. The importance weight can be also approximated by box frequencies by sampling from a prior.  %because $\pi(\bm\theta^{(t)})= \pi(\bm\theta^{(t)}|\Omega_k)
\begin{remark}\label{rmk:rm1}{\rm (Sampling from $\pi(\bm\theta\C \Omega_k)$)}
In practice, it might be difficult to efficiently sample from $\pi(\bm\theta\C \Omega_k)$ for priors other than uniform. It might be worthwhile to run ABC-Tree under a step-function approximation to $\pi(\bm\theta)$ through $\pi_{\bm\Omega}(\bm\theta)=\sum_{k=1}^{K}\pi_k\mathbb{I}\{\bm\theta\in \Omega_k\}$, where $\pi_k=\int_{\Omega_k}\pi(\bm\theta)/|\Omega_k| d\bm\theta$. This prior  boils down to uniform sampling inside the box for each arm pull. Each resulting ABC-Tree sample  $\bm{\theta}^{(t)}$ obtained by the inner loop  would have to be reweighted by a  weight $u ^{(t)}= \frac{\pi(\bm\theta)}{q_k}$ when $\bm{\theta}^{(t)}\in\Omega_k$. The posterior approximation \eqref{eq:step_posterior} would have to undergo suitable reweighting as well.
\end{remark}

Before proceeding, we highlight the fact that besides adaptively collected samples from an ABC posterior, Algorithm \ref{alg:qe_abc} yields a step-function approximation to the $\varepsilon$-posterior \eqref{eq:abc_post} supported on the partition $\bm\Omega$ as a by-product.
 
 \vspace{-0.5cm} 
\subsubsection{Adaptive Histogram Posterior Approximation}\label{sec:post_approx}
Recall that each distribution $\mathrm{Beta}\left(\alpha_k^{(t)},\beta_k^{(t)}\right)$ reflects the uncertainty about the value of the marginal likelihood inside $\Omega_k$. After $T$ runs of the algorithm, the mean of this distribution constitutes a reasonable estimator of this value. We can thereby form a step-function approximation to the $\varepsilon$-posterior \eqref{eq:abc_post} by playing the bandit game as follows
\begin{equation}
\hat p_{\epsilon}^\Omega(\bm\theta\C \bm X)=\sum_{k=1}^K \frac{\hat p^{(T)}_\Omega(k)}{|\Omega_k|}\times \mathbb I\{\bm\theta\in\Omega_k\},\label{eq:step_posterior}
\end{equation}
where $|\Omega_k|$ is the volume of $\Omega_k$ and
$
\hat p^{(T)}_\Omega(k)\propto \pi(\Omega_k)\frac{\alpha_k^{(T)}}{\alpha_k^{(T)}+\beta_k^{(T)}}.
$
The approximation $\hat p^{(T)}_\Omega(k)$ to the $\varepsilon$-posterior mass  of $\Omega_k$ is proportional to the product of the prior mass of $\Omega_k$ and the mean ABC reward for the arm $\Omega_k$.  For the uniform prior on $\Theta$ we have $\pi(\Omega_k)=|\Omega_k|/|\Theta|$ in which case the volumes $|\Omega_k|$ cancel each other out and $\hat p^{(T)}_\Omega(k)$ can be normalized to sum to $|\Theta|$.  The quality of   the $\epsilon$-posterior approximation \eqref{eq:step_posterior} depends on the quality of the partition  $\bm\Omega=\{\Omega_k\}_{k=1}^K$ and its ability to capture local curvatures of the $\epsilon$-posterior. Updating $\bm\Omega$ occurs in the outer loop of ABC-Tree discussed in Section 
\ref{sec:outer}. From \eqref{eq:intuition_2} we see that the Algorithm \ref{alg:qe_abc} learns a step-wise function that approximates the ABC likelihood (posterior) {\em up to a constant}, i.e. we do not require during the computation that the approximation corresponds to a density. This is why we refrain from modeling the arm reward probabilities using a Dirichlet distribution. Note also that the discrete posterior approximation  \eqref{eq:step_posterior} is  very different from a step-wise  approximation that could be obtained after ABC rejections through conjugate updating of a Dirichlet prior distribution (or using Polya Trees) after observing acceptance counts for each bin. Our approach learns from the intermediate ABC rejections by updating the Beta distributions.

\vspace{-0.5cm}
\subsubsection{Benefits of Regularized Exploitation}\label{sec:online}

%{\cp Comment: Would it be rather Regularized $Exploitation$ than Regularized $Exploration$?}{\color{red} I struggled with it. I leave it up to you. Change it everywhere in the text once you decide}

For ABC, some proposal distributions are better than others. One can define an optimal proposal for a target distribution $\bm{p}=(p_1,\ldots,p_K)'$, where $p_k=p_\varepsilon(\Omega_k\C\bm X)$ {
with respect to a \emph{utility} function $\omega_{\bm p}(\bm{q})$ as its minimizer}
\begin{equation}
    {{\bm q}}^*({{\bm p}}) = \argmax_{{{\bm q}}:\|\bm q\|_1 =1} \omega_{\bm p}(\bm q).
    \label{eq:optmodel}
\end{equation}
Note that $\bm p$ is unknown, and thus $\bm q^*(\bm p)$ is unknown as well. Instead, $\bm p$ can be sequentially learned over the course of the ABC simulation process, and $\bm q^*(\bm p)$ can be estimated by $\bm q^*(\hat{\bm p}^{(t)})$ given the current estimate $\hat{\bm p}^{(t)}$ at iteration $t$. We primarily {
concern ourselves with 
\begin{equation}
    \omega^1_{\bm{p}}(\bm{q}) = -\|\bm{q}-\bm{p}\|_2^2\,,\qquad \omega^2_{\bm{p}}(\bm{q}) = \frac{\sum_{k=1}^K q_kp_k/\pi_k}{\sum_{k=1}^K p_k\pi_k/q_k}.
    \label{eq:p_omega}
\end{equation}
The first one is maximized at $\bm{q}^*(\bm{p})=\bm{p}$, implying that $\hat{\bm{q}}^{(t)}=\hat{\bm{p}}^{(t)}$, and so parameters are proposed by sampling from our current model for $\bm{p}$. This approach has been employed in other likelihood-free inference methods \citep{papamakarios2019sequential, sharrock2022sequential}.  The second one $\omega_{\bm{p}}^2$ is the sampling efficiency proposed by \citet{alsing2018optimal}, which is another reasonable candidate as its maximization increases the ABC acceptance rate while controlling importance weight variance.} %The literature also mentions another criterion  which seeks to regularize the log acceptance rate with the KL divergence to the posterior to define optimality \citep{beaumont2009adaptive, filippi2013optimality}. 
At each iteration $t$ of the inner loop (Algorithm \ref{alg:qe_abc} line 6) we update $\bm q^{(t)}=(q^{(t)}_1,\dots, q^{(t)}_K)'$ through the refinement of the posterior estimate $ \hat{\bm p}^{(t)}$ via $\bm q^{(t)}=\bm q^*( \hat{\bm p}^{(t)})$. We regard this projection as  \emph{regularized} exploitation. The purpose of this regularization is to improve the variance of the importance weights by trading-off the ABC acceptance rate and {to avoid oversampling areas around the posterior mode {(as will be shown in Example \ref{ex:1})}. Note that the utility $\omega_{{p}}({q})$ can be quite a general criterion. We can also use the criteria in \citep{beaumont2009adaptive, filippi2013optimality}, which adopt the KL divergence to define optimality {(See, Section \ref{sec:additional_exp})}. {Since our tree-based partitioning effectively restricts the parameter space to be finite, the maximization in Eq. \eqref{eq:optmodel} becomes tractable to perform after every subsequent simulation for a variety of choices of $\omega_{\bm{p}}$.
}

{%\cp 
%\begin{thm}
%Let $q_\varepsilon(\bm{\theta})\propto q(\bm{\theta})L_\varepsilon(\bm{X}\mid\bm{\theta})$ be the distribution of accepted ABC samples at threshold $\varepsilon$ when proposing from $q$. If we sample $\wt{\theta}_1,\ldots,\wt{\theta_N}\sim q_\varepsilon$
%\end{thm}
}
%{\cb Other choices of $\omega_{{p}}$ that satisfy such qualities are, for example, {\cp please list}.We discuss other choices for $\omega_{{p}}$ later on. }

While the utility of an individual proposal $\bm q$ can be measured by $\omega_{\bm p}(\bm q)$, we want to define a measure of effectiveness of an {\em entire inner loop} responsible for producing the sequence of proposal distributions used throughout a sampling procedure.  Therefore, we introduce a new notion of \emph{regret} of an ABC algorithm $\mathcal{A}$ {for learning the optimal proposal}. 
The regret of not using the optimal proposal for an approximation to the $\varepsilon$-posterior $\bm p$ can be defined as the following  loss in efficiency of the proposal distributions
\begin{equation}\label{eq:our_regret}
    R_T(\mathcal{A}, \bm p):= \omega_{\bm p}(\bm q^*) - \E\left[\frac{1}{T}\sum_{t=1}^T\omega_{\bm p}\left(\hat{\bm q}^{(t)}\right)\right]\,
\end{equation}
{where $\mathcal{A}$ denotes the algorithm that decides each proposal $\hat{\bm q}^{(t)}$ based on previous information $\hat{\bm{q}}^{(1)},Y^{(1)},\ldots,\hat{\bm{q}}^{(t-1)},Y^{(t-1)}$. The expectation is taken over the interaction between the algorithm's choices $\hat{\bm{q}}^{(t)}\mid\{\hat{\bm{q}}^{(1)},Y^{(1)},\ldots,\hat{\bm{q}}^{(t-1)},Y^{(t-1)}\}$, and ABC acceptances $Y^{(t)}\mid \hat{\bm{q}}^{(t)}$. Such a regret can be seen as loosely parallel to the notion of cumulative regret in multi armed bandits \eqref{eq:thompson}, where the key distinction is that the regret considers the distribution used to propose arms to play rather than only the arm choices and corresponding rewards.} 

The following theorem says that Algorithm \ref{alg:qe_abc} indeed controls the regret for suitable efficiency functions $\omega$. {Note that $\omega^1$ satisfies the conditions in Assumption \ref{ass:2} for any $\delta>0$. In Section \ref{sec:special_case}, we show that there exists $\delta>0$ such that $\omega^2$ satisfies Assumption \ref{ass:2}}.%%% for finite parameter spaces.Consider the special case where $\pi(\bm\theta)$ and $L_\varepsilon(\bm{X}\mid\bm\theta)$ are both constant functions of $\bm{\theta}$ within each region $\Omega_k$. In such cases, let $\pi_k:=\pi(\bm{\theta}{_k})$ and $p_k:=p_\varepsilon(\bm\theta{_k}\mid\bm{X})$ for some $\bm\theta{_k}\in\Omega_k$. For example, this could be a model with a finite parameter space{\ck, where the parameter $\bm\theta_k$ identifies the partition box by $\Omega_k=\{\bm\theta_k\}$}. 

\begin{assumption}\label{ass:2}
{Let $\mathcal{S}^K=\{\bm{p}\in\mathbb{R}^K:\sum_{k=1}^K p_k=1, p_j\geq 0 \text{ for }j=1,\ldots,K\}$. For all $\bm{p}\in\mathcal{S}^K$,} there exists $\delta>0$ such that $\omega_{\bm{p}}$ satisfies
\begin{align*}
\text{(For upper bound)}~&{|}\omega_{\bm{p}}(\bm{q}) - \omega_{\bm{p}}(\bm{q}') {|}\lesssim \|\bm{q}-\bm{q}'\|_2^2,\\    
\text{(For upper bound)}~&\|\bm{q}^*(\bm{p}) - \bm{q}^*(\bm{p}')\|_2 \lesssim \|\bm{p}-\bm{p'}\|_2\,,
\end{align*}
for all {$\bm{p}'\in B_\delta(\bm{p})\cap\mathcal{S}^K$, $\bm{q}'\in B_\delta(\bm{q})\cap\mathcal{S}^K$, and
\begin{align*}
\text{(For lower bound)}~&{|}\omega_{\bm{p}^0}(\bm{q}^*(\bm{p}^0)) - \omega_{\bm{p}^0}(\bm{q}^*(\bm{p'})) {|}\gtrsim \|\bm{p}^0-\bm{p}'\|_2^2,
\end{align*}
%\[
%\|\bm{q}^*(\bm{p}) - \bm{q}^*(\bm{p})\| \lesssim \|\bm{p}-\bm{p}'\|^\gamma
%\]
%and 
%\[
%\|\bm{q}^*(\bm{p}) - \bm{q}^*(\bm{p})\| \gtrsim \|\bm{p}-\bm{p}'\|^\gamma
%\]
for all $\bm{p}'\in B_{\delta}(\bm{p}^0)$ where $\bm{p}^0=(1/K,\ldots,1/K)'$, and $B_\delta(\bm{x})\equiv \{\bm{y}:\|\bm{y}-\bm{x}\|_2\leq \delta\}$.}
\end{assumption}

\begin{thm}\label{prop:regretub}  Consider a slightly modified version of Algorithm \ref{alg:qe_abc} that employs an initialization phase by first playing each arm  $m=K\gamma^{-1}\log^2(T)$ times. 
For $Y$ defined in \eqref{eq:mean_reward}, let $\mu_k=P(Y=1\mid \Omega_k)$ be the ABC acceptance rate when sampling from $\Omega_k$. Let $\pi_k=\pi(\Omega_k)=1/K$ be the prior mass on $\Omega_k$
and let $\mathcal{I}_\gamma$ be a set of probability vectors $\bm p=(p_1,\dots,p_K)'$ such that $p_k\propto \pi_k \mu_k$ when $\min\limits_{1\leq k\leq K}\mu_k\geq \gamma>0$. Denoting the modified algorithm as $\mathcal{A'}$, we have (for sufficiently large $T$)
\[
\sup_{\bm p\in\mathcal{I}_\gamma} R_T(\mathcal{A}',\bm p) \lesssim \frac{K^2 \log^2(T)}{\gamma T}\quad\text{and}\quad
\inf_{\mathcal{A}}\sup_{\bm p\in\mathcal{I}_\gamma} R_T(\mathcal{A},\bm p)\gtrsim \frac{K^2}{\gamma\,T},
\]
where the infimum is taken over all possible strategies $\mathcal{A}$ for choosing bins $\Omega_k$.
\end{thm}

\proof See  Section \ref{sec:theorem1_proof} in the supplement.

The above result highlights that this regularized exploitative approach controls regret at a rate that at most corresponds to a lower bound up to logarithmic factors. Intuitively, maximizing the sampling efficiency for $\hat{{\bm {p}}}^{(t)}$ inherently enforces a sufficient amount of exploration, {\ck while} \emph{pushing} the proposal distribution closer to the prior in order to control importance weight variance.

\begin{figure}[!t]
    \centering
    \includegraphics[width=0.8\linewidth]{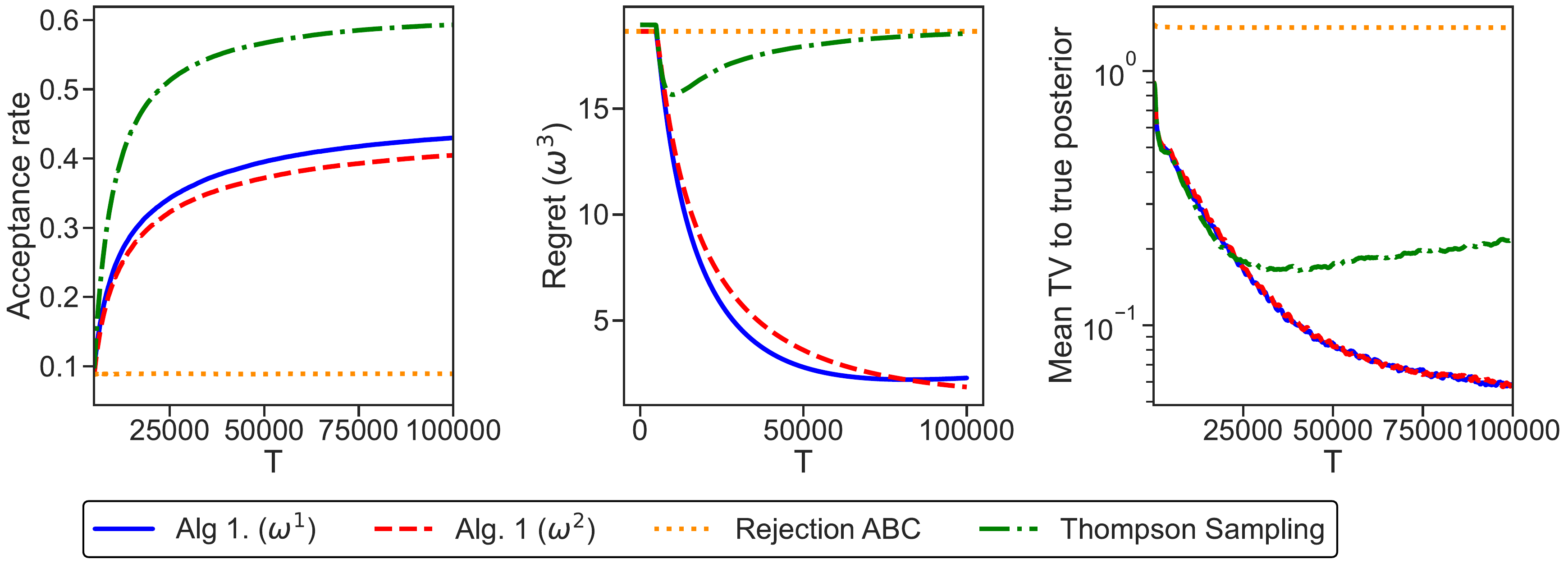}
    \caption{Empirical regret comparison. We compare various methods for sequentially updating proposals in a finite parameter space example. See Example \ref{ex:1} for details.}
    \label{fig:regrets-finite}
\end{figure}

\begin{exa}\label{ex:1}
 {{We consider a toy example with a discretized normal likelihood on the parameter space $\Theta=\{\theta_1,\dots,\theta_{100}\}'$ with a uniform prior}. In Figure \ref{fig:regrets-finite}, we compare rejection ABC (orange) with four schemes for sequentially updating the proposal distribution. {We use Algorithm \ref{alg:qe_abc} with each utility function in \eqref{eq:p_omega} and Thompson Sampling (green)}. Despite Thompson Sampling having the largest acceptance rate, it is unable to recreate the posterior as the variance of the importance weights becomes large while it eventually over-samples the posterior mode. Rejection ABC has the lowest acceptance rate. Algorithm \ref{alg:qe_abc} represents a sweet spot between them. We refer to Section \ref{implementation:empregretcomp} for details.}
\end{exa}

%What about Inner Loop: ABC bandit? or, ABC Thompson Sampling? Those are not my favorite, but the one I used is also not my favorite. I think we need to keep thinking for a good title for this inner loop Ok!!! :>
%ABC bandit is okay actually, maybe it can be slightly changed into something good
%Yeah probably. When you are writing, (I am not a good writer though), you can think like you are guiding people into our world, who are not so much patient and kind (!), but we do not want to loose their attention and interest until we reach and finish our meat part. Probably wise and kind waiter (explaining the special menu that night - so that the customers in the end choose the most expensive special meal) for the hungry mad customers? I don't know, any way, as long as we keep in mind the connection between chuncks and sections, and maitaining a succint writing while putting all we want to say (is it even possible?), we will be good :> I believe that you will do wonderful!! See you!

\vspace{-0.5cm}
\subsection{Outer Loop: Adaptive Discretization}\label{sec:outer}
The inner loop learns the optimal proposal for a given partition. The ability of the partition to adapt to the shape of the ABC posterior will massively influence the ABC sampling success, as demonstrated in the following example.

\begin{exa}
In Figure \ref{fig:finite-1d}, we consider a toy model $X\sim 0.8\times\mathcal{N}(100\times\theta,3)+0.2\times\mathcal{N}(100\times\theta+5,0.25)$ with a prior $\theta\sim\mathcal{U}(0,1)$. We display the performance of Algorithm \ref{alg:qe_abc} where the parameter space $\Theta=[0,1]$ has been partitioned into regular bins. A too-coarse grid (first row) may not be able to eliminate trials in the areas where the posterior is small. On the other hand,  a too-fine grid (bottom row) may lead to overfitting \citep{van2017bayesian, rovckova2020posterior}. In practice, the optimal grid size depends on the smoothness of the underlying ABC acceptance rate function, and of the intractable likelihood function, which is not known \emph{a priori}.  The middle row presents partitioning which is able to provide a more accurate approximation to the ABC posterior within fewer trials.
% {To overcome these challenges, we will sequentially refine partitions in the outer loop, by using tree-based regressions on our current state of knowledge.} Note that penalized regression approaches with recursive partitioning can adapt not only to local smoothness, but also find an optimal number of bins \citep{vanderpasBayesianDyadicTrees2017}. Furthermore, additive (ensemble) trees may perform even better in identifying the optimal partition in a Bayesian context \cite{rovckova2020posterior}.

\end{exa}

\begin{figure}[!t]
    \centering
    \includegraphics[width=0.48\linewidth]{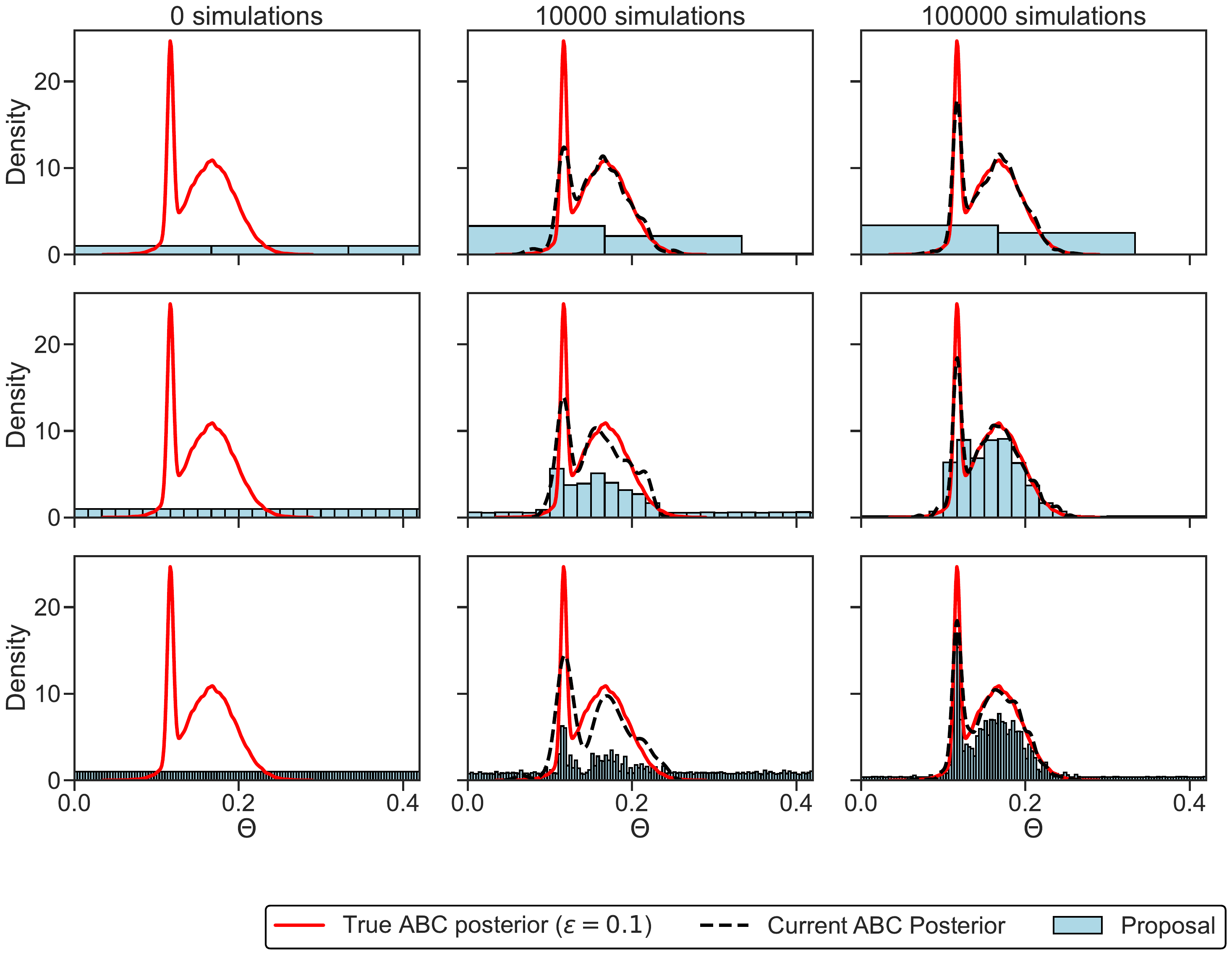}
    \includegraphics[width=0.50\linewidth]{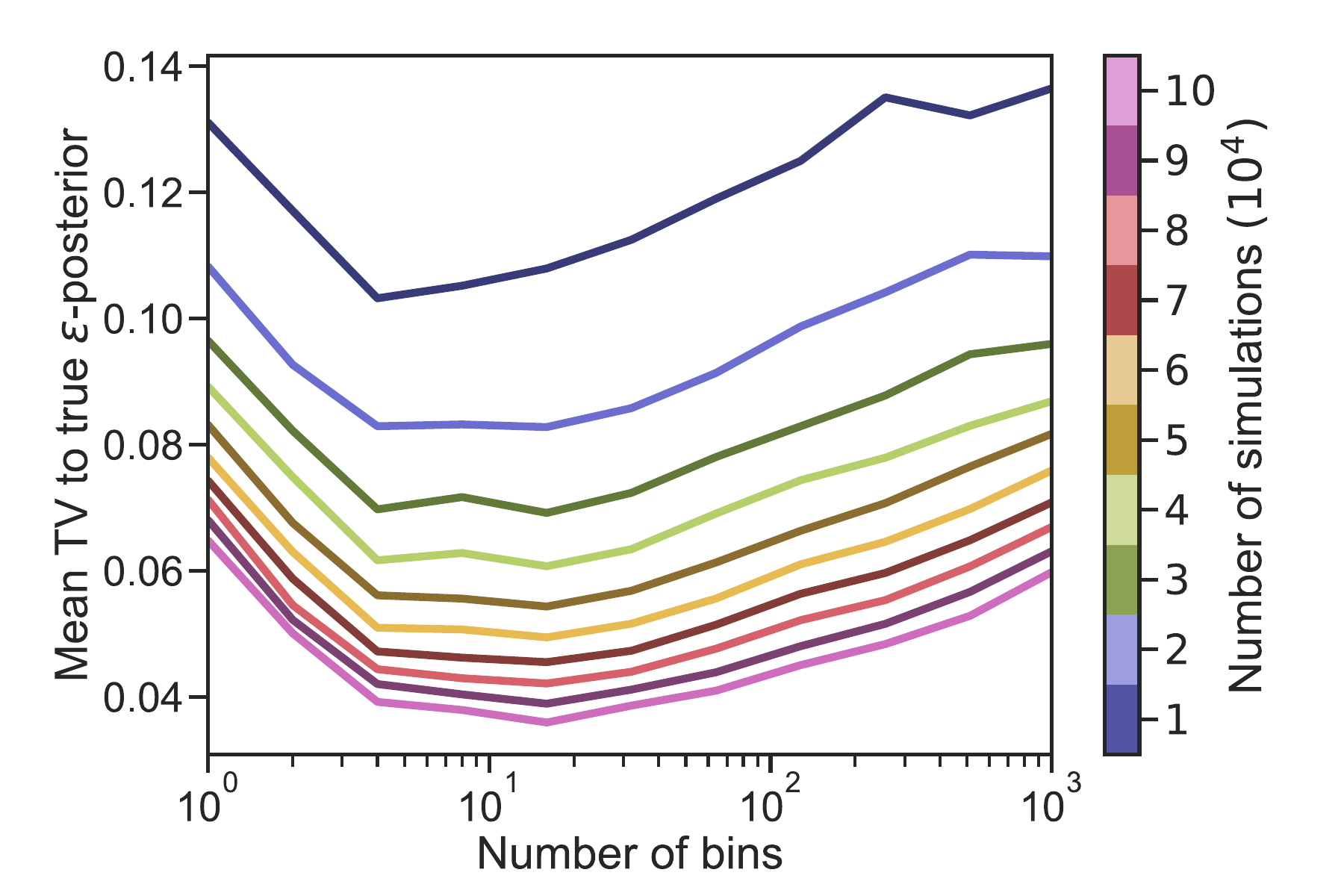}
    \caption{{\textbf{Fixed grid}. Performance of Algorithm \ref{alg:qe_abc} on a univariate Gaussian mixture example using a regular   grid of various sizes for partitioning $\Theta$. Left: kernel density estimates from accepted samples. Right: Effect of bin size on accuracy of posterior sampling.}}
    %{ \cp where learning about all of them requires a large quantity of  iterations, and information is no longer shared between neighboring bins, which are likely to have relatively limited variance in the accepted rate.}  
    %\cb discuss (the last plot): Veronika, what should be the message delivered by the third row? }%I am not sure what message we intend to deliver by this plot. Why do we need to consider a fixed grid? }%$\theta\sim\mathcal{U}(0,10)$, and $X\sim\mathcal{N}(\theta,1)$ {\cb Gaussian mixture}. The grid in the top row consists of 10 bins, while the grid in the bottom row has 100.}
    \label{fig:finite-1d} 
\end{figure}

The outer loop of ABC-Tree seeks to sequentially discretize the parameter space $\Theta$ into  hyperrectangles on which to apply Algorithm \ref{alg:qe_abc}. Each iteration $t$ of the inner loop assumes that the ABC tolerance is fixed at $\varepsilon$. We index the outer loop iterations with $s$ and use a sequence of decreasing thresholds $\varepsilon_1>\dots>\varepsilon_S$, one $\varepsilon_s$ for each outer loop $s$ (analogously to particle-filtering based ABC approaches \cite{sisson2007sequential, beaumont2009adaptive}).
Beginning with larger tolerance thresholds allows us to more quickly learn areas of the parameter space that are likely to be relevant at lower tolerance levels, which will be necessary for accurate posterior inference.

We denote  the ABC lookup table from the $s^{th}$ round with 
\begin{equation}
\mathcal D^{s}(\varepsilon)\equiv \{(\bm{\theta}_s^{(t)}, \wt{\bm X}_{s}^{(t)}, Y^{(t)}_{s}(\varepsilon))\}_{t=1}^{T_{s}},\label{eq:Ds}
\end{equation}
where  
the rewards $Y^{(t)}_{s}(\varepsilon)=\mathrm{ABC}(\bm X,\wt{\bm X}_{s}^{(t)}, \varepsilon)$ have been computed on $(\bm{\theta}^{(t)}, \wt{\bm X}_{s}^{(t)})$ at a chosen threshold $\varepsilon$ (not necessarily $\varepsilon_s$).  With $\mathcal D^{<s}(\varepsilon)=\{\mathcal D^1(\varepsilon),\dots, \mathcal D^{s-1}(\varepsilon)\}$ we denote the entire history before the $s^{th}$ round. The goal of the $s^{th}$ inner loop is to learn a new partition $\bm\Omega^{s}$ reflecting the allocation of past acceptances {\em at the most recent threshold $\varepsilon_{s-1}$.} This is why we use the training data $\mathcal D^{<s}(\varepsilon_{s-1})$,
recomputing the binary rewards  at a level $\varepsilon_{s-1}$ for all  preceding rounds.  The new partition 
\begin{equation}\label{eq:partition_scheme}
\bm\Omega^{s}=\{\Omega_k^s\}_{k=1}^{K_s}\sim P\left(\mathcal D^{<s}(\varepsilon_{s-1}),\bm\Omega^{s-1}\right) %\bm\Omega^{s}=\{\Omega_k^s\}_{k=1}^{K_s}\sim P\left(\mathcal D^{<s}(\varepsilon_{s-1}\right),\bm\Omega_{s-1}) 
\end{equation}
is chosen according to a rule (distribution) $P$ which can be either a delta measure when the partitioning with a deterministic function of the data (such as CART) or random when it represents a sample from a distribution over partitions given $D^{<s}(\varepsilon_{s-1})$. This rule may also depend on the previous partition $\bm\Omega^{s-1}$ if it were to be (deterministically) refined by adding  more splits. We primarily focus on tree-shaped partitions and discuss various choices of $P$ later. A well chosen tree structure will allow us to decipher variability of the acceptance rate by increasing resolution in relevant areas and, at the same time, splitting less often in flat (negligible) likelihood areas.

\begin{algorithm}[t]
\caption{ABC-Tree}\label{alg:seqtreeabc}
\spacingset{1.2}
\begin{algorithmic}[1]
\Statex \textbf{Input:} Prior $\bm\pi(\cdot)$, observed data $\bm{X}$, initial tolerance $\varepsilon_1$,
\Statex {\color{white}\textbf{Input:}} Number of rounds $S_{max}$,  tolerance progression rule $g(\varepsilon)$ % number of acceptances per batch , initial partition $\bm \Omega=\{\Omega_1,\ldots,\Omega_K\}$%, {\cb ABC acceptance operator ${\rm ABC}(\bm X,\bm\theta,\varepsilon)$}
\State \textbf{Initialize}: $s=1$, $K_1=1$, $\bm\Omega^1=\{\Omega_k\}_{k=1}^{K_1} = \{\Theta\}$, $\bm\alpha =\bm{\beta} =1$
\State While $s\leq S_{max}$
\State $~~~~$ {\em Inner Loop}: Collect $\{(\bm{\theta}_s^{(t)}, \wt{\bm X}_{s}^{(t)}, Y^{(t)}(\varepsilon_s), u^{(t)}\}_{t=1}^{T_s}$ using Algorithm \ref{alg:qe_abc} with $\bm\Omega^s,  \bm\alpha , \bm{\beta} ,\varepsilon_s$
\State $~~~~$ $s=s+1$
\State $~~~~$ {\em Collect  Data}: $\mathcal D\equiv \mathcal D^{<s}(\varepsilon_{s-1})=\cup_{j=1}^{s-1}\mathcal D^{j}(\varepsilon_{s-1})$ where $\mathcal D^j(\varepsilon)$ is defined in \eqref{eq:Ds}
\State $~~~~$ {\em Update a Partition}: $\bm \Omega^s = \{\Omega_k^s\}_{k=1}^{K_s} \sim P(\mathcal D)$
\State $~~~~$ {\em Update Initialization}: $\bm\alpha=(\alpha_1,\dots,\alpha_{K_s})'$ and $\bm\beta=(\beta_1,\dots,\beta_{K_s})'$  for $k=1,...,K_s$
\State $~~~~~~$ $\alpha_k \gets  \sum\limits_{(\bm{\theta}, Y )\in   \mathcal D} \mathbb{I}\{ \bm\theta \in \Omega_k^s\} Y $ 
\State $~~~~~~$ $\beta_k \gets \sum\limits_{(\bm{\theta}, Y )\in   \mathcal D} \mathbb{I}\{ \bm\theta \in \Omega_k^s\}(1- Y)$ 
\State $~~~~$ {\em Reduce Tolerance}: $\varepsilon_s\gets g(\varepsilon_{s-1})$ 
 \State \Return $\Omega_{S_{max}}$ and $\mathcal D$
\end{algorithmic}
\end{algorithm}

The ABC-Tree enterprise is visualized in Figure \ref{fig:progression-ex}. Each row corresponds to an Inner Loop, keeping the partition $\bm\Omega^s=\{\Omega_{k}^s\}_{k=1}^{K_s}$ and the tolerance level $\varepsilon_s$ fixed. The Inner Loop updates the probability mass of the proposal (and the piecewise constant posterior approximation) within each hyper-rectangle $\Omega^s_k$ at every iteration $1\leq t\leq T_s$. After accepting $B$ parameter values, we move onto $(s+1)^\mathrm{st}$ inner loop at a decreased level $\varepsilon_{s+1}$ and on an  updated  partition $\bm\Omega^{s+1}$.
 Note that the parameters   $\bm\alpha$ and $\bm\beta$ in line 8 and 9 of Algorithm \ref{alg:seqtreeabc} provide an initialization of the mean reward Beta distributions for the $(s+1)^\mathrm{st}$ Inner Loop based on the rejection information from previous rounds.  When updating the partition through \eqref{eq:partition_scheme}, we employ all of our current information in an \emph{extended window} approach, in which we construct the partition based on all of our previous data points. As an alternative, one can employ a \emph{rolling window} style of analysis by restricting the partitioning rule $P$ to only depend on the most recent completed outer loop. This is a reasonable choice because as the algorithm progresses, the acceptance threshold $\varepsilon$ becomes smaller, causing the $\varepsilon$-likelihood to become closer to the true intractable likelihood function. As such, the most recent batch of simulated data is likely to provides a more representative sample of the true posterior than the previous data. 

\subsubsection{Partitioning Choices} 

Several options for $P$ in \eqref{eq:Ds} are available. We consider several options below but this list is by no means exhaustive.
\paragraph{Classification Trees.} We can model the past ABC acceptances using a regularized classification tree $\bm\Omega^s=P\left(\mathcal D^{<s}(\varepsilon_{s-1})\right)$ obtained by regressing  $Y^{(t)}_{s'}(\varepsilon_{s-1})=\mathrm{ABC}(\bm X,\wt{\bm X}_{s'}^{(t)}, \varepsilon_{s-1})$ onto $\{\bm\theta_{s'}^{(t)}\}$ for $1\leq t\leq T_{s'}$ and $1\leq s'<s$. Using a pruned tree after cross-validation is a deterministic assignment and  will lead to a rough partition which may wrap sharp posterior modes into larger boxes through data-adaptive (non-dyadic) splitting. This strategy will be sufficient for learning an ABC proposal and for obtaining an ABC posterior approximation from the samples $\bm\theta_s^t$ after importance re-weightining. This strategy may not necessarily yield an accurate step-function approximation to the posterior (as outlined in Section \ref{sec:post_approx}).

 \paragraph{Bayesian Additive Regression Trees.}  BART \cite{chipman2010bart} is an ensemble of trees yielding finer partitions (by overlapping many small partitions) with good approximation properties \citep{rovckova2020posterior} The partitioning function $\bm\Omega^s=P\left(\mathcal D^{<s}(\varepsilon_{s-1})\right)$ in \eqref{eq:partition_scheme} is a {\em random draw} from the BART posterior (overlapping partitions for each sampled tree) after a burnin period. 
 While the default strategy in BART is $50$ trees, a draw of a forest this large  will lead to a very refined partition capable of approximating the acceptance function well but increasing the number of arms to pull in the next round. {The optimal number of bins depends on the dimension of $\Theta$ as well as the smoothness of the underlying function and the number of data points \citep{rovckova2020posterior}. Not knowing the optimal number of bins in advance, we recommend to use BART for higher dimensional parameter spaces, which may adaptively finds the optimal number of bins.}
%\paragraph{Classification Trees} We can model the past ABC acceptances using a regularized (pruned) classification tree (CART \citep{Breiman1984ClassificationAR}) which is a deterministic function of the data $D^s$.  This yields a rough partition which may be able to wrap sharp posterior modes into larger boxes through data-adaptive (non-dyadic) splitting.

\paragraph{Dyadic Partitioning.} Similarly to the hierarchical online optimization algorithm of \citet{bubeck2011x}, we can  bisect the parameter space dyadically at a midpoint along each box. For the compact parameter space $\Theta=[0,1]^d$, it is easy to compute the volume of each dyadic cell. We can grow a binary tree over the parameter space from a root node by splitting after sufficiently many parameter values  have been accepted at the $\varepsilon_s$ threshold. {We split the leaf node in which the largest quantity of parameters were proposed $r$ times after each iteration of the outer loop, where $r$ may be chosen by the user and may depend on the dimension of the parameter space. In the example below, we set $r=10$.} The dimension in which to split at each depth of the tree can be simply selected sequentially, or %chosen by identifying the dimension with maximal variation in acceptance rate in the selected leaf node. For instance, 
{by splitting in the direction that best separates accepted from rejected parameter values. This adds a form of variable selection that allows the algorithm to split in dimensions that are more relevant to variation in the acceptance rate within the leaf node.} %{\color{red} Sean, please elaborate on this strategy. Stopping criterion, splitting criterion etc.} {\cp I tried to elaborate on it.}%Its basic demonstration is shown in Figure \ref{fig:hierarchical-progression} and Figure \ref{fig:hierarchical-progression2} in Appendix.

As an alternative, the dynamic trees of \citet{taddy2011dynamic} provide another appealing option for partitioning $\Theta$. Bayesian inference over these trees is accomplished via a particle learning algorithm that is amenable to sequential updating, and a filtered posterior that limits the amount of change with each new observation. For the purpose of ABC posterior sampling, the partition may not need to be fine enough as long as it can successfully direct regions of high-low likelihood. {In Table \ref{tab:complexity} (Supplement Section \ref{sec:add_figures}), we see that our tree ABC can be as fast as SMC-ABC \citep{sisson2007sequential, beaumont2009adaptive} when $K < dB$}.

\begin{exa}\label{ex:2}
We show the progression of ABC proposals on a simple synthetic example  with a mixture of two bivariate isotropic Gaussians. With $\boldsymbol\theta=(-1,-1)^\top$ we observe data from 
$\bm X\sim 0.3\,\mathcal{N}\left(\boldsymbol\theta, \,\mathbf{I}\right) + 0.7 \,\mathcal{N}\left(\boldsymbol\theta+3\,\mathbf{1}, \frac{1}{4}\,\mathbf{I}\right)$ and  assume a uniform prior\footnote{While here we did not rescale the parameters to be inside $[0,1]^2$, the partitioning idea is the same.} $\theta_i\sim\mathcal{U}(-5,5)$ for $i\in\{1,2\}$. Figure \ref{fig:hierarchical-progression} displays the sequence of partitioning-based proposal distributions  using a pruned classification tree (top row), BART {(five trees with 1000 burn-in)}, and dyadic partitioning. As $\varepsilon$ decreases, and we accrue more simulated data,  the proposal distributions more closely encapsulates the true ABC posterior.  
\end{exa}

\begin{figure}
    \centering
    \includegraphics[width=\textwidth]{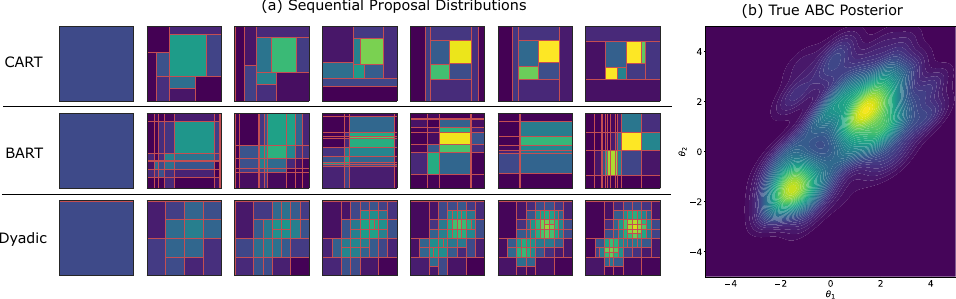}
    \caption{Sequential refinements of the proposal distribution using various tree-based partitioning methods on the Gaussian mixture example discussed in Example \ref{ex:2}. Top: CART. Middle: BART. Bottom: Dyadic partitioning.}
    \label{fig:hierarchical-progression}
\end{figure}

\vspace{-1cm}
\section{MAP-Tree}\label{sec:MAPTREE}
Another common goal, besides (ABC) posterior sampling, is computing an actual point estimator. While the posterior mean can be obtained directly through simulation, the sometimes preferred maximum-a-posteriori  (MAP) estimator is not immediately available when the likelihood cannot be evaluated.  We develop a  maximum-a-posteriori  (MAP) estimation strategy based solely on simulations from the likelihood, i.e. no likelihood values or gradients are required. Similarly as with ABC-Tree, we build a partition by sequentially refining high-probability areas of the posterior. However, unlike with ABC-Tree, our objective is not to find a proposal for ABC to reflect   entire posterior uncertainty but rather to find a partition that can isolate the posterior mode. 
Below, we describe the inner and outer loops of MAP-Tree designed for this purpose.

\vspace{-0.5cm}
 \subsection{Inner Loop: Best Arm Identification}
We consider a fixed, small  `goal' tolerance $\varepsilon$ and a fixed budget of $T$ evaluations of the simulator. Our objective is to  maximize the ABC posterior $p_\varepsilon(\bm\theta \C \bm X)$ in \eqref{eq:abc_post}  as accurately as possible under this budget. This can be framed as  active learning with a regret 
\begin{equation}\label{eq:MAP_reg}
    r_T = p_\varepsilon(\bm\theta^*\C \bm X) -p_\varepsilon(\wh{\bm\theta}^{(T)}\C \bm X), 
\end{equation}
where $\bm\theta^* = \argmax_{\bm\theta\in\Theta} p_\varepsilon(\bm\theta\C\bm X)$ and where $\wh{\bm\theta}^{(T)}$ is the best guess at $\bm\theta^*$ after $T$ steps. To find an approximation to $\bm\theta^*$, we adopt a partitioning perspective similar to ABC-Tree. For a given partition $\bm\Omega=\{\Omega_k\}_{k=1}^K$, the goal of the inner loop is minimize the regret \eqref{eq:MAP_reg} using a proposal distribution supported on $\bm\Omega$.
To find the posterior mode $\bm\theta^*$ we start with an easier problem. We find a mode of an approximation to $p_\varepsilon(\bm\theta^*\C \bm X)$ based on its projection  onto step functions supported on $\bm\Omega=\{\Omega_k\}_{k=1}^K$.  Based on this projection, defined as $\sum_{k=1}^K\mu_\Omega (k)\mathbb I\{\bm\theta\in \Omega_k\}$ with $ \mu_\Omega (k)=\frac{p_\epsilon(\Omega_k\C \bm X)}{|\Omega_k|}$,  we  attempt to identify the bin index $ k^* $ for which  $ \mu_\Omega (k)$ is the largest. Namely, we define
$$
k^* =\arg\max_{1\leq k\leq K} \mu_\Omega(k).
$$
{We aim to estimate $k^* $ by $\hat k^{(T)} $ obtained through playing a bandit game with a goal of minimizing $P(\hat{k}^{(T)}\neq k^*)$.}
%\begin{equation}\label{eq:MAP_reg2}
%    r_T^\Omega  = \mu_\Omega(k^* )- \mu_\Omega(\hat k^{(T)}).
%\end{equation}
We estimate $\hat{k}^{(T)}=\argmax_k \pi_k\alpha_k^{(T)} / \big( |\Omega_k| (\alpha_k^{(T)}+\beta_k^{(T)})\big)$, the arm with the maximal perceived average posterior peak according to data accrued over the course of the game.
%{\cp \citet{audibert2010best} note that $\Delta_{\mathrm{min}}P(\hat{k}^{(T)}\neq k^*) \leq \mathbb{E} [r_T^\Omega] \leq P(\hat{k}^{(T)}\neq k^*)$, where $\Delta_\mathrm{min}=\min_{k\neq k^*} \mu_\Omega(k^*)-\mu_\Omega(k)$ denotes the minimal gap, and thus the regret mirrors the probability of misidentifying the optimal arm.}
The values $ \mu_\Omega(k)$ depend on the marginal likelihood within $\Omega_k$ that is unknown, which we learn about over time.
However, our regret is different from \eqref{eq:thompson} and manifests itself as an instance of the best arm identification problem \cite{audibert2010best}. For this reason, we may also seek to modify the playing strategy to de-emphasize greedy maximization and allocate more effort to differentiating between nearly optimal options.

The inner loop algorithm is presented as Algorithm \ref{alg:map_abc} in Section \ref{sec:ts}. Algorithm \ref{alg:map_abc} essentially operates identically to Algorithm \ref{alg:qe_abc} except  for the \emph{acquisition rule} for choosing the next arm given the history. {For MAP, we are not concerned with controlling the importance weight variance of our proposal, and instead seek to maximize the prior-weighted ABC acceptance rate.} We want to oversample small bins having large estimated average posterior mass and zoom in onto these bins in the following outer loop. Therefore, Algorithm \ref{alg:map_abc} directly takes a maximum 
  $I^{(t)} = \argmax\limits_{1\leq k\leq K} \eta_k^{(t)}\pi(\Omega_k)/|\Omega_k|$, where $\eta_k^{(t)}$ is the sampled acceptance rate of $k$-th bin. This algorithm is more exploitative as there is no regularization (as discussed in Section \ref{sec:online}) nor is there sampling from a categorical distribution to pick an arm. {In contrast with ABC-Tree which  chooses an arm by first deterministically estimating the mean rewards and then randomly picking an arm, ABC-MAP  instead first samples mean rewards $\eta_k^{(t)}$  and then picks the best arm based on the magnitude of these rewards.}

Alternatively, we can also pull arms according to Top-two Thompson Sampling (TTTS) \citep{russo2016simple} designed for best-arm identification in the multi-arm bandit problem. By oversampling the arm perceived as second-best, the algorithm has more confidence in identifying the best arm. When using the top-two variant, we pull arm by $I^{(t)} =  \argmax\limits_{1\leq k\leq K} \eta_k^{(t)}\pi(\Omega_k)/|\Omega_k|$ with probability $b\in (0,1)$. With probability $1-b$, we alternatively repeatedly resample $\bm{p}^{(t)}$ from the posterior model until a different parameter value from the original maximizer is chosen.

%and then 
%$$
%\bm\theta^{(T)}=\hat\theta(\Omega_{\hat k}),
%$$
%where $\hat\theta(\Omega)$ can be a center of the box $\Omega$ or something else {\color{red}elaborate}.
%Our algorithm again has two loops. The inner loop uses a Thompson Sampling style framework but targeted to best arm identification (as opposed to acceptance rate maximization). 
%The goal is to pull more often the arm having larger estimated posterior. For the outer loop, we refine the partition to better isolate the mode of the $\varepsilon$-posterior.  

{The following theorem argues that Algorithm \ref{alg:map_abc} controls the probability of misidentifying the optimal arm in a finite parameter space at a near optimal rate.

\begin{thm}\label{thm:map_rate}
Consider a parameter space $\Theta$ partitioned as $\Theta=\cup_{k=1}^K \Omega_k$. Consider a uniform prior on regions $\pi_k=1/K$ for each $k$, and let $\mu_k=\int_{\Omega_k}p_{\epsilon}(\bm{\theta}\mid\bm{X})\,\mathrm{d}\bm{\theta} / |\Omega_k|$ be the average $\epsilon$-posterior mass within $\Omega_k$. Let $k^*=\argmax_k \mu_k$ and let $\hat{k}^{(T)}_\mathcal{A}$ be the estimate of $k^*$ after $T$ simulations using an algorithm $\mathcal{A}$. Denote  by $\mathcal{A}'$ the Algorithm \ref{alg:map_abc} with  $b=1/2$. Then, for any $\bm\mu=(\mu_1,\dots,\mu_K)'$ there exists  a constant $\Gamma_{\bm{\mu}}$ such that 
$$
P(\hat{k}^{(T)}_{\mathcal{A}'}\neq k^*)=\exp(-T(\Gamma_{\bm{\mu}} + o(1))),
$$ where $o(1)$ denotes a vanishing sequence as $T\to\infty$. Moreover, any algorithm $\mathcal{A}$ must satisfy $P(\hat{k}^{(T)}_\mathcal{A} \neq k^*) \geq \exp(-T(2\Gamma_{\bm{\mu}}+o(1)))$.
\end{thm}

\proof See  Section \ref{sec:MAP_proof} in the supplement.

Theorem \ref{thm:map_rate} is a consequence of the analysis of top-two Thompson sampling in \citet{russo2016simple}, and shows that Algorithm \ref{alg:map_abc} identifies the optimal mode at nearly the optimal rate, up to a factor of 2 in the exponent. This factor can be nullified when $b$ is chosen with oracle knowledge, or adaptively learned during the course of the algorithm \citep{russo2016simple}. 

{As an example, in Figure \ref{fig:trajectory-map}, we {apply} Algorithm \ref{alg:map_abc} with TS and Top-two TS (TTTS) on a toy example with a finite parameter space $\Theta=\{\theta_j\}_{j=1}^{50}$ equipped with a uniform prior. As the algorithm progresses, it proposes the parameter value with maximal $\varepsilon$-posterior density $(\theta_{28})$ more frequently and  with an increasing rate. For more details, see Section \ref{sec:alg3_MAP_details}.}
}

\begin{figure}%[H]
    \centering
    \includegraphics[width=0.8\textwidth, height=7cm]{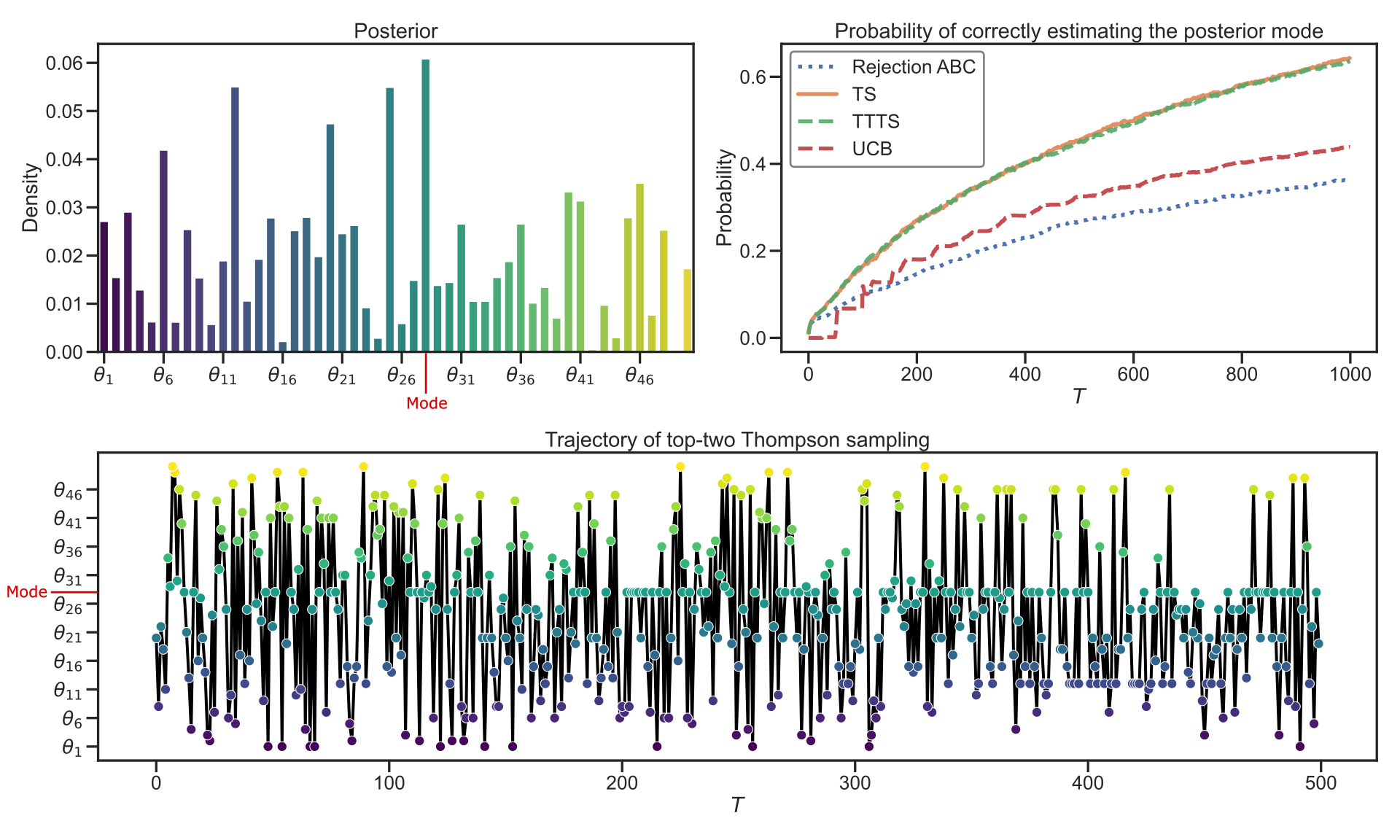}
     %\includegraphics[width=0.45\textwidth]{figs/finitemap-comparison.png}
    %\caption{(a-b) Trajectory plot of the proposed parameters when using Algorithm \ref{alg:map_abc} over $T=100$ rounds. We consider a toy model with a uniform prior and acceptance probabilities given by $\bm{\mu}=(0.15,0.35,0.6,0.25,0.1)^\top$. (c-d) Comparison of methods for likelihood-free MAP estimation in a finite parameter space. We consider a toy model with a uniform prior and acceptance probabilities given by $\bm\mu=(.1, .2, .21, .225, .17)^\top$. \cb Sean, in Example 2, would you like to explain some details of the baseline methods? Very short is enough, but at the moment, I am not sure what they are}
    \caption{Algorithm \ref{alg:map_abc} on a parameter space with   cardinality 50. Upper Left:  the ABC acceptance rate for each parameter (which equals the $\epsilon$-posterior under the uniform prior). Upper Right: Comparison of the average probability of correctly identifying the mode. Lower: Trajectory of the proposed parameters.}
    \label{fig:trajectory-map}
\end{figure}

\subsection{Outer Loop: Adaptive Discretization}

The partitioning approach laid out in Section \ref{sec:outer} may also be applied to the MAP estimation problem. Similarly, we discretize the parameter space into a finite set of bins using a tree-based partitioning scheme.
%As we expect the posterior mode to not change too much between each $\varepsilon$ level, the proposal distribution for maximizing the previous $\varepsilon$-posterior will be a good start for maximizing the subsequent $\varepsilon$-posterior. 
However, unlike for ABC posterior sampling, we prefer finer partitions (such as BART with many trees or CART without too much regularization) that are fine enough to isolate the posterior mode in a small box.
To this end, dyadic partitioning based on sequentially refining the  box $\hat k^{(T)}$  after each Outer Loop round might be preferred. 

 %The partitioning approach laid out in Section \ref{sec:outer} may also be applied to the MAP estimation problem over a generic continuous parameter space. Similarly, we discretize the parameter space into a finite set of bins using a tree-based partitioning scheme, and employ Algorithm \ref{alg:map_abc} to identify the bin that contains the largest average $\varepsilon$-posterior mass. As we expect the posterior mode to not change too much between each $\varepsilon$ level, the proposal distribution for maximizing the previous $\varepsilon$-posterior will be a good start for maximizing the subsequent $\varepsilon$-posterior. Depending on our method of partitioning, as we collect more simulation data over the course of the algorithm, we can identify smaller regions of variation in the ABC posterior with higher certainty.

\begin{figure}
    \centering
    \includegraphics[width=0.8\textwidth]{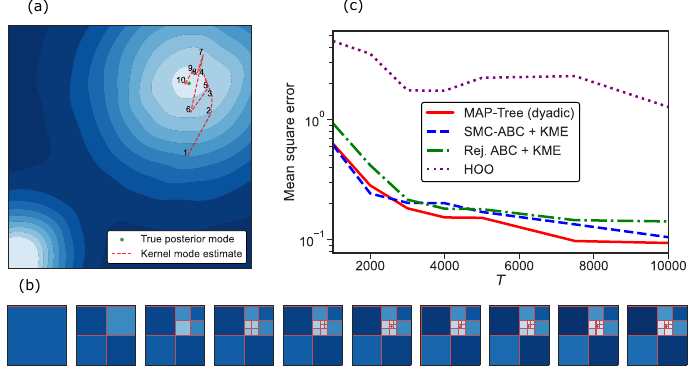}
    \caption{ (a) Evolution of MAP estimates when using Algorithm \ref{alg:seqtreemap}. The posterior mode is estimated via a kernel mode estimate of accepted samples at the final $\varepsilon$ level. (b) The progression of the dyadic partition over the course of the algorithm. (c) Average estimation error over 100 iterations for various methods of estimating the posterior mode. }
    %\caption{(a) Evolution of MAP estimates when using Algorithm \ref{alg:seqtreemap}. The posterior mode is estimated via both a kernel mode estimate, and the center of the bin with larger acceptance rate. (b) Progression of the partition over the course of the algorithm. The color indicates the relative average acceptance rate within each bin.}
    \label{fig:map-progression-dyadic-3split-20k-gauss2d}
\end{figure}

\medskip

The nested structure of Algorithm \ref{alg:seqtreemap} for MAP estimation is very similar to Algorithm \ref{alg:seqtreeabc}, and is presented in Section \ref{sec:ts} in the Supplement. 
At the end of the algorithm, the question remains how to estimate the mode $\bm\theta^*$.  
%This can be done by sampling uniformly from the modal bin, which will better approximate the true mode if the partitioning scheme is amenable to creating many splits and a fine resolution near the mode. 
When the dimensionality of the parameter space is not too large, it is feasible to maximize a kernel density estimate of the accepted samples at the final $\varepsilon$ level to estimate the mode of the continuous posterior distribution. One could also treat the center of the bin $\hat k^{(T)}$ after several Outer Loops as a point estimate of the mode, and the remainder of the box as a ``confidence" region. {In Figure \ref{fig:map-progression-dyadic-3split-20k-gauss2d}, we apply Algorithm \ref{alg:seqtreemap} using dyadic partitioning in the context of {Example \ref{ex:2}}. We can see that the algorithm can closely estimate the true posterior mode with more accuracy than its counterparts. See Section \ref{sec:MAP_fig_details} in the Supplement for more details.}

\section{Experiments}\label{sec:experiemnt}
Recovering an image from an unmasked fraction is an important problem in many machine learning applications \cite{garnelo2018conditional}. In this section, we focus on the uncertainty quantification of the masked image recovery problem on the MNIST\footnote{The MNIST database is a large database of handwritten digits $0\sim 9$ that is commonly used to train various image processing systems. The MNIST database contains $70,000$ $28\times 28$ grey images with the class (label) information, and is available at \url{http://yann.lecun.com/exdb/mnist/}.} dataset \citep{lecun2010mnist} with deep neural network generative models.

We assume a deterministic image generator $G$, e.g., a decoder, s.t. $\bm X = G(\bm\theta)+\bm \epsilon,$ where $\bm\theta\sim N(0,I_{d_{\bm\theta}})$ and $\bm \epsilon\sim N(0, \sigma_0^2I_{d_x})$ for some $d_{\bm\theta}$, $d_x$ and $\sigma^2_0>0$. Here, the randomness of the image $\bm X$ comes from the latent vector $\bm\theta$, our parameter of interest, and the added noise $\bm \epsilon$. As in Figure \ref{fig:masked}, we observe only a part of data $\bm X_{\rm obs}$ of $\bm X=(\bm X_{\rm obs},\bm X_{\rm mask})'.$ By using a posterior sample $\tilde{\bm\theta}\sim \pi(\bm\theta|\bm X_{\rm obs})$, we can generate possible full images via $\wt{\bm X}_{\tilde{\bm \theta}} = G(\tilde{\bm \theta})+\bm \epsilon$, $\bm \epsilon\sim N(0,\sigma_0^2I_{d_x}).$ For our experiments, we have obtained the generator $G$ (decoder) by training an autoencoder with $\sigma_0 = 0.05$ by using the label {(digit) }information similar to \cite{makhzani2015adversarial, tolstikhin2017wasserstein}, {such that each digit number forms a cluster in $\Theta$ (see, Figure \ref{fig:gen-map-resultster-label} (a), and for more details, Section \ref{sec:MNIST_details})}. 
\begin{figure}
    \centering
    \includegraphics[width=\linewidth]{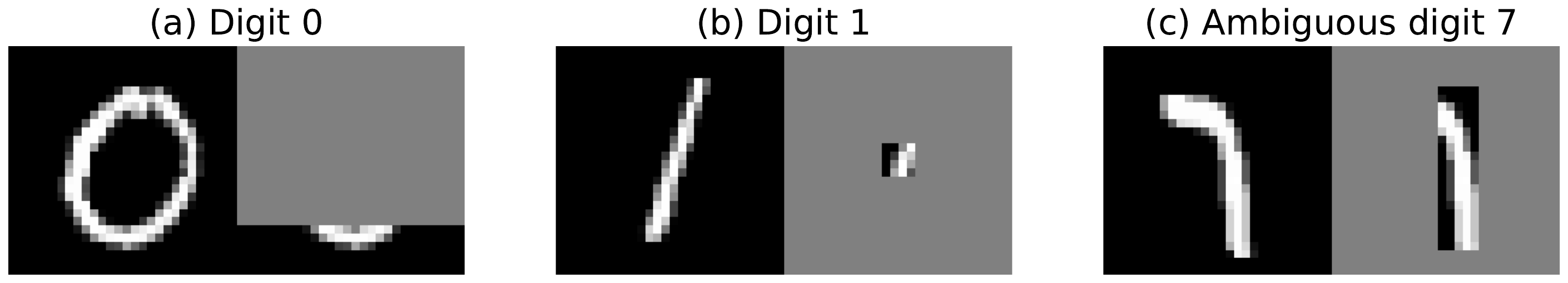}
    \caption{MNIST images and their unmasked portions.}
    \label{fig:masked}
\end{figure}
We experimentally investigate the performance of ABC-Tree {(with $\omega^1$ in \eqref{eq:p_omega} as the utility function)}, comparing it to rejection ABC and SMC-ABC as baselines, all running with a budget off $T=10^5$ simulator evaluations. For MAP  estimation, we maximize a kernel density estimate for the baseline methods to estimate the posterior mode, and compare to MAP-Tree. For SMC-ABC and ABC-Tree (and MAP-Tree), we set the tolerance progression rule $g(\varepsilon)=0.9\varepsilon$ and quota $B=1000$ according to standard heuristics. We note that each method may potentially respond better or worse to different hyperparameter configurations. As a reference distribution, we run rejection ABC for $T=10^8$ samples, regarding the 1000 samples with lowest discrepancies as an approximated true posterior sample. For the partitioning, we use a classification tree (CART, \cite{Breiman1984ClassificationAR}) with maximum number of leaves bounded by 1000. 

%and our MAP algorithm will focus on splitting near those modes. We check these questions in two ways.

\begin{figure}
    \centering
    \includegraphics[width=\linewidth]{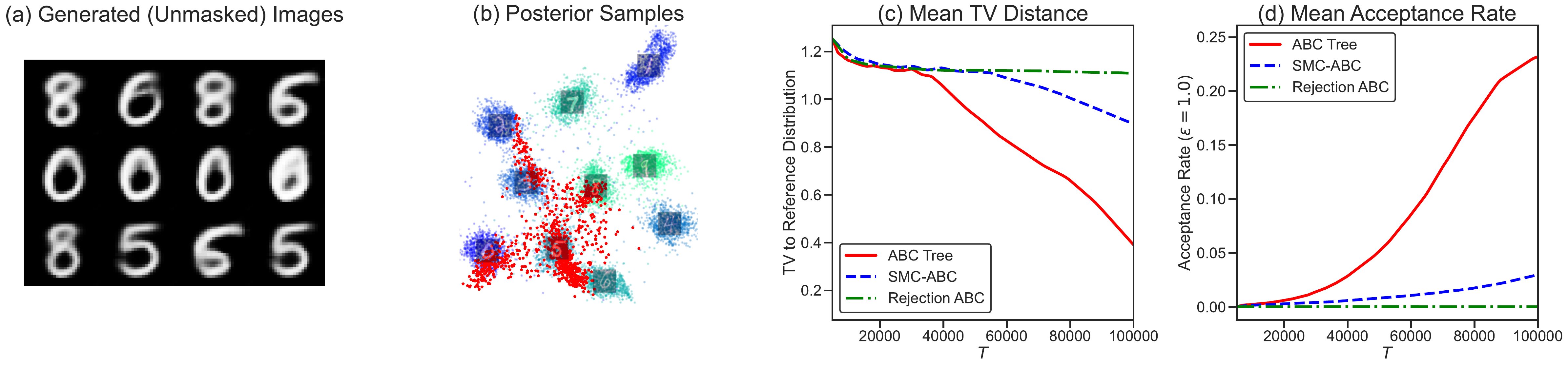}
    \caption{{(a) Generated $\tilde{\bm X}_{\tilde \theta}$ from ABC-Tree with low discrepancies. (b) ABC-Tree posterior samples (red) visualized using iGLoMAP dimension reduction of $\Theta$. (c) ABC-Tree converges quickly to the reference distribution. (d) Comparison of acceptance rates. }}
    \label{fig:exp_thesis} 
\end{figure}

\subsection{Posterior Sampling}\label{sec:post_sampling}
We consider a latent dimension $d_{\bm \theta}=5$, for which the trained autoencoder provides good reconstruction performance. {The prior is set as a uniform distribution over $[-8,8]^{5}$ based on the distribution of encoded data in $\Theta$.} We mask an image with label 0 as shown in Figure \ref{fig:masked} (a), where the visible portion defines $\bm X_{\rm obs}$. The generated examples by ABC-Tree are shown in Figure \ref{fig:exp_thesis} (a). The diversity of these numbers implies that the posterior potentially can be multimodal in a five-dimensional space, which may add difficulty when sampling. To see how the sampling process captures such multimodality, we apply iGLoMAP \citep{iglomap}, a non-linear dimension reduction technique designed to capture both global and local structure in the data. This two-dimensional visualization in Figure \ref{fig:exp_thesis} (b) shows that our method catches multimodality in the posterior. Meanwhile, as a proxy for the distance between the reference distribution and the sampled posteriors, we can compute the empirical categorical distribution on the digits $\{0,...,9\}$, where each generated ${\tilde{\bm\theta}}$ receives an estimated label by a 10-nearest neighbor classifier trained on the encoded space of the MNIST dataset. Note that as shown in Figure \ref{fig:exp_thesis} (c), the distance to the reference categorical distribution decreases more rapidly for ABC-Tree than baselines. Furthermore, as in Figure \ref{fig:exp_thesis} (d), the acceptance rate increases much faster for ABC-Tree ($\varepsilon = 0.81$). The categorical distribution of each method for the final 1000 samples is depicted in Figure \ref{fig:wanted}, where the distribution of ABC-Tree ($T=10^5$) most closely resembles that of rejection ABC with $T=10^8$. In Figure \ref{fig:others_based}, we also show a dimension reduction of the baselines. The final 1,000 samples of each method are displayed in Figure \ref{fig:genmod1} in the supplement. {Section \ref{sec:MNIST_details} provides details on the choice of $d_{\bm\theta}$, the application of iGLoMAP, and the encoded space of MNIST. Section \ref{sec:mnist_additional} presents further empirical investigation of ABC-Tree on MNIST regarding the choice of the utility function and the tree size limits.}

%\begin{figure}
%    \centering
%    \includegraphics[width=0.8\linewidth, height=8cm]{figs/generative-map-fig.pdf}
%    \caption{MAP Estimation of a handwritten digit from an unmasked portion. (a) True digit and unmasked poriton. (b) Final proposal distribution. (c) Rates of correctly estimating the digit. (d). Distribution of MAP estimates from various methods.}
%    \label{fig:gen-map-results}
%\end{figure}

\begin{figure}
    \centering
    \includegraphics[width=\linewidth]{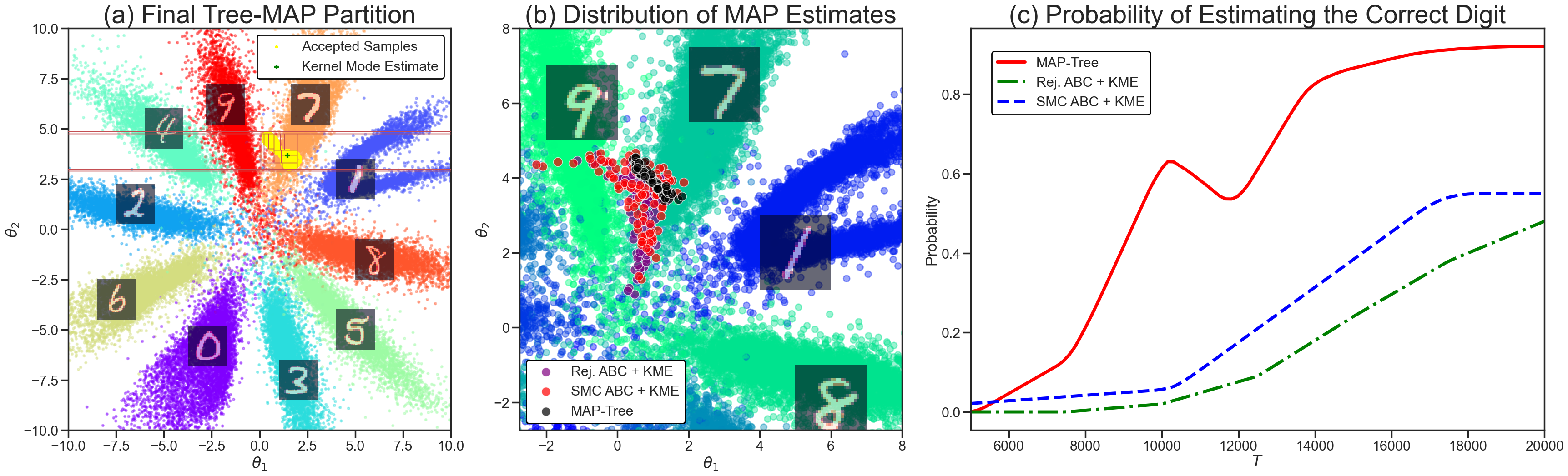}
    \caption{{MAP Estimation of a handwritten digit from an unmasked portion. (a) Final partition and accepted samples from Tree-MAP. (b) Distributions of MAP estimates in $\Theta$ from different methods over repeated trials. (c) Rates of correctly estimating the digit (7).}}
    \label{fig:gen-map-resultster-label}
\end{figure}

\subsection{Likelihood-free MAP estimation}
For MAP estimation, we focus on $d_{\bm\theta}=2$ as it allows for visualization 
of {$\Theta$}, the posterior samples, and the tree splits as well. {The prior is a uniform distribution over $[-12,12]^{2}$ to encompass the distribution of encoded data.} As {$\Theta$ with $d_{\bm\theta}=2$} is not sufficient for embedding the MNIST dataset {(see, Section \ref{sec:MNIST_details})}, the trained generator $G$ can represent only a few stereotypical digit shapes for each number. In this situation, the posterior distribution would resemble a unimodal distribution. 

In order to assess our methodology for MAP estimation, we chose an example image where the digit can plausibly be a member in multiple classes. In Figure \ref{fig:masked} (c), we see an example of a handwritten digit which resembles a 7, though could plausibly resemble a 9, or a 1. A pilot run of rejection ABC with a large amount ($10^8$) of samples and kernel mode estimation confirms that the (approximated) posterior mode lies in a region of {$\Theta$} surrounded by data points with label 7. In Figure \ref{fig:gen-map-resultster-label} (a), we see that the final partition used by MAP-Tree zooms in on highly relevant regions in $\Theta$, which are the most likely to produce masked images resembling our observed data. In (b), we visualize the distribution of MAP estimates in {$\Theta$} for each method over 100 repeated trials, observing a more narrow distribution localized within a region corresponding to the correct digit for MAP-Tree. {From each MAP estimate in $\Theta$, we compute the estimated digit by a 10-nearest neighbor classification trained on the encoded original dataset, as in Section \ref{sec:post_sampling}.} In (c), we see that over 100 repeated trials, the estimated digits resulting from MAP-Tree are classified as the correct digit 7 with significantly higher probability when compared with baseline ABC methods and kernel mode estimation. The full progression of MAP-Tree on this example can be seen in Figure \ref{fig:mapgen1map}.

%{\color{red} How do you define a prior over the finite parameter space $\Theta$? Uniform over what kind of cube? I dont understand the figures in Figure 10. What are the different colors and the different dots? What are the black and red dots on top? Why do you talk about MAP proposal distribution?}

\section{Concluding Remarks}\label{sec:concluding}

In this work, we  developed a new way of approaching likelihood-free inference and MAP estimation by combining tree partitioning with active learning. This allows us to encourage ABC sampling from the most supported regions of the parameter space. ABC-Tree consists of an outer loop (refining a partition of the parameter space) and an inner loop (refining an ABC proposal on a given partition).  We illustrated our approach on a difficult problem of uncertainty quantification over masked image fractions. In Section \ref{sec:additional_exp} in the Supplement, we also provide simulation results on the Heston stochastic volatility model \cite{heston1993closed} and the Lotka-Volterra model. We believe that our development provides a possible avenue for scaling up ABC to high-dimensional settings, and can potentially inspire further research bridging trees and bandit algorithms.%, {\cb Lotka-Volterra Model, and Linear-Nonlinear Encoding Model.}

\setlength{\bibsep}{0pt plus 0.3ex}

\bibliography{sources}

\newpage

%\subsubsection{To be moved to Sec 2.1}

%{\cp Paragraph discussing modeling the likelihood vs posterior. In this paper, we model the likelihood for ABC and the posterior for MAP.}

%\tableofcontents

\appendix
\startcontents[appendix]
\printcontents[appendix]{l}{1}{\section*{Appendix Table of Contents}}
\pagebreak

\section{Proof of Theorems}
\subsection{Proof of Theorem \ref{prop:regretub}}\label{sec:theorem1_proof}

\paragraph{Setting and notation.} Denote the simplex on $K$ elements as $\mathcal{S}^K$. Consider a $K$-armed bandit setup with Bernoulli reward distributions with means $\mu_1,\ldots,\mu_K$. $\bm{\mu}$ represents our ABC acceptance rate, where $\mu_k=\mathbb{E}[ABC(\bm{X},\Omega_k,\varepsilon)]$, and thus also our unnormalized $\varepsilon$-likelihood. Let $\boldsymbol\pi\in\mathcal{S}^K$ be the uniform prior, $\pi_k=1/K$ for $k=1,\ldots,K$. Define 
\[
p_k:=\frac{\pi_k\mu_k}{\sum_{j=1}^K\pi_j\mu_j}\,
\]
for $k=1,\ldots,K$. Thus ${\bm {p}}$ represents the target posterior distribution on our parameter space. Let the function ${\bm {q}}^*:\mathcal{S}^K\to\mathcal{S}^K$ denote the optimal proposal distribution for posterior ${\bm {p}}$ and optimality criterion $\omega_{\bm{p}}$, given by ${\bm {q}}^*({\bm {p}})=\argmax_{{\bm {q}}\in\mathcal{S}^K}\omega_{\bm{p}}({\bm {q}})$. We characterize a problem instance by $\nu=(\bm{\mu}, \bm{\pi})$. We define the regret of algorithm $\mathcal{A}$ on problem instance $\nu$ after $T$ iterations as
\[
R_T(\mathcal{A}, \nu) := \omega_{\bm{p}}({\bm {q}}^*({\bm {p}}))-\frac{1}{T}\sum_{t=1}^T\mathbb{E}\left[\omega_{\bm{p}}\left({\bm {q}}^{(t)}\right)\right]\,.
\]
where the expectation is with respect to the probability measure induced by the interplay of algorithm $\mathcal{A}$ and problem instance $\nu$. The notation mirrors that of the canonical bandit environment defined in \citet{lattimore2020bandit}.

\iffalse
\paragraph{Upper Bound} We consider a slightly modified verison of Algorithm \ref{alg:qe_abc} that employs an initialization phase of size $m=K\gamma^{-1}\log^2(T)$ for each arm. Denoting this algorithm as $\mathcal{A'}$, we have
\[
\sup_{\nu\in\mathcal{I}_\gamma} R_T(\mathcal{A}',\nu) \lesssim \frac{K^2 \log^2(T)}{\gamma T}\,.
\]
where $\mathcal{I}_\gamma$ denotes the class of problem instances with $\|\bm{\mu}\|_1\geq\gamma$ and $\bm{\pi}$ uniform.

The initialization phase of size $m$ for each arm means that we play each arm $m$ times in a row deterministically, before beginning the strategy of sampling from $\hat{\bm{q}}^{(t)}$ in the algorithm. In numerical simulations, such an initialization phase is not necessary, but we use it here for the proof.
\label{thm:upperbound}
\fi

\begin{thm}
    Consider a slightly modified version of Algorithm \ref{alg:qe_abc} that employs an initialization phase by first playing each arm  $m=K\gamma^{-1}\log^2(T)$ times. 
For $Y$ defined in \eqref{eq:mean_reward}, let $\mu_k=P(Y=1\mid \Omega_k)$ be the ABC acceptance rate when sampling from $\Omega_k$. Let $\pi_k=\pi(\Omega_k)=1/K$ be the prior mass on $\Omega_k$
and let $\mathcal{I}_\gamma$ be a set of probability vectors $\bm p=(p_1,\dots,p_K)'$ such that $p_k\propto \pi_k \mu_k$ when $\min\limits_{1\leq k\leq K}\mu_k\geq \gamma>0$. Denoting the modified algorithm as $\mathcal{A'}$, we have (for sufficiently large $T$)
\[
\sup_{\bm p\in\mathcal{I}_\gamma} R_T(\mathcal{A}',\bm p) \lesssim \frac{K^2 \log^2(T)}{\gamma T}\quad\text{and}\quad
\inf_{\mathcal{A}}\sup_{\bm p\in\mathcal{I}_\gamma} R_T(\mathcal{A},\bm p)\gtrsim \frac{K^2}{\gamma\,T},
\]
where the infimum is taken over all possible strategies $\mathcal{A}$ for choosing bins $\Omega_k$.
\end{thm}

\begin{proof} We will first prove the upper bound, and later the lower bound.

\paragraph{Upper Bound}
First, note that $\hat{\bm{p}}^{(t)}\to \bm{p}$ as $t\to\infty$ (and $T\to\infty$). This is because Lemma \ref{lemma:conc} ensures that each arm is played atleast a constant (independent of $T$) fraction of $T$ times, ensuring $\hat{\mu}^{(t)}_k\to\mu_k$ for all $k$ and in turn $\hat{\bm{p}}^{(t)}\to\bm{p}$. This implies that for sufficiently large $T$, there exists $T_0$ such that $t>T_0$ implies $\hat{\bm{p}}\in B_\delta(\bm{p})$. Define $T_1=\max\{T_0, Km\}$. 

We now apply Assumption \ref{ass:2} and decompose the regret into that which occurred up to round $T_1$, and that which occurred after.
    \begin{align*}
        R_T(\mathcal{A'},\bm{p})&=\frac{1}{T}\sum_{t=1}^T \left(\omega_{\bm{p}}\left({\bm {q}}^*({\bm {p}})\right)-\mathbb{E}\left[\omega_{\bm{p}}\left(\bm {q}^*\left(\hat{{\bm {p}}}^{(t)}\right)\right)\right]\right) \\
        &\leq \frac{C'\,K^2\,\log^2(T)}{\gamma\,T} + \frac{C'}{T}\sum_{t=T_1}^T\mathbb{E}\left\|{\bm {q}}^*({\bm {p}})-{\bm {q}}^*\left(\hat{{\bm {p}}}^{(t)}\right)\right\|_2^2
    \end{align*}
    We may bound the second term as
    \begin{align*}
        \frac{C'}{T}\sum_{t=T_1+1}^T\mathbb{E}\left\|{\bm {q}}^*({\bm {p}})-{\bm {q}}^*\left(\hat{{\bm {p}}}^{(t)}\right)\right\|_2^2 &\leq 
        \frac{C''}{T}\sum_{t=T_1+1}^T \mathbb{E}\left\|{\bm {p}}-\hat{{\bm {p}}}^{(t)}\right\|_2^2 \\ 
        &\leq \frac{C'''}{\|\bm{\mu}\|_1^2\,T}\sum_{t=T_1+1}^T \mathbb{E}\left\|\boldsymbol{\mu}-\hat{\boldsymbol{\mu}}^{(t)}\right\|_2^2 \\ 
        &\leq \frac{C'''K}{\|\bm{\mu}\|_1^2\,T}\sum_{t=T_1+1}^T \mathbb{E}\left\|\boldsymbol{\mu}-\hat{\boldsymbol{\mu}}^{(t)}\right\|_2^2 \\ 
        &= \frac{C'''K}{\|\bm{\mu}\|_1^2\,T}\sum_{t=T_1+1}^T \sum_{k=1}^K \mathbb{E}_{T^{(t)}_k}\mathbb{E}_{\hat{\mu}_k^{(t)}\mid T_k^{(t)}}\left[\left(\mu_k-\hat{\mu}_k^{(t)}\right)^2\right] \\
        &\leq \frac{C'''K}{\|\bm{\mu}\|_1^2\,T}\sum_{t=T_1+1}^T \sum_{k=1}^K \mathbb{E}_{T_k^{(t)}}\left[\frac{\mu_k(1-\mu_k)}{T_k^{(t)}}\right]\,,
    \end{align*}
    using Lemma \ref{lem:linalg}, and the fact that $\hat{\mu}_k^{(t)}\mid T_k^{(t)}=n$ is distributed as the sample mean of $n$ Bernoulli random variables with mean $\mu_k$. If $\mu_k=0$, the $k$th term in the sum will be 0 for every $t$. We proceed to consider when $\mu_k>0$, and thus $p_k>0$. In such cases, we have
    \begin{align*}
         & \mathbb{E}_{T_k^{(t)}}\left[\frac{\mu_k(1-\mu_k)}{T_k^{(t)}}\right] \\
         \lesssim\quad & \mathbb{E}_{T_k^{(t)}}\left[\frac{\mu_k(1-\mu_k)}{T_k^{(t)}}\mid T_k^{(t)} > \frac{p_kt}{4}\right]P(T_k^{(t)} > \frac{p_kt}{4}) + \mathbb{E}_{T_k^{(t)}}\left[\frac{\mu_k(1-\mu_k)}{T_k^{(t)}}\mid T_k^{(t)} \leq \frac{p_kt}{4}\right]P(T_k^{(t)} \leq \frac{p_kt}{4})\\
         \lesssim \quad&\frac{\mu_k(1-\mu_k)}{p_k\,t} + P\left(T^{(t)}_k \leq \frac{p_k t}{2}\right) \\
         \lesssim\quad & \frac{\|\bm{\mu}\|_1}{t} + T^{-1}+\exp(-p_kt/64)
    \end{align*}
    using Lemma \ref{lemma:conc}. Therefore, summing over $t$, we have
    \[
    \sum_{t=T_1+1}^T \mathbb{E}_{T_k^{(t)}}\left[\frac{\mu_k(1-\mu_k)}{T_k^{(t)}}\right] \lesssim \|\bm{\mu}\|_1\log(T)\,,
    \]
    for $T$ sufficiently large (in particular, $\log(T)\geq \|\bm{\mu}\|_1^{-1}(\exp(p_k/64)-1)^{-1}$). Therefore, this second term is upper bounded as 
    \[
    \frac{1}{T}\sum_{t=T_1+1}^T\mathbb{E}\left\|{\bm {q}}^*({\bm {p}})-{\bm {q}}^*\left(\hat{{\bm {p}}}^{(t)}\right)\right\|_2^2  \lesssim \frac{K^2\log(T)}{\gamma\,T}
    \]
    and thus
    \[
    R_T(\mathcal{A}',\bm{p})\lesssim \frac{K^2\,\log^2(T)}{\gamma\,T}
    \]
    as desired.

    %\begin{align*}
    %    &\leq \frac{C''''K}{\|\bm{\mu}\|_1^2\,T}\sum_{t=1}^T \sum_{k=1}^K \frac{\mu_k(1-\mu_k)}{p_k\,t} \\
    %    &\leq \frac{C'''''K}{\|\bm{\mu}\|_1^2\,T}\sum_{t=1}^T \sum_{i=1}^K \frac{(1-\mu_k)\|\bm{\mu}\|_1}{t} \\
    %    &\leq \frac{C''''''\,K^2}{\|\bm{\mu}\|_1 T}\log(T)\,. 
    %\end{align*}
    %\begin{align*}
    %    \frac{C'''K}{\|\bm{\mu}\|_1^2\,T}\sum_{t=1}^T \sum_{k=1}^K \mathbb{E}_{T_k^{(t)}}\left[\frac{\mu_k(1-\mu_k)}{T_k^{(t)}}\right] &\leq \frac{C'''K}{\|\bm{\mu}\|_1^2\,T}\sum_{t=1}^T \sum_{k=1}^K \left(\frac{\mu_k(1-\mu_k)}{p_k\,t} + P\left(T^{(t)}_k \leq \frac{p_k t}{2}\right)\right) \\
    %    &\leq \frac{C''''K}{\|\bm{\mu}\|_1^2\,T}\sum_{t=1}^T \sum_{k=1}^K \frac{\mu_k(1-\mu_k)}{p_k\,t} \\
    %    &\leq \frac{C'''''K}{\|\bm{\mu}\|_1^2\,T}\sum_{t=1}^T \sum_{i=1}^K \frac{(1-\mu_k)\|\bm{\mu}\|_1}{t} \\
    %    &\leq \frac{C''''''\,K^2}{\|\bm{\mu}\|_1 T}\log(T)\,. 
    %\end{align*}
    %Using the fact that $\hat{\mu}_k^{(t)}\mid T_k^{(t)}=n$ is distributed as the sample mean of $n$ Bernoulli random variables with mean $\mu_k$. In the second to last inequality, we used Lemma \ref{lem:helper1} to de Thus, we conclude
    %\[
    %R_T \lesssim \frac{K^2\log(T)}{\gamma T} \,.
    %\]

\paragraph{Lower bound} First consider the instance $\bm{p}$ in $\mathcal{I}_\gamma$ given by $p_k\propto \pi_k\mu_k$, where $\bm{\pi}$ is the uniform prior $\boldsymbol\pi=(1/K,\ldots,1/K)^\top$ and the mean acceptance rates are given by $\boldsymbol\mu = (\gamma/K,\ldots,\gamma/K)^\top$. Now fix any algorithm $\mathcal{A}$, and choose $i_0 = \argmin_{j=1,\ldots,K} \mathbb{E}_{\boldsymbol\mu}[T_j^{(T)}]$. Because $\sum_{j=1}^K \mathbb{E}_{\boldsymbol\mu}[T_j^{(T)}] = T$, we know that necessarily
    \[
    \mathbb{E}_{\boldsymbol\mu}[T_{i_0}^{(T)}]\leq T/K.
    \]
    We then define a second instance $\bm{p}'$ in $\mathcal{I}_\gamma$ where $p_k'\propto \mu_k'\pi_k$, where 
    \[
    \boldsymbol\mu'=\left(\frac{\gamma-\Delta/(K-1)}{K},\ldots,\frac{\gamma-\Delta/(K-1)}{K},\frac{\gamma+\Delta}{K},\frac{\gamma-\Delta/(K-1)}{K},\ldots,\frac{\gamma-\Delta/(K-1)}{K}\right)^\top\,,
    \]
    adding a small $\Delta/K$ to the $i_0$th position in the vector, and slightly decreasing all other coordinates proportionately to maintain a norm of $\gamma$. We denote $\bm{p}=\bm{\mu}/\|\bm{\mu}\|_1$ and $\bm{p'}$ similarly. As a result of Assumption \ref{ass:2}, letting $\hat{\bm{q}}^{(t)}$ denote the $t^\mathrm{th}$ proposal employed by $\mathcal{A}$, we have
    \begin{align*}
    & R_T(\mathcal{A},\bm{p}) + R_T(\mathcal{A}, \bm{p}') \\
    =\quad& \sum_{t=1}^T \left(\omega_{\bm{p}}({\bm {q}}^*({\bm {p}}))-\mathbb{E}\left[\omega_{\bm{p}}(\hat{\bm {q}}^{(t)})\right] + \omega_{\bm{p}'}(\bm{q}^*(\bm {p}'))-\mathbb{E}\left[\omega_{\bm{p}'}(\hat{\bm {q}}^{(t)})\right]\right) \\ 
    \gtrsim \quad& \mathbb{E}\|{\bm {q}}^*({\bm {p}})-\hat{\bm {q}}^{(t)}\|_2^2 + \frac{1}{T}\sum_{t=1}^T \mathbb{E}\|{\bm {q}}^*({\bm {p}}')-\hat{\bm {q}}^{(t)}\|_2^2 \\
    \gtrsim \quad& \frac{1}{T}\sum_{t=1}^T\|{\bm {p}} - {\bm {p}}'\|_2^2\,.
     \end{align*}
    This can be lower bounded as 
    \begin{align*}
    \|{\bm {p}} - {\bm {p}}'\|_2^2 &= \frac{\Delta^2}{\gamma^2}+\frac{\Delta^2}{\gamma^2(K-1)} \gtrsim \frac{\Delta^2}{\gamma^2}.
    \end{align*}
    Let $\mathbb{P}_{\bm{p}, \mathcal{A}}$ and $\mathbb{P}_{\bm{p}', \mathcal{A}}$ be the probability measures on the data resulting from of the interplay of algorithm $\mathcal{A}$ and reward distributions $\bm{p}$, $\bm{p}'$ respectively. We upper bound the KL-divergence between these measures as
    \begin{align*}
    D_{\mathrm{KL}}(\mathbb{P}_{\bm{p},\mathcal{A}} \,\|\,\mathbb{P}_{\bm{p}',\mathcal{A}}) &= \sum_{k=1}^K \mathbb{E}_{\bm{p}}[T_{k}^{(T)}]\,D_\mathrm{KL}(\mathrm{Ber}(\mu_k) \| \mathrm{Ber}(\mu'_k)) \\
    &\leq \frac{T}{K} D_\mathrm{KL}(\mathrm{Ber}(\mu_{i_0}) \| \mathrm{Ber}(\mu'_{i_0})) + T D_\mathrm{KL}(\mathrm{Ber}(\mu_{j}) \| \mathrm{Ber}(\mu'_{j})) \text{ for $j\neq i_0$}\\
    &\leq \frac{T}{\gamma}\left(\frac{\Delta}{K}\right)^2 + \frac{TK}{\gamma}\left(\frac{\Delta}{K}\right)^2 \\ 
    &\lesssim \frac{T \Delta^2}{K^2 \gamma}\,.
    \end{align*}
    Choosing $\Delta\asymp K\sqrt{\frac{\gamma}{T}}$ thus implies that the two bandit environments $\bm{p}$ and $\bm{p}'$ are indistinguishable. This implies via a two-point argument that the regret has lower bound
    \[
    \inf_{\bm{p}}\sup_{\mathcal{A}} R_T(\mathcal{A},\bm{p}) \gtrsim \frac{K^2}{\gamma T}\,.
    \]
\end{proof}

\begin{lemma}\label{lem:linalg}
For any nonzero vectors $\bm{u},\bm{v}\in\mathbb{R}^K$, we have
\[
\left\|\frac{\bm{u}}{\|\bm{u}\|_1}-\frac{\bm{v}}{\|\bm{v}\|_1}\right\|_2 \leq \frac{1+\sqrt{K}}{\|\bm{u}\|_1}\|\bm{u}-\bm{v}\|_2
\]
\end{lemma}
\begin{proof}
We have
\begin{align*}
\left\|\frac{\bm{u}}{\|\bm{u}\|_1}-\frac{\bm{v}}{\|\bm{v}\|_1}\right\|_2 &\leq \left\|\frac{\bm{u}}{\|\bm{u}\|_1}-\frac{\bm{v}}{\|\bm{u}\|_1}\right\|_2 + \left\|\frac{\bm{v}}{\|\bm{u}\|_1}-\frac{\bm{v}}{\|\bm{v}\|_1}\right\|_2 \\
&\leq \frac{1}{\|\bm{u}\|_1}\left\|\bm{u}-\bm{v}\right\|_2 + \|\bm{v}\|_2\left|\frac{1}{\|\bm{u}\|_1}-\frac{1}{\|\bm{v}\|_1}\right| \\
&\leq \frac{1}{\|\bm{u}\|_1}\left\|\bm{u}-\bm{v}\right\|_2 + \sqrt{K}\left|\frac{\|\bm{v}\|_1}{\|\bm{u}\|_1}-1\right| \\
&\leq \frac{1}{\|\bm{u}\|_1}\left\|\bm{u}-\bm{v}\right\|_2 + \frac{\sqrt{K}}{\|\bm{u}\|_1}\left\|\bm{u}-\bm{v}\right\|_1 \\
&\leq \frac{1+\sqrt{K}}{\|\bm{u}\|_1}\|\bm{u}-\bm{v}\|_2
\end{align*}
\end{proof}

\begin{lemma}\label{lem:helper1}
After an initialization phase of size $m=K\gamma^{-1}\log^2(T)$ for each arm, for all $t>Km$, we have
\begin{equation}
%P\left(\bigcup_{i=1}^K\left\{T_i \geq \frac{3p_i\,T}{8}\right\}\right) \geq 1- \frac{c\,K}{T}
P\left(T_k^{(t)} \leq \frac{p_k\,t}{2}\right)  \leq T^{-1}+\exp(-p_kt/64)\,.
\label{eq:conc}
\end{equation}
\label{lemma:conc}
\end{lemma}
\begin{proof}
First, fix $k$ and fix $t\geq Km$. Let $T_k^{(t)}$ denote the (random) number of times arm $k$ was played up to time $t$. Due to initialization, we have $T_k^{(t)}\geq m$ with probability 1. We first bound the tail probability as
    \begin{align*}
    P\left(\hat{p}_{k}^{(t)} \leq \frac{p_k}{4}\,\mid\,T_k^{(t)}=t_k\right) &= P\left( \frac{\hat{\mu}_k^{(t)}}{\sum_{j=1}^K \hat{\mu}_j^{(t)}} \leq \frac{\mu_k}{2\sum_{j=1}^K \mu_j}\right) \\ 
    &= P\left(2\hat{\mu}_k^{(t)}\sum_{j=1}^K \mu_j \leq \mu_k \sum_{j=1}^K \hat{\mu}_j^{(t)}\right) \\
    &= P\left(\frac{1}{t_k}\sum_{s=1}^{t_k} \left(2(\sum_{j=1}^K \mu_j)X_{ks} - \mu_k \sum_{j=1}^K X_{js}\right)\leq 0 \right).
    \end{align*}
\noindent Defining
    \[
    Z_{ks}:=2(\sum_{j=1}^K \mu_j)X_{ks} - \mu_k \sum_{j=1}^K X_{js},
    \]
    we observe $Z_{ks}$ is a bounded random variable with $-K\mu_k \leq Z_{ks} \leq 2\sum_{j=1}^K \mu_j$, and $\mathbb{E}[Z_{ks}]=\mu_k\sum_{j=1}^K \mu_j$. Thus, a Chernoff-Hoeffding bound yields 
    \begin{align*}
    P\left(\frac{1}{T_k^{(t)}}\sum_{s=1}^{T_k^{(t)}} Z_{ks} - \mu_k \sum_{j=1}^K \mu_j \leq -\mu_k\sum_{j=1}^K \mu_j\right) &\leq \exp\left(-\frac{2m(\mu_k\sum_{j=1}^K \mu_j )^2}{(\sum_{j=1}^K 2\mu_j + \mu_k)^2}\right) \\
    &\leq \exp\left(-\frac{m}{2}\mu_k^2\right) \\
    &\leq \exp\left(-\frac{K\gamma^{-1}\log^2(T)}{2}\mu_k^2\right) \\
    &\leq T^{-1}
    \end{align*}
    for sufficiently large $T$ (in particular, for $T\geq \exp(4\gamma/(K\mu_k^2))$). 
    Thus 
    \[
    P\left(\hat{p}_k^{(t)} \leq \frac{p_k}{2}\,\mid\,T_k^{(t)} =t_k \right)\leq T^{-2}\,.
    \]
    Since this holds for all $t_k\geq m$, and $T_k^{(t)}\geq m$ for all $t>Km$ with probability 1, we may also conclude
    \[
    P\left(\hat{p}_k^{(t)} \leq \frac{p_k}{2}\right)\leq T^{-2}\,.
    \]
    for all $t> Km$. 
    
    \noindent Let $\mathcal{E}=\bigcap_{s=m+1}^t\left\{\hat{p}_k^{(s)}\geq p_k/2\right\}$, and note that $P(\mathcal{E}^c)\leq T^{-1}$. Next, we bound
    \begin{align*}
    P\left(T_k^{(t)} \leq \frac{p_k\,t}{4}\,\mid\,\mathcal{E}\right) &\leq P\left(T_k^{(t)} \leq \frac{p_k\,t}{4}\right) \\
    &\leq P\left(T_k^{(t)} -m\leq \frac{p_k\,t}{4}\right) \\
    &\leq P\left(\sum_{s=m+1}^t W_s \leq \frac{p_k\,t}{4} \right) \\ 
    &\leq \exp\left(-\frac{p_k(t-m)}{4}\left(\frac{t-2m}{2(t-m)}\right)^2\right) \\ 
    &\leq \exp\left(-\frac{p_k\,t}{64}\right) \,.
    \end{align*}
    where $W_1,\ldots,W_t \overset{\mathrm{iid}}{\sim} \mathrm{Bernoulli}(p_k/2)$. 
    Finally, we combine these results to yield
    \begin{align*}
        P\left(T_k^{(t)}\leq \frac{p_kt}{4}\right) &= P\left(T_k^{(t)}\leq \frac{p_kt}{4}\,\mid\,\mathcal{E}\right)P(\mathcal{E}) + P\left(T_k^{(t)}\leq \frac{p_kt}{4}\,\mid\,\mathcal{E}^c\right)P(\mathcal{E}^c) \\
        &\leq P\left(T_k^{(t)}\leq \frac{p_kt}{4}\,\mid\,\mathcal{E}\right) + P(\mathcal{E}^c) \\
        &\leq T^{-1}+\exp(-p_kt/64)\,.
    \end{align*}
    %Combining the above, we have
    %\[
    %P\left(\bigcup_{i=1}^K\left\{T_i \geq \frac{3p_i\,T}{8}\right\}\right) \geq 1- \frac{c\,K}{T}
    %\]
    %for some constant $c>0$.
    
\end{proof}

\subsection{Special Case of Utility Function} \label{sec:special_case}
Here, we study the utility of a proposal distribution ${\bm {q}}$ by its sampling efficiency \cite{alsing2018optimal}
\[
\omega_{\bm{p}}({\bm {q}}):=\frac{\sum_{k=1}^K q_kp_k/\pi_k}{\sum_{k=1}^K p_k\pi_k/q_k}\,,
\]
where terms in the denominator are defined to be $0$ whenever $p_k=0$. Because of this, it is clear that $\bm{q}^*(\bm{p})$ will set $q_k=0$ for each zero $p_k$, and it suffices to consider only the support of $\bm{p}$ when $\bm{p}$ has zero entries.

We show that Assumption \ref{ass:2} holds for this particular choice of $\omega_{\bm{p}}$ under a uniform prior $\bm{\pi}=(1/K,\ldots,1/K)'$, with a modified family discrete probability measures $\mathcal{I}_\gamma'$ that imposes the additional restriction that each nonzero $p_k$ is bounded below by a universal constant $c>0$. 
\begin{lemma}
    For any posterior ${\bm {p}}\in\mathcal{S}^K$ there exist constants $C_{\bm{p}}>0$, $\delta>0$ such that for any ${\bm {q}}\in\mathcal{S}^K\cap B_\delta({\bm {q}}^*({\bm {p}}))$,
    \[
    \omega_{\bm{p}}({\bm {q}}^*({\bm {p}})) - \omega_{\bm{p}}({\bm {q}}) \leq C_{\bm{p}} \|{\bm {q}}^*({\bm {p}})-{\bm {q}}\|_2^2
    \]
\label{lemma:assumption1_1}
\end{lemma}
\begin{proof}
    Note $\omega_{\bm{p}}$ is twice continuously differentiable in ${\bm {q}}$, and $\omega_{\bm{p}}$ is maximized at ${\bm {q}}^*({\bm {p}})$ subject to the constraint
    \[
    c(\mathbf{x}):=1-\sum_{k=1}^K x_k = 0\,.
    \]
    Expanding $\omega_{\bm{p}}$ around some point $\hat{{\bm {q}}}\in \mathcal{S}^K\cap B_{\delta}({\bm {q}}^*({\bm {p}}))$ yields
    \begin{align*}
    \omega_{\bm{p}}(\hat{{\bm {q}}}) = \omega_{\bm{p}}({\bm {q}}^*({\bm {p}}))\ &+ \nabla \omega_{\bm{p}}({\bm {q}}^*({\bm {p}}))^\top (\hat{{\bm {q}}}-{\bm {q}}^*({\bm {p}})) \\ &+ (\hat{{\bm {q}}}-{\bm {q}}^*({\bm {p}}))^\top\nabla^2 \omega_{\bm{p}}({\bm {q}}^*({\bm {p}}))\,(\hat{{\bm {q}}}-{\bm {q}}^*({\bm {p}})) \\ &+ o(\|\hat{{\bm {q}}}-{\bm {q}}^*({\bm {p}})\|^2)
    \end{align*}
    Second order necessary conditions for constrained optimization imply that $\omega_{\bm{p}}$
    satisfies $\nabla \omega_{\bm{p}}({\bm {q}}^*({\bm {p}}))^\top (\hat{{\bm {q}}}-{\bm {q}}^*({\bm {p}}))=0$ and $(\hat{{\bm {q}}}-{\bm {q}}^*({\bm {p}}))^\top\nabla^2 \omega_{\bm{p}}({\bm {q}}^*({\bm {p}}))\,(\hat{{\bm {q}}}-{\bm {q}}^*({\bm {p}}))\leq 0$ for any $\hat{{\bm {q}}}\in\mathcal{S}^K\cap B_{\delta}({\bm {p}}^*({\bm {p}}))$. This yields the desired result, as 
    \[
    \lvert \omega_{\bm{p}}({\bm {q}}^*({\bm {p}}))-\omega_{\bm{p}}(\hat{{\bm {q}}})\rvert \leq C_{\bm{p}} \|{\bm {q}}^*({\bm {p}})-\hat{{\bm {q}}}\|^2
    \]
    for $\hat{{\bm {q}}}\in \mathcal{S}^K \cap B_\delta({\bm {q}}^*({\bm {p}}))$. The $\bm{p}$ dependent constant may be replaced with a universal constant by taking the supremum, which is bounded above due to the assumption $p_k>c$.
\end{proof}

\begin{lemma}
    For any fixed $\boldsymbol{\pi}$, the function ${\bm {q}}^*({\bm {p}})$ is locally Lipschitz in ${\bm {p}}$. That is, there exist constants $L, \delta_2$ such that 
    \[
    \|{\bm {q}}^*({\bm {p}})-{\bm {q}}^*({\bm {p}}')\|_2 \leq L \|{\bm {p}}-{\bm {p}}'\|_2
    \]
    for all ${\bm {p}}'\in \mathcal{S}^K\cap B_{\delta_2}({\bm {p}})$.
    \label{lemma:assumption2_1}
\end{lemma}

\begin{proof}
    First, we note that in \cite{alsing2018optimal}, it is shown that ${\bm {q}}^*({\bm {p}})$ can be expressed by the implicit equation
    \begin{equation}
    \bm{q}^*({\bm {p}})_k = {\sqrt{\frac{p_k\pi_k}{2A({\bm {q}}^*({\bm {p}}))-\frac{p_k}{\pi_k}}}}{\Bigg/}{\sum_{j=1}^K \sqrt{\frac{p_j\pi_j}{2A({\bm {q}}^*({\bm {p}}))-\frac{p_j}{\pi_j}}}}
    \label{eq:qstarimplicit}
    \end{equation}
    where $A:\mathcal{S}^K\to\mathbb{R}$ is given by
    \[
    A({\bm {q}}):=\sum_{k=1}^K \frac{q_k p_k}{\pi_k}\,.
    \]
    Thus, the functional form for the optimal proposal ${\bm {q}}^*({\bm {p}})$ is known explicitly except for the value of the scalar $A:=A({\bm {q}}^*({\bm {p}}))$, and given ${\bm {p}},\bm{\pi}$, it is uniquely determined by $A$. Rearranging, we can express $A$ with the implicit equation
    \begin{align*}
    A&=\sum_{k=1}^K\frac{q^*({\bm {p}})_kp_k}{\pi_k}= {\sum_{k=1}^K\frac{p_k}{\pi_k}\sqrt{\frac{p_k\pi_k}{2A-\frac{p_k}{\pi_k}}}}\Bigg/{\sum_{j=1}^K \sqrt{\frac{p_j\pi_j}{2A-\frac{p_j}{\pi_j}}}}\,,
    \end{align*}
    equivalently written as 
    \[
    g({\bm {p}},A):=\sum_{k=1}^K \left(A-\frac{p_k}{\pi_k}\right)\sqrt{\frac{p_k\pi_k}{2A-\frac{p_k}{\pi_k}}}=0\,.
    \]
    Note that $g\in\mathcal{C}^2(A)$, and observe that 
    \[
    \frac{\partial g}{\partial A}=\sum_{k=1}^K \left(\frac{A-\frac{p_k}{\pi_k}}{2A-\frac{p_k}{\pi_k}}\right)\sqrt{\frac{p_k\pi_k}{2A-\frac{p_k}{\pi_k}}}
    \]
    which is always positive. Thus, the implicit function theorem yields that for any fixed ${\bm {p}}\in\mathcal{S}^K$, there exists $\delta_1>0$ and a unique twice continuously differentiable function $A^*:\mathcal{S}^K\cap B_{\delta_1}({\bm {p}}^*)\to\mathbb{R}$ such that $A^*({\bm {p}})=A({\bm {q}}^*({\bm {p}}))$. 
    
    If we now denote ${\bm {q}}^{\circ}(A,{\bm {p}})$ to be a function defined by Equation \ref{eq:qstarimplicit} with $A$ replacing $A({\bm {q}}^*({\bm {p}}))$, we can note that ${\bm {q}}^{\circ}(A,{\bm {p}})$ is twice continuously differentiable in $A$ and $\bm{p}$, and that ${\bm {q}}^*({\bm {p}})={\bm {q}}^{\circ}(A^*({\bm {p}}), {\bm {p}})$. Therefore, the function ${\bm {q}}^*({\bm {p}})$ is twice continuously differentiable for ${\bm {p}}\in B_{\delta_1}({\bm {p}}^*)$. Thus, we have for any ${\bm {p}}'\in \mathcal{S}^K\cap B_{\delta_1}({\bm {p}})$,
    \[
    \|{\bm {q}}^*({\bm {p}})-{\bm {q}}^*({\bm {p}}')\|_2 \leq L_{\bm{p}} \|{\bm {p}}-{\bm {p}}'\|_2\,.
    \]
    The $\bm{p}$-dependent constant may be replaced with a universal constant by taking the supremum, once again as $p_k>c$.
\end{proof}

\begin{lemma}
    There exist constants $C, \delta_2, \delta_3$ such that for any ${\bm {p}}\in\mathcal{S}^K\cap B_{\delta_2}(\mathbf{u})$ and ${\bm {q}}\in\mathcal{S}^K\cap B_{\delta_2}({\bm {q}}^*({\bm {p}}))$, we have
    $\omega({\bm {q}}^*({\bm {p}}),{\bm {p}},\mathbf{u}) - \omega({\bm {q}},{\bm {p}},\mathbf{u}) \geq C \|{\bm {q}}^*({\bm {p}})-{\bm {q}}\|_2^2$.
    \label{lemma:negativedefinite}
\end{lemma}
\begin{proof}
    We will show that the function $g_{\bm {p}}({\bm {q}})$, defined by 
    \begin{align*}
    g_{\bm {p}}({\bm {q}}) &= \frac{\sum_{k=1}^K q_kp_k}{\sum_{k=1}^K{p_k/q_k}} \propto \omega_{\bm{q}}(\bm{p})
    \end{align*}
    has a negative definite Hessian for ${\bm {p}}$ sufficiently close to $(1/K,\ldots,1/K)'$. A short calculation yields that 
    \[
    \nabla^2_{\bm {q}} \, g_\mathbf{u}(\mathbf{u})=4\mathbf{1}_K\mathbf{1}_K^\top - 2K\,I_K
    \]
    and for any $\mathbf{z}\in\mathbb{R}^K\setminus \{\mathbf{0}\}$ such that $\sum_{k=1}^K z_k=0$, we have
    \begin{align*}
        \mathbf{z}^\top \left(\nabla_{\bm {q}}^2 g_\mathbf{u}(\mathbf{u})\right) \mathbf{z} &= 4\mathbf{z}^\top \mathbf{1}_K \mathbf{1}_K^\top \mathbf{z} - 2K\mathbf{z}^\top \mathbf{z} \\ 
        &= -2K\|\mathbf{z}\|_2^2 < 0\,.
    \end{align*}
    This implies that there exists a constant $C$ such that
    \[
    \omega_{\bm{u}}(\mathbf{u})-\omega_{\bm{u}}({\bm {q}}) \geq C \|\mathbf{u}-{\bm {q}}\|_2^2\,.
    \]
    To make this hold for ${\bm {p}}$ in a neighborhood around $\mathbf{u}$, we note that the function $g_{\bm {p}}({\bm {q}})$ is twice continuously differentiable with respect to ${\bm {p}}$ around $\mathbf{u}$, which implies that 
    \[
    \mathbf{z}^\top \left(\nabla^2_{\bm {q}} g_{\bm {p}}({\bm {q}}^*({\bm {p}}))\right)\mathbf{z} < 0
    \]
    for all ${\bm {p}}\in \mathcal{S}^K\cap B_{\delta_2}(\mathbf{u})$ for some $\delta_2>0$. Thus, there exists some $\delta_3$, such that for all ${\bm {p}}\in \mathcal{S}^K\cap B_{\delta_2}(\mathbf{u})$, for all ${\bm {q}} \in \mathcal{S}^K \cap B_{\delta_3}({\bm {q}}^*({\bm {p}}))$, we have
    \[
    \omega({\bm {q}}^*({\bm {p}}),{\bm {p}},\mathbf{u}) - \omega({\bm {q}},{\bm {p}},\mathbf{u}) \geq C \|{\bm {q}}^*({\bm {p}})-{\bm {q}}\|_2^2\,.
    \]
\end{proof}

\begin{lemma}
    The function ${\bm {q}}^*({\bm {p}})$ is locally invertible around $\mathbf{u}$, and thus the inverse is locally Lipschitz. 
    \label{lemma:inverselips}
\end{lemma}
\begin{proof}
    We consider the implicit function
    \[
    g_i({\bm {p}},{\bm {q}}^*):=(q^*_i)^2 - \frac{\omega({\bm {q}}^*)p_i\pi_i}{2A({\bm {q}}^*)^2 - \frac{p_i}{\pi_i}A({\bm {q}}^*)}
    \]
    observing that calculation yields
    \[
    J_{\mathbf{g},{\bm {q}}^*}(\mathbf{u}, \mathbf{u})=\frac{2}{K}I - \left(3+\frac{2}{K^2}\right)\mathbf{1}\mathbf{1}^\top,~{\rm and}~J_{\mathbf{g}, {\bm {p}}}(\mathbf{u}, \mathbf{u})=-2k\,I + \left(3-\frac{2}{K}\right)\mathbf{1}\mathbf{1}^\top.
    \]
     Note that both are invertible. Applying the implicit and inverse function theorems yields the existence of $\delta_4>0$ such that for any ${\bm {p}}\in \mathcal{S}^K\cap B_{\delta_4}(\mathbf{u})$, the (implicitly defined) function ${\bm {q}}^*({\bm {p}})$ is bijective. Furthermore, since we argued ${\bm {q}}^*(\cdot)$ is locally Lipschitz with constant $L$, then there exists $\delta_5$ such that for $\bm{p}\in B_{\delta_5}(\bm{u})$, the inverse function of $\bm{q}^*$ is Lipschitz with constant $1/L$.
\end{proof}

Lemma \ref{lemma:assumption1_1} and Lemma \ref{lemma:assumption2_1} assert that the first two parts of the assumption are satisfied. Combining Lemma \ref{lemma:inverselips} and Lemma \ref{lemma:negativedefinite} yields that the third part of Assumption \ref{ass:2} is satisfied.

\subsection{Proof of Theorem \ref{thm:map_rate}} \label{sec:MAP_proof}%{\cb In this case, we can define $\pi_k = \pi(\bm{\theta_k})$ and $\ell_k =P\{ {\rm ABC}(\bm X,\Omega_k,\varepsilon)=1\}$. Thus, region $k$ will yield a reward of $\pi_k$ with probability $\ell_k$, and the mean reward of region $\Omega_k$ equates to $\pi_k \ell_k$, which is proportional to the $\varepsilon$-posterior mass on region $\Omega_k$ as $\pi_k\ell_k \propto {\cb C_{\bm \theta_k}} p_\varepsilon(\bm{X}\mid\bm\theta_k)$. In this case, the goal is to identify the modal region $\Omega_{k^*}:=\argmax_{k=1,\ldots,K} \pi_k\ell_k$, corresponding to the region that yields the largest mean reward. }
%
%It is a consequence of the analysis in \citet{russo2016simpleBayesianAlgorithms}, due to $\mathbb{E}[Y_\varepsilon|\bm \theta] \propto p_\varepsilon(\bm \theta\mid\bm{X})$. 
\paragraph{Setting} Consider a problem instance defined by vector of ABC acceptance rates $\bm{\ell}$ and prior distribution $\bm{\pi}\in\mathcal{S}^K$. Thus, we have the ABC posterior defined via $p_k\propto \ell_k\pi_k$ for $k=1,\ldots,K$. Let $k^*=\argmax_k p_k$ and let $\hat{k}^{(T)}_\mathcal{A}$ be the estimate of the maximal $k$ after $T$ simulations using Algorithm $\mathcal{A}$. Denote Algorithm \ref{alg:map_abc} with the choice $b=1/2$ as $\mathcal{A}'$.

There exists an instance dependent constant $\Gamma_{\bm{\mu}}$ such that $P(\hat{k}^{(T)}_{\mathcal{A}'}\neq k^*)=\exp(-T(\Gamma_{\bm{\mu}} + o(1)))$, where $o(1)$ denotes a vanishing sequence as $T\to\infty$. 

\begin{proof}
The likelihood-free MAP estimation setting with a uniform prior on a finite parameter space can be characterized as an instance of the best-arm identification problem in multi-armed bandits. Each arm is a parameter in the parameter space, with reward distribution given by $\mathrm{Bernoulli}(\ell_k)$, where $\ell_k$ is the mean ABC acceptance rate when proposing parameters from parameter $k$. Since the ABC acceptance rate is proportional to the $\varepsilon$-likelihood (Lemma \ref{lem:crucial}), identifying the arm with maximal mean reward in the bandit setup is equivalent to identifying the parameter value with maximal $\varepsilon$-likelihood. 

Theorem 1 of \citet{russo2016simple} may then be applied to obtain the desired result.

\end{proof}

%\paragraph{Lower Bound}
%Moreover, for any algorithm $\mathcal{A}$ on problem instance $\bm{\mu}$, we must necessarily sustain $P(\hat{k}^{(T)}_\mathcal{A} \neq k^*) \geq \exp(-T(2\Gamma_{\bm{\mu}}+o(1)))$.

\section{Implementation Details}

\subsection{Empirical Regret Comparison}
\label{implementation:empregretcomp}

We consider a toy example on a discrete parameter space $\theta\in\{\theta_1,\dots,\theta_{100}\}'$, where  we treat each plausible  parameter value $\theta_k$ 
  as a bin $\Omega_k$. {The ABC likelihood is defined via $U\mid\theta_j\sim\mathcal{U}(j-1,j)$ and $X\mid U\sim \mathcal{N}(U, 5)$.} We consider a uniform prior on $\{\theta_1,\dots,\theta_{100}\}$. 
We apply standard rejection ABC as a baseline, proposing parameters from the prior. We sequentially update the proposal distribution using four schemes:  Algorithm \ref{alg:qe_abc} using each $\omega$ function in \eqref{eq:p_omega}, and Thompson Sampling (picking an arm with the largest sampled value from beta distributions). We report the acceptance rate, regret (using $\omega^3$) as well as the TV distance to the true posterior in Figure \ref{fig:regrets-finite}. Despite Thompson Sampling having the largest acceptance rate, it is unable to recreate the posterior as well as other approaches as the variance of the importance weights become large as it eventually starts to oversample the posterior mode. Rejection ABC has the lowest acceptance rate. Algorithm \ref{alg:qe_abc} represents a sweet spot between them. The three variants of Alg. \ref{alg:qe_abc} employ different choices of utility functions $\omega_{\bm p}$, that prioritize having a high acceptance rate at varying levels. Prioritizing a higher acceptance rate may increase performance in the shorter term, as it helps the algorithm to progress to a lower $\varepsilon$ level sooner, but may eventually result in importance weights with high variance as the nearly modal areas could be oversampled.
\subsection{Tree Regression (CART)}
\label{implementation:carthyperparameters}

To fit the trees from which the proposal distributions are made, we use a classification tree (in the style of CART) with a fixed maximum depth, number of leaves, and a minimum number of samples per leaf in order to prevent overfitting to the noisy data. For our experiments, we set the maximum number of leaves to be 1000 and the minimum number of samples per leaf to be 10.

\subsection{Details of Figure \ref{fig:trajectory-map}}\label{sec:alg3_MAP_details}

   We {apply} Algorithm \ref{alg:map_abc} with TS and Top-two TS (TTTS) on a toy example with a finite parameter space $\Theta=\{\theta_j\}_{j=1}^{50}$ equipped with a uniform prior.   In Figure \ref{fig:trajectory-map}, the upper left shows the $\varepsilon$-posterior values (normalized mean ABC acceptance rate for each parameter value). In this example, the difficulty arises from the fact that several other parameter values exhibit posterior mass nearly as high as that of the mode $\theta_{28}$. In the upper right plot {in Figure \ref{fig:trajectory-map}}, we estimate the posterior mode by the arm with maximal estimated posterior mass, and compare our TS procedures with rejection ABC (and maximization), and the upper confidence bound (UCB) algorithm \citep{auer2002finite} (treating the ABC acceptance as a binary reward), displaying the empirical probability of correctly identifying the mode over $10^4$ iterations. We observe that Algorithm \ref{alg:map_abc} can identify the true $\varepsilon$-posterior mode more frequently than other methods, as it obtains more information about the nearly-modal parameter values. The trajectory of the algorithm's choices in each round for one given sample path is displayed on the bottom {in Figure \ref{fig:trajectory-map}}. As the algorithm progresses, it proposes the parameter value with maximal $\varepsilon$-posterior density $(\theta_{28})$ with an increasing rate.

\subsection{Details of Figure \ref{fig:map-progression-dyadic-3split-20k-gauss2d}}\label{sec:MAP_fig_details}
In Figure \ref{fig:map-progression-dyadic-3split-20k-gauss2d}, we apply Algorithm \ref{alg:seqtreemap} using a dyadic partitioning on the bivariate Gaussian mixture toy example from \textbf{Example \ref{ex:2}}, with a total budget of $20000$ simulator evaluations. As we can see, the higher resolution areas in the partition concentrate near the posterior mode, and the algorithm can closely estimate the true posterior mode under a small budget. An analogous example where the partitioning is done via classification trees can be found in Figure \ref{fig:extra-cart} in Section \ref{sec:add_figures}, where we observe similar behavior. We compare performance on estimating the posterior mode in terms of squared error, taking the mean over 100 iterations. We compare against rejection ABC and SMC-ABC with kernel mode estimation, as well as the hierarchical online optimization algorithm (HOO) of \citet{bubeck2011x}. HOO is an algorithm designed for minimizing cumulative regret in bandit setting with a continuous arm space. The algorithm is a variant of the UCB algorithm that employs a dyadic partitioning using a binary tree on a measurable space. We set $\rho=\nu_1=1$, and use acceptance at a fixed $\varepsilon=6.0$ as a binary reward. We observe that MAP-Tree can estimate the posterior mode with more accuracy than its counterparts.} %A leaf node is split after every iteration. Using Algorithm \ref{alg:seqtreemap} with a dyadic partitioning implements a similar partitioning scheme.%, now used for posterior sampling instead of bandit optimization. 
%Instead of splitting the tree after each simulation, we refrain from splitting until each $\varepsilon$ batch completes to keep the computational cost feasible. The acquisition mechanism used to choose which leaf node in which to sample remains our regularized Thompson sampling style approach for posterior sampling, and becomes a standard or top-two Thompson sampling when the goal is shifted instead to MAP estimation.

\subsection{Details of the MNIST Experiment in Section \ref{sec:experiemnt}}\label{sec:MNIST_details}

For training an autoencoder on MNIST, a common choice for $d_{\bm\theta}$ is between 5 to 10. For us, to obtain a reference distribution through the standard rejection ABC with a reasonable number of simulations $T=10^8$, we choose the smallest dimension, that is, we consider $d_{\bm\theta}=5$ for the posterior sampling.

We train an autoencoder by using the Wasserstein autoencoder \cite{tolstikhin2017wasserstein} with noise $\sigma_0 = 0.05$. Note that an autoencoder involves two neural networks, a decoder $G$ and an encoder $Q$. We can apply $Q$ to each MNIST digit image $X\in [0,1]^{28 \times 28}$ so that $Q(X)\in \Theta$. Often, this space $\Theta$ (or collectively $\{Q(X)|X\text{ is in the training set}\}$) is called a latent space in the generative model literature. When training the autoencoder, we also add a label (digit) information during the optimization process similar to \cite{makhzani2015adversarial} so that the latent space forms a cluster for each digit number in $\Theta$ as in Figure \ref{fig:gen-map-resultster-label} (a). 

For visualization, we train iGLoMAP's dimensional reduction mapper that reduces the 5-dimensional parameter vectors ($\theta\in \Theta$) to 2-dimensional vectors. The mapper is a standard feedforward deep neural network. To save the training time, we train the network on 8,000 points of $Q(X)$, and then we map our generated $\tilde{\bm\theta}$ and $Q(X)$ as well for all 60,000 training images $(X)$ of MNIST.

\section{Additional Experiment Results and Details}\label{sec:additional_exp}

In this section, we present additional results from our method on various other likelihood-free inference models. We include a comparison between optimizing various optimality criteria, including $\omega^1$ and $\omega^2$ as well as a third utility function $\omega^3$. We define
\begin{equation}\label{eq:omega3}
    \omega^3_{\bm{p},\bm{\ell}^\varepsilon}(\bm{q}) = \log\left(\sum_{k=1}^K q_k \ell^\varepsilon_k\right) - D_{\mathrm{KL}}(\bm{q}\|\bm{p}),
\end{equation}
which depends not only on the posterior $\bm{p}$, but on the acceptance rates $\ell_k^\varepsilon$ in each bin at tolerance $\varepsilon$. This criterion seeks to regularize maximizing the log acceptance rate of $\bm{q}$ by ensuring that $\bm{q}$ cannot stray too far from $\bm{p}$. It performs very well initially due to a large acceptance rate, but later may begin to oversample the nearly modal areas once a sufficiently low tolerance has been reached. When using $\omega^3$ as the utility criterion in Algorithm \ref{alg:seqtreeabc} and Algorithm \ref{alg:qe_abc}, we compute $\ell_k^\varepsilon$ by $\eta_k^{(t)}$, our current model for the acceptance rate from region $k$.

\subsection{Heston Stochastic Volatility Model Simulation}%\label{sec:heston}

In quantitative finance, modeling the dynamics of asset prices is fundamental tool for risk management, trading, and investment strategies. The Heston model is a mathematical framework used to describe the evolution of financial market prices, particularly focusing on the stochastic volatility of an asset \citep{heston1993closed}. Introduced by Steven Heston in 1993, it allows for more realistic modeling of the volatility surface in options pricing compared to models that assume constant volatility, like the Black-Scholes model \citep{black1973pricing}. Bayesian analysis has been considered for posterior inference in stochastic volatility models, but has been traditionally challenging \citep{shephard2020statistical, fruhwirth2003bayesian, martin2019auxiliary}. Typical approaches to Bayesian inference in the Heston model involve an exact approach using MCMC \citep{fruhwirth2003bayesian} and related particle filtering and sequential importance sampling methods \citep{hore2010bayesian}.
\begin{figure}
    \centering
  \includegraphics[width=0.65\textwidth]{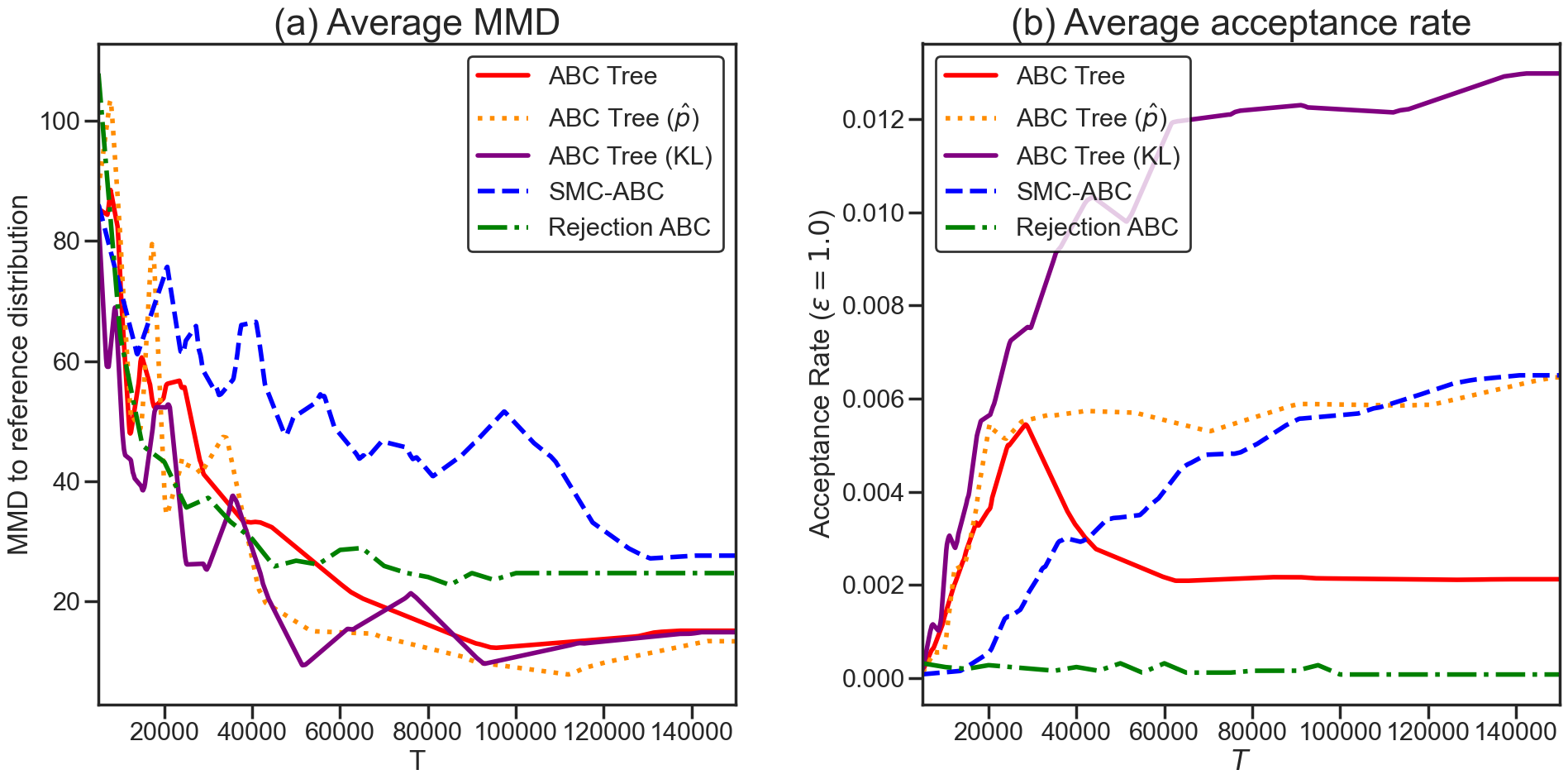}
  \includegraphics[width=0.32\linewidth]{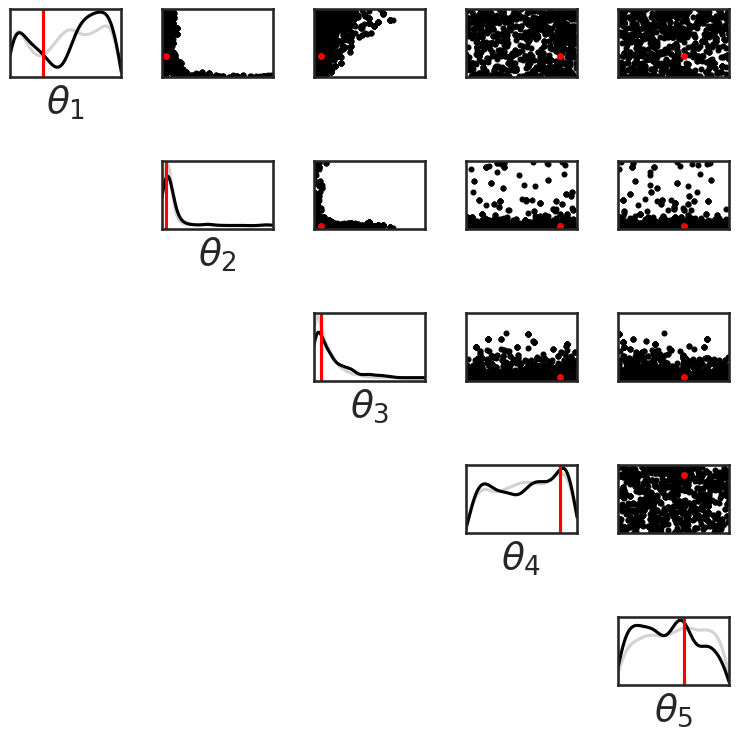}
    \caption{\textbf{Heston model results}. Left: Total variation distance and maximum mean discrepancies between posterior samples and a ``true'' rejection ABC posterior generated from $10^8$ simulations. Right: Posterior distribution obtained by Tree ABC on the Heston model after $10^5$ simulations}
    \label{fig:heston-results}
\end{figure}
The Heston model describes the dynamics of an asset's current price $S_t$ as
\begin{equation}
   \mathrm{d}S_t = \mu S_t\,\mathrm{d}t+\sqrt{\nu_t}S_t\,\mathrm{d}W^1_t ,
   \label{eq:price-dynamics}
\end{equation}
where $\mu$ denotes the risk-free interest rate, $\sqrt{\nu_t}$ denotes the volatility of the asset's price, and $W^1_t$ is a Wiener process governing stochastic fluctuations in price. The volatility $\sqrt{\nu_t}$ is assumed to follow an Ornstein-Uhlenbeck process $\mathrm{d}\sqrt{\nu_t}=-\xi\sqrt{\nu_t}\,\mathrm{d}t+\delta \,\mathrm{d}W^2_t$ where $W^2_t$ is another Wiener process. This may alternatively be written as the square-root process on the variance
\begin{equation}
   \mathrm{d}\nu_t = \kappa(\xi-\nu_t)\,\mathrm{d}t+\sigma\sqrt{\nu_t}\,\mathrm{d}W^2_t,
   \label{eq:vol-dynamics}
\end{equation}
where $W^1$ and $W^2$ have correlation $\rho$. 

%\subsubsection{Simulated Data}

We can simulate asset prices according to the Heston model by discretizing the dynamics \eqref{eq:price-dynamics} and \eqref{eq:vol-dynamics} in the above system of stochastic differential equations. The model is parameterized by $\bm{\theta}:=(\kappa,\xi,\nu_0,\rho,\sigma)\in\mathbb{R}^5$. We choose the following summary statistics of the simulated asset price time series: mean of log returns, mean of squared log returns, log standard deviation of log returns, log standard deviation of squared log returns, autocorrelation of squared log returns at lag 1, autocorrelation of squared log returns at lag 2, correlation of log returns and squared log returns. 
These summary statistics were shown to be effective for ABC in stochastic volatility models by \citet{martin2019auxiliary}.
We apply our ABC-Tree method to sample from the posterior of the Heston model. We consider a risk-free interest rate of $\mu=0.02$ and initial asset price $S_0=100$. As our observed data, draw a sample of 2000 discrete time steps each from a discretization of \eqref{eq:price-dynamics} and \eqref{eq:vol-dynamics}. The prior for $\bm{\theta}$ is given by $\pi(\bm\theta)=\mathcal{U}(0,10)\otimes\mathcal{U}(0,1)\otimes\mathcal{U}(0,1)\otimes\mathcal{U}(-1,1)\otimes\mathcal{U}(0,1).$  %$\pi(\bm\theta)=\mathcal{U}((0,10)\times(0,1)\times(0,1)\times(-1,1)\times(0,1)).$ 

The initial results are in Figure \ref{fig:heston-results}. We find that the sequential learning performed in ABC-Tree can allow for a more efficient proposal distribution to be learned faster. This leads to improved performance relative to typical ABC methods when under a budget of simulator evaluations. Our method trades off approximation accuracy of the posterior in exchange for the ability to sequentially learn from previous simulations. 
When the number of simulations becomes very large, any methods that progressively learn a proposal distribution would have also amassed enough data to have a good approximation of the posterior. In such cases, we posit that differences in performance relate more to the difference in approximation accuracy between the method for modeling the likelihood (or posterior), and in the optimality criterion by which the proposal distribution is calculated. For example, on smooth, unimodal densities, the KDE proposal of SMC-ABC may better approximate the optimal proposal, while the tree-based proposal used by ABC-Tree may better approximate the optimal proposal for problem instances with varying levels of local smoothness or posteriors resembling a step function. 
We also find empirically that Tree ABC appears to outperform SMC-ABC especially well when the parameter space is higher dimensional.

\subsection{Lotka-Volterra Model}

The Lotka-Volterra predator-prey model is commonly used in ecology to model the dynamics of the populations of two species interacting as predatory and prey. Letting $x$ denote the population of the prey and $y$ the population of predators, the growth rates of the two species are modeled by the ordinary differential equations $\frac{\mathrm{d}x}{\mathrm{d}t} = \theta_1 x - \theta_2 xy$ and $
\frac{\mathrm{d}y}{\mathrm{d}t} = \theta_3 xy - \theta_4 y$, which are parameterized by the four-dimensional vector $\boldsymbol\theta:=(\theta_1,\theta_2,\theta_3,\theta_4)^\top \in [0,1]^4$. Within the model, prey are assumed to reproduce at an exponential rate in absence of predators, and their rate of decrease is proportional to the product of their population with that of the predators. On the other hand, the population of predators  increases proportionally to this product, and decays exponentially. We simulate from the model as a Markov jump process, where the populations are recorded at discrete time steps with exponentially-distributed random wait times.

\begin{figure}[H]
    \centering
    \includegraphics[width=\linewidth]{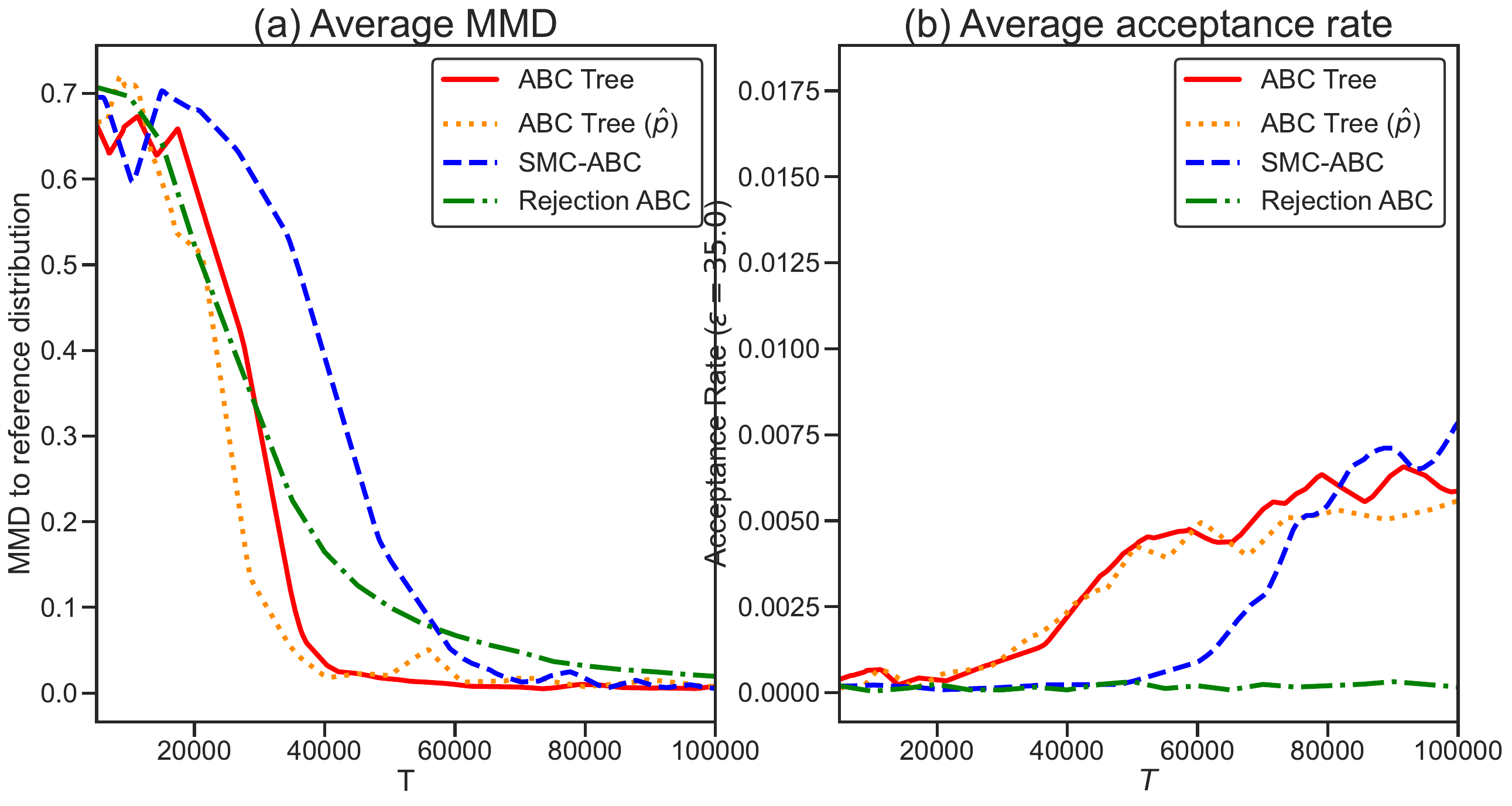}
    \caption{\textbf{ABC-Tree performance.} We evaluate maximum mean discrepancy to a `true' ABC posterior, obtained with a much larger budget. We also plot the negative log probability of true parameters. In all cases, higher is lower. For each method, we use a budget of $10^5$ samples, and multiply the tolerance by a factor of 0.9 after accepting 1000 parameter values. We perform 10 repetitions of each method. We plot performances of each run as partially transparent, and an interpolated and smoothed average in a bold opaque color.} 
    \label{fig:lotka-abc-results}
\end{figure}

\iffalse
\subsection{Linear-Nonlinear Encoding Model}

The linear-nonlinear encoding model has been used in neuroscience literature to model the response of neurons to various stimuli over time. We adapt the setup of \citet{gonccalves2020training}, modeling a spike of neural activity within each temporal bin as
\begin{equation}
    z_t\sim \mathrm{Ber}(\sigma(v_t^\top f + \beta))
\end{equation}
for $v_t$ a vector of white noise inputs relating the 8 previous time bins to the current one, $f$ a linear filter with $9$ parameters, a univariate scalar bias $\beta$. $\sigma(\cdot)$ represents the sigmoid function.

\begin{figure}[H]
    \centering
    \includegraphics[width=12cm]{figs/ln-sim.png}
    \caption{A comparison of ABC-Tree to SMC-ABC on the linear-nonlinear encoding model in terms of maximum mean discrepancy to the true ABC posterior (lower is better), and log probability of true parameters (lower is better). For each method, we use a total budget of $10^5$ samples, and multiply the tolerance by a factor of 0.9 after accepting 1000 parameter values. We perform 10 repetitions of each method.} 
    \label{fig:ln-abc-results}
\end{figure}

\fi

\subsection{Additional MNIST Experiment (Image Unmasking)}\label{sec:mnist_additional}

\begin{figure}
    \centering
    \includegraphics[width=0.95\linewidth]{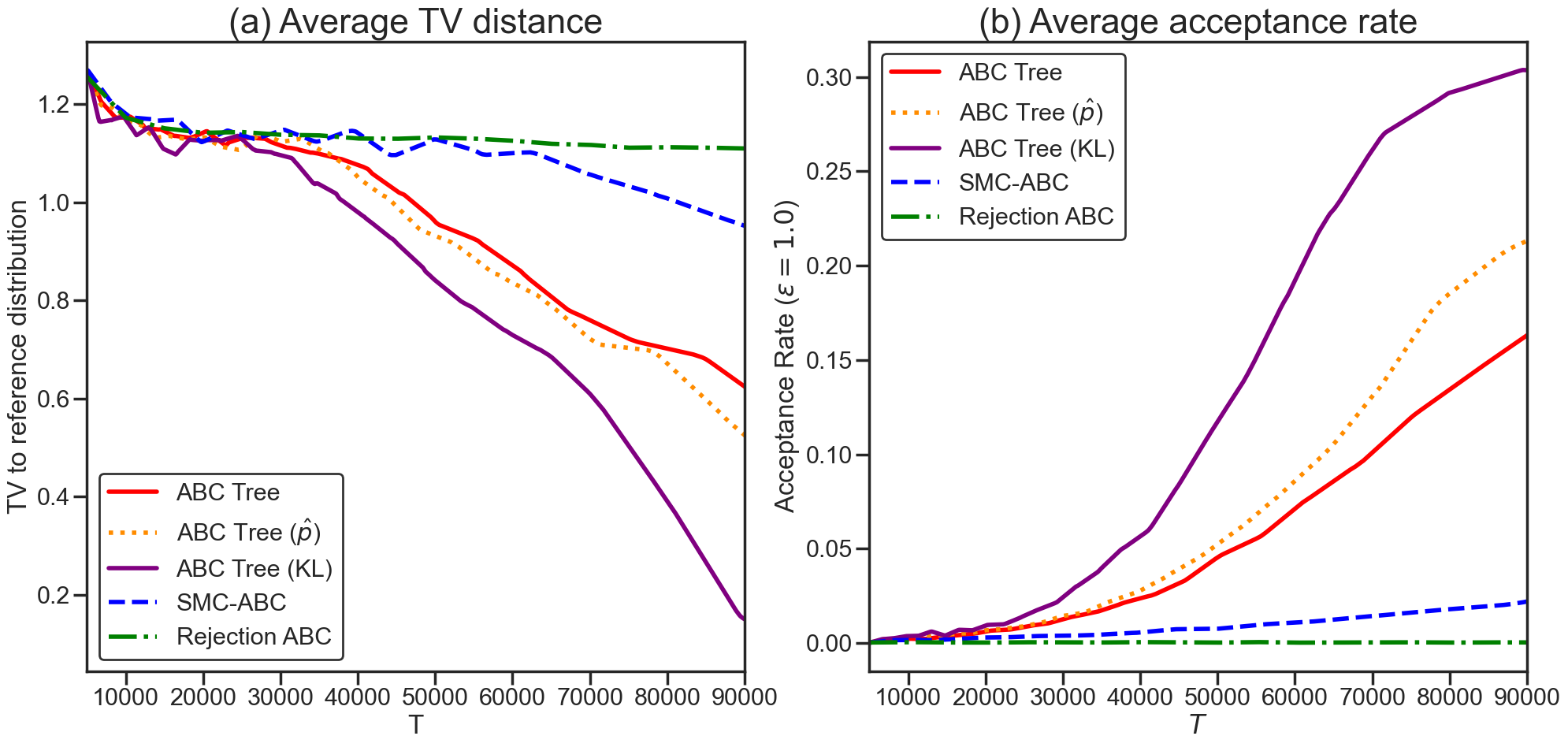}
    \caption{Comparison between the three $\omega_{\bm{p}}$ functions under the settings of Section \ref{sec:post_sampling}. Left: the total variation distance to the reference distribution. Right: the acceptance rate.}
    \label{fig:regul_important} 
\end{figure}

\begin{figure}
    \centering
    \includegraphics[width=0.9\linewidth]{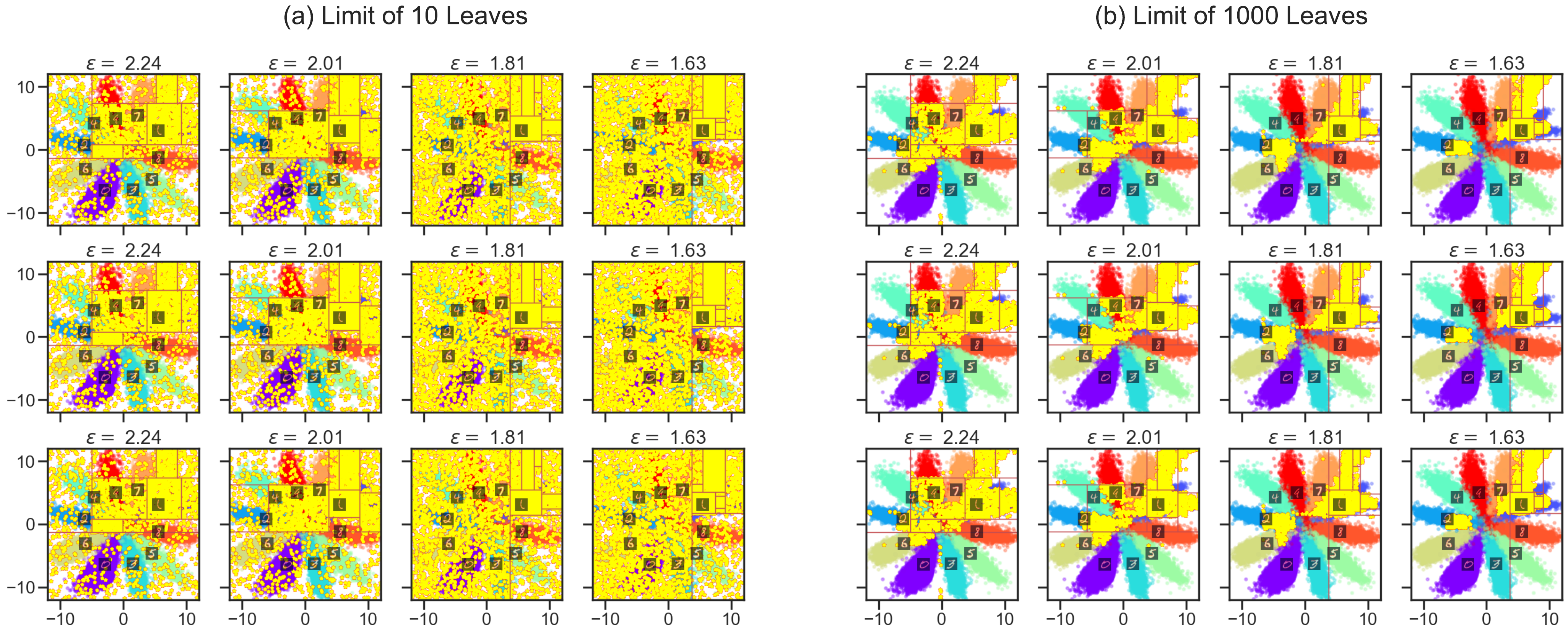}
    \caption{Comparison of ABC-Tree between when n\_leaves is 1000 and 10. (a) n\_leaves = 10. (b) n\_leaves = 1000. The digit is 1 as in Figure \ref{fig:masked} (b).}
    \label{fig:tree_size}
\end{figure}

We also include an additional demonstration for the MNIST image unmasking experiment, in which we include a comparison of various $\omega_{\bm{p}}$ functions. In Figure \ref{fig:regul_important}, which compares the convergence speed to the estimated true distribution across different optimality criteria. In this experiment, biasing the choice of criterion to favor a larger acceptance rate allows the algorithm to more quickly approximate the posterior at a small $\varepsilon$. In addition, to visualize the effect of different tree size limits, we consider $d_{\bm\theta}=2$ and impose a maximum of 10 and 1000 leaves, in Figure \ref{fig:tree_size}. We can see that the progression is similar, with the samples concentrated near the posterior modes for both cases. However, with a maximum of 10 leaves, the distinction between less important and moderately important areas is less clear, yielding some parameters spread across low posterior regions. On the contrary, the tree approximation with a maximum of $1000$ leaves is much sharper, with proposed parameters concentrated near multiple local modes. 

\subsection{Additional Table and Figures}\label{sec:add_figures}

Table \ref{tab:complexity} shows a comparison of the computational cost for various active learning ABC strategies, including SMC-ABC,
where $T$ is the total number of iterations across the batches, $d$ is the dimensionality of the parameter space, $K$ is the maximum number of hyper-rectangles in a partition, and $B$ is the number of acceptances before moving to the subsequent tolerance level. The cost when determining the partition depends on the method used, and the entry in the table reflects that of CART. ABC-Tree using either CART or dyadic partitioning has a lower cost than SMC-ABC  when $K<dB$.
\begin{table}[!h]
    \centering
    \begin{tabular}{|c|c|c|c|}
    \hline
         Strategy & Cost per tolerance level & Cost per simulation & Total cost\\
    \hline
         ABC-Tree (CART) & $O(Bd\log(B))$ & $O(K)$ & $O(Td\log(B)+TK)$ \\
    \hline
         ABC-Tree (Dyadic) & $O(K)$ & $O(K)$ & $O(TK)$ \\
    \hline
         SMC-ABC & $O(B^2)$ & $O(dB)$ & $O(TdB)$ \\
    \hline
         Rejection ABC & $0$ & $O(1)$ & $O(T)$ \\
    \hline
    \end{tabular}
    \caption{Comparison of computational cost of various sequential ABC strategies }
    \label{tab:complexity}
\end{table}

\begin{figure}[!h]
    \centering
    \includegraphics[width=\textwidth]{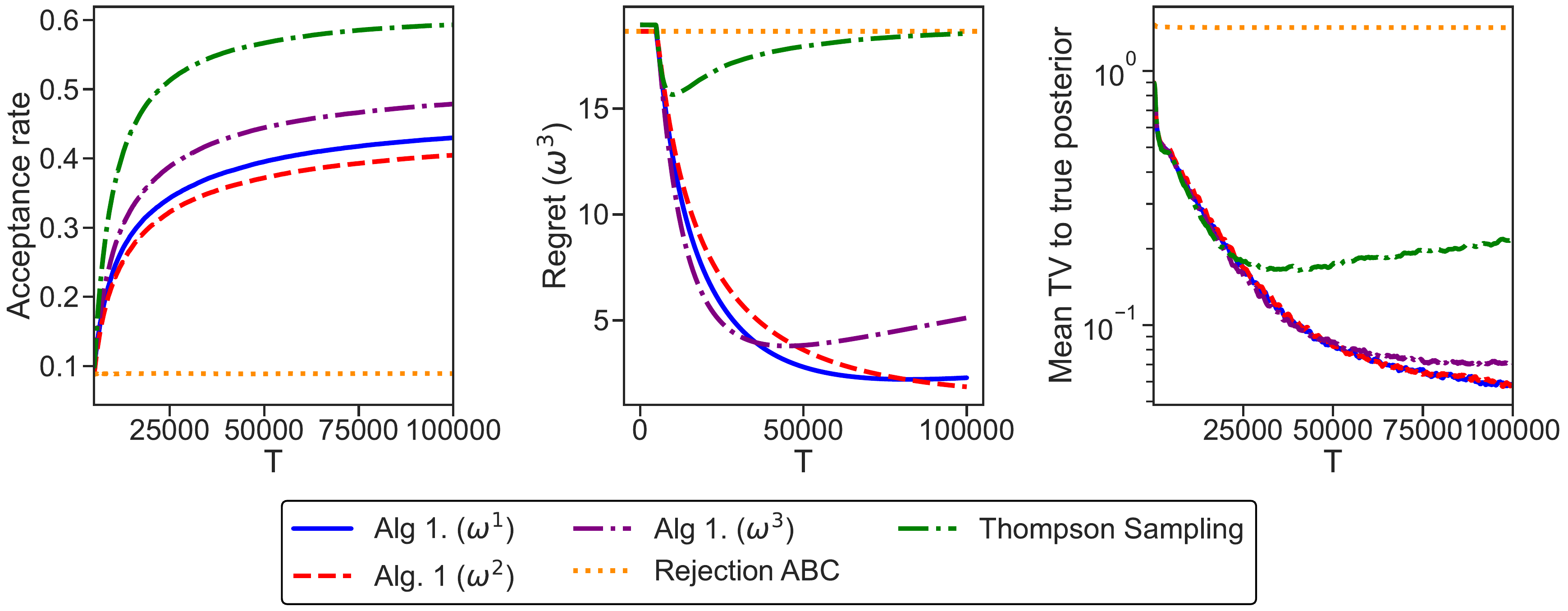}
    \caption{Comparing different adaptive proposal criteria for ABC in a finite parameter space, using the same setting as Example 1 (Figure \ref{fig:regrets-finite}), but with a third utility function $\omega^3$ defined in \eqref{eq:omega3}. ABC-Tree with $\omega^3$ performs superior initially due to a large acceptance rate, but later begins to oversample the nearly modal areas.}
    \label{fig:comp-with-omega3}
\end{figure}

\begin{figure}[!h]
    \centering
    \includegraphics[width=\textwidth]{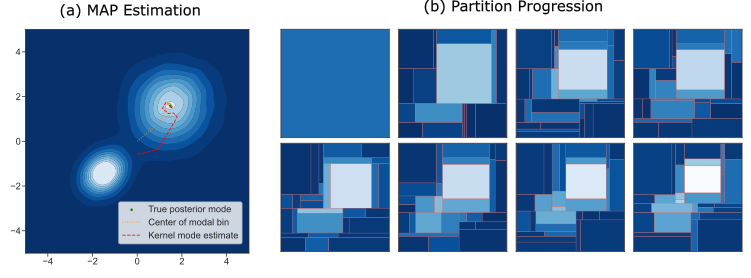}
    \caption{Another version of Figure \ref{fig:map-progression-dyadic-3split-20k-gauss2d} where CART is used for partitioning}
    \label{fig:extra-cart}
\end{figure}

\begin{figure}[!h]
    \centering
    \includegraphics[width=\linewidth]{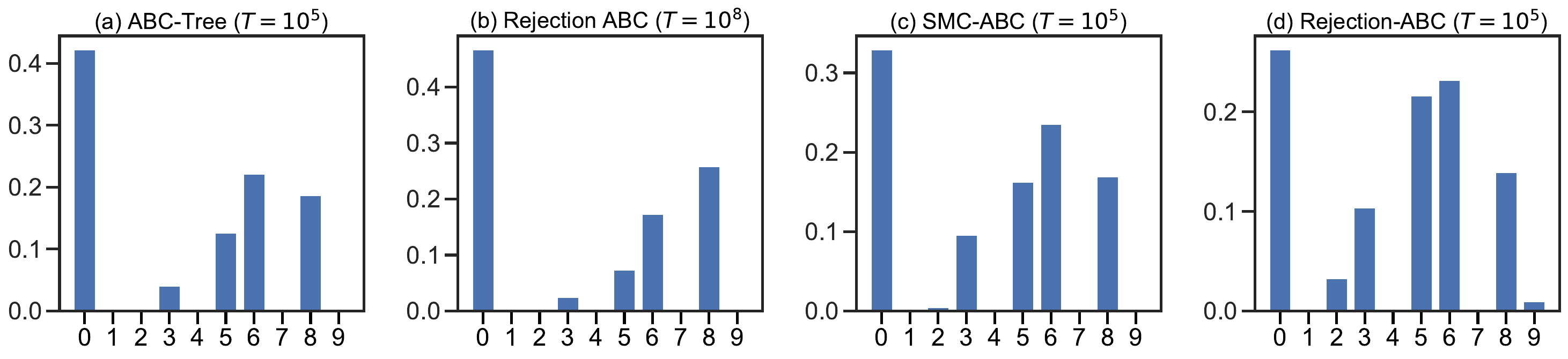}
    \caption{The estimated categorical distribution on digits. The estimation by Tree-ABC with $T=10^5$ highly resembles that of rejection ABC with $T=10^8$.}
    \label{fig:wanted}
\end{figure}

\begin{figure}[!h]
    \centering
    \includegraphics[width=\linewidth, height=8cm]{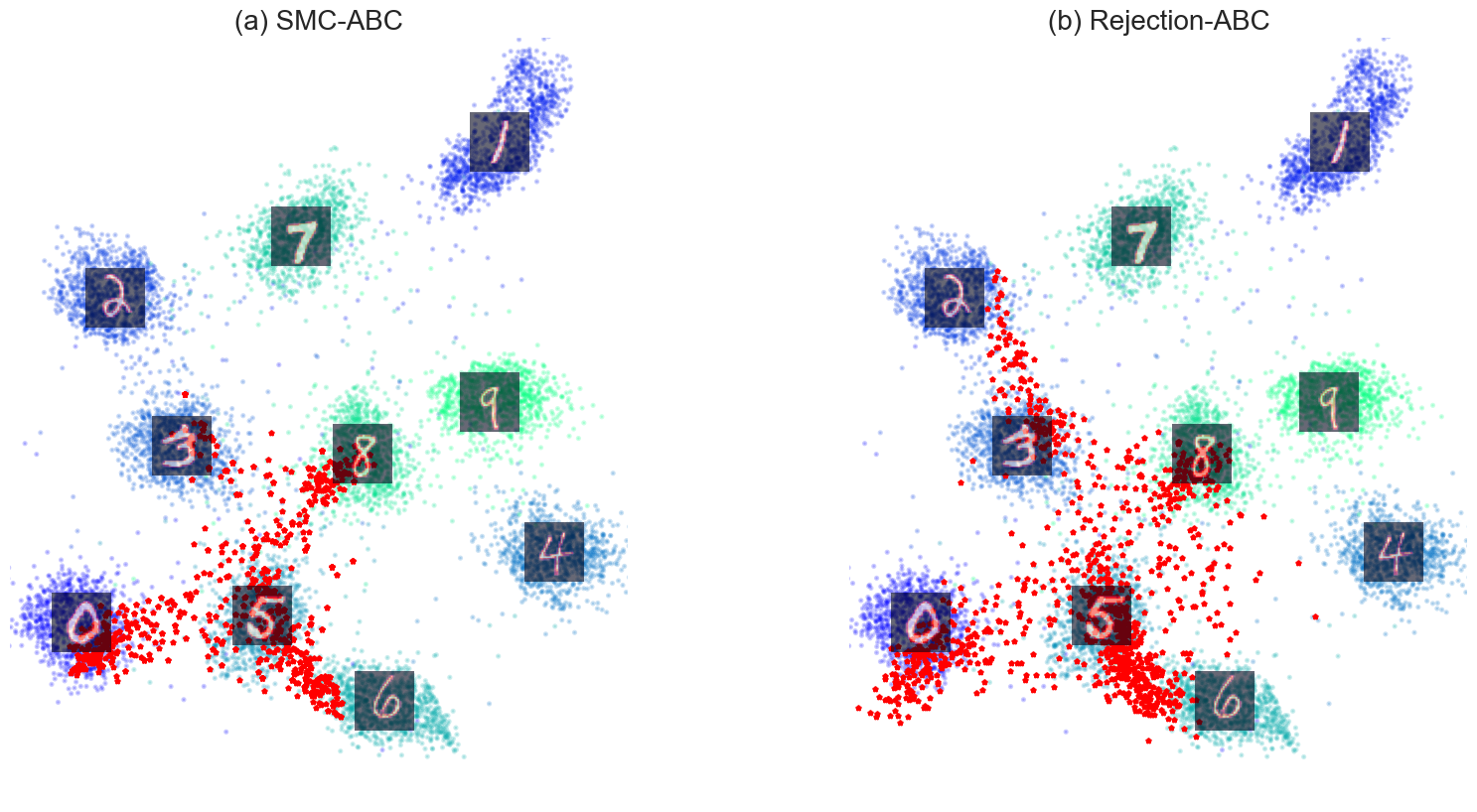}
    \caption{The iGLoMAP dimensional reduction of the baselines. (a) Rejection ABC. (b) SMC-ABC. SMC-ABC shows similar results.}
    \label{fig:others_based}
\end{figure}

\begin{figure}
    \centering
    \includegraphics[width=0.93\linewidth]{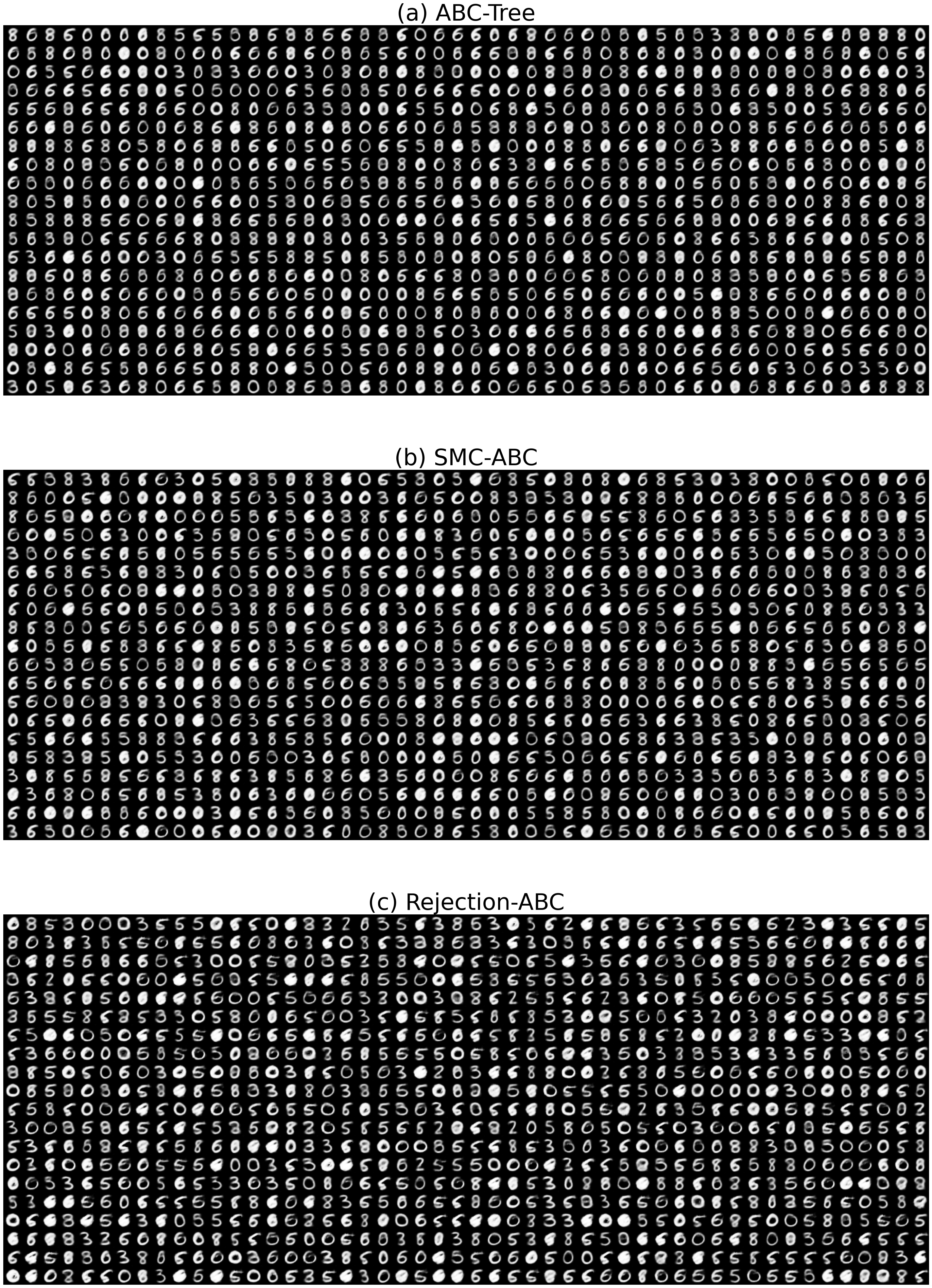}
    \caption{The deep generative model simulation. The 1000 samples from each method.}
    \label{fig:genmod1}
\end{figure}

\begin{figure}
\begin{center}
     \includegraphics[width=0.8\textwidth]{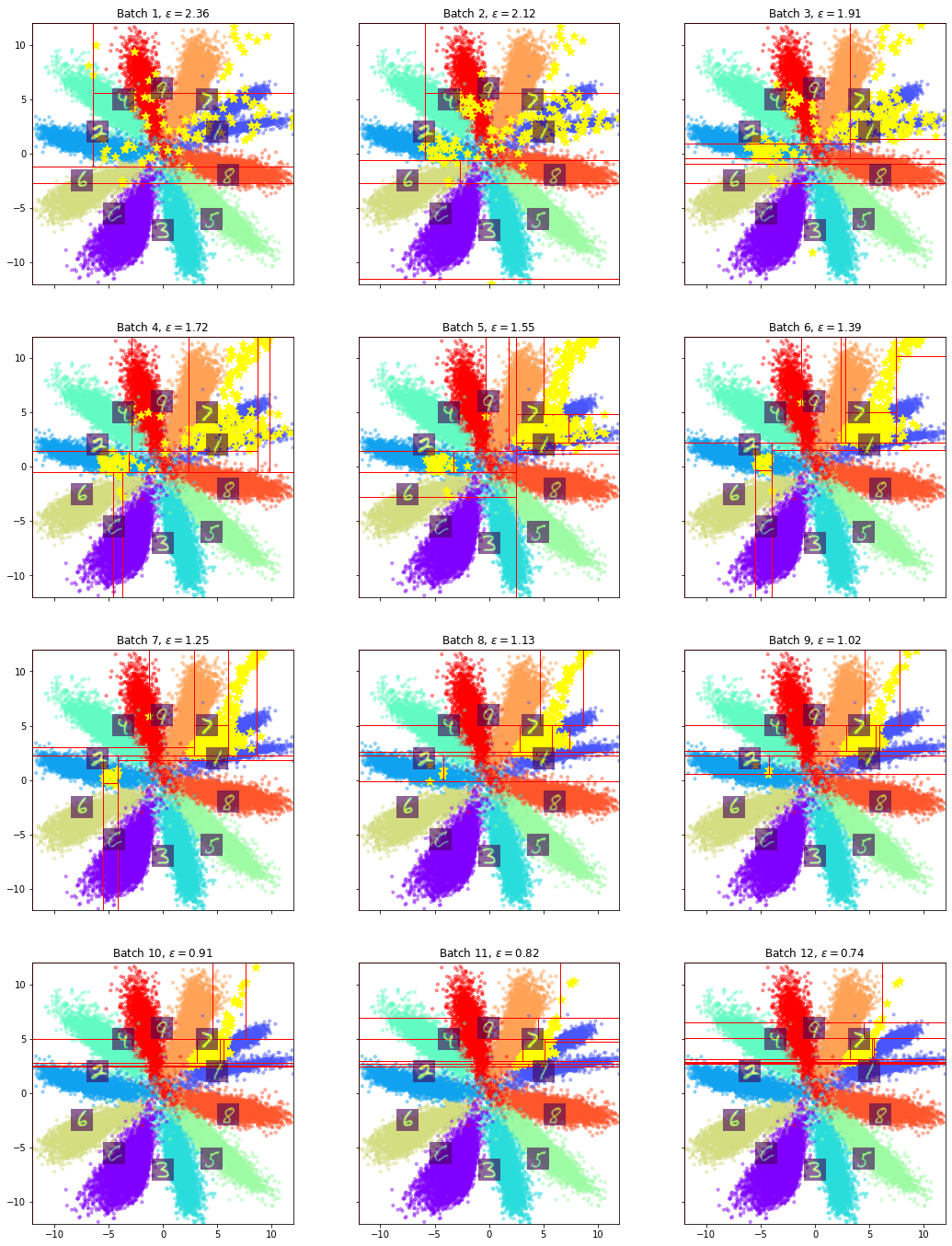}
\end{center}
    \caption{The progression for the MAP experiment on generative model corresponding to Figure \ref{fig:gen-map-resultster-label}.}
    \label{fig:mapgen1map}
\end{figure}

%\begin{figure}[H]
%    \centering
%    \caption{Evolution of STTTS MAP estimate as number of simulations increase for the Lotka-Volterra experiment. Here we plot the first two components of the four-dimensional parameter space. The direction of the path is indicated by the arrows, and the path represents the estimate of the posterior mode after $s$ simulations have been done, with $s$ ranging from $100$ to $20000$. The background shows the $\varepsilon_T$-marginal posterior density, showing an approximation of the objective function that STTTS is optimizing. This density is obtained from a kernel density estimate of rejection ABC with a large number of samples. As we can see, STTTS is able to move towards an area where the objective function is high using few simulations. {\cb Sean, here I do not see any figure yet (JK). If you are not planning to work on it, we can put it in Appendix for a while (Additional experiment results section)}}
%    \label{fig:sttts_lotka}
%\end{figure}

\clearpage

\section{Algorithms}\label{sec:ts}

\begin{algorithm}
\caption{Inner Loop: Thompson Sampling for MAP-Tree}\label{alg:map_abc}
\spacingset{1.2}
\begin{algorithmic}[1]
\Statex \textbf{Input:} Prior $\bm\pi(\cdot)$, observed data $\bm{X}$, tolerance $\varepsilon$, partition $\bm\Omega=\{\Omega_k\}_{k=1}^K$, budget $T_\mathrm{max}$,
\Statex {\color{white}\textbf{Input:}}   acceptance quota $B$ and    $\bm{\alpha}=(\alpha_1,\dots, \alpha_K)'$ and
$\bm{\beta}=(\beta_1,\dots,\beta_{K})'$
\State \textbf{Initialize} $t=0$ and $\alpha_k^{(0)}=\alpha_k$ and $\beta_k^{(0)}=\beta_k$
\While{$N \leq B$ and $t \leq T_\mathrm{max}$}
\State $t\gets t+1$
% \State Set $\hat{{\bm {p}}}^{(t)} \propto \bm\alpha^{(t-1)}/(\bm\alpha^{(t-1)}+\bm\beta^{(t-1)})$
\State{$~\bm{\Pi}^{(t)}: = \bigotimes_{i=1}^K  \mathrm{Beta}\left(\alpha_i^{(t-1)}, \beta_i^{(t-1)}\right)$}
\State\label{tmp1}{(Binned Likelihood)} Sample $\bm\eta^{(t)} \sim \bm{\Pi}^{(t)}$
\State \label{ln:explo}(Binned Posterior) $\hat{{ {p}}}_k^{(t)}\propto \eta^{(t)}_k\pi(\Omega_k)/|\Omega_k)|$ for $k=1,\ldots,K$
\State\label{tmp2}{(Exploitation)} Choose arm ${ I}=I^{(t)}  \gets \argmax_i \hat{p}_i^{(t)}$
\State { \textbf{If} Top-Two:} 
\State { $~~~~~~~~$ Generate $Z\sim {Bernoulli}(b)$}
\State { $~~~~~~~~$ \textbf{If} $Z$ is not 1:}
\State \label{tmp3}$~~~~~~~~$ $~~~~~~~~$ Resample $\bm{\eta}\sim\bm{\Pi}^{(t)}$ and set $\hat{{ {p}}}_k\propto \eta_k\pi(\Omega_k)/|\pi(\Omega_k)$ until $\argmax_i\hat{p}_i \neq I$ 
\State $~~~~~~~~$ $~~~~~~~~$ Set $I=I^{(t)}\gets \argmax_i\hat{p}_i$
\State {(Simulate data)} $\wt {\bm X}_{\bm\theta^{(t)}}\sim L(\cdot |\bm\theta^{(t)})$ where $\bm\theta^{(t)}\sim \pi\left(\bm\theta|\Omega_{I^{(t)}}\right)$
\State {(ABC reward)} $Y^{(t)} \gets {\rm ABC}(\bm X,\wt{\bm X}_{\bm\theta^{(t)}},\varepsilon)$, where $d^{(t)} \gets d(s(\bm{X}),s(\wt{\bm{X}}_{\bm{\theta}^{(t)}}))$, 
\State \label{ln:info}{(Posterior update)} $(\alpha_{I }^{(t+1)}, \beta_{I }^{(t+1)})\gets (\alpha_{I }^{(t)}+Y^{(t)}, \beta_{I }^{(t)}+1-Y^{(t)})$
\State \textbf{if} {$Y^{(t)}=1$} \textbf{then} $N\gets N+1$ 
\EndWhile
\State $T\gets t$
\State \Return $\mathcal{D}(\varepsilon)=\{(\bm\theta^{(t)},\wt{\bm{X}}_{\bm{\theta}^{(t)}}, Y^{(t)}(\varepsilon)\}_{t=1}^T, T$ %accepted parameters and importance weights $\{(\theta_{I^{(m)}},u^{(m)}):Y^{(m)}=1\}$ 
\end{algorithmic}
\end{algorithm}

\begin{algorithm}[!h]
\caption{MAP-Tree}\label{alg:seqtreemap}
\spacingset{1.2}
\begin{algorithmic}[1]
\Statex \textbf{Input:} Prior $\bm\pi(\cdot)$, observed data $\bm{X}$, initial tolerance $\varepsilon_1$,
\Statex {\color{white}\textbf{Input:}} Number of rounds $S_{max}$,  tolerance progression rule $g(\varepsilon)$ % number of acceptances per batch , initial partition $\bm \Omega=\{\Omega_1,\ldots,\Omega_K\}$%, {\cb ABC acceptance operator ${\rm ABC}(\bm X,\bm\theta,\varepsilon)$}
\State Collect $\Omega_{S_{max}},\mathcal{D}$ from Algorithm \ref{alg:seqtreeabc} with Algorithm \ref{alg:map_abc} as the inner loop
\State \textbf{MAP estimation}
\State $~~~~$ \textbf{If} KDE \textbf{then} $\hat{\bm\theta}^{\rm(MAP)}\gets$ $\argmax_{\bm\theta\in\Theta}{\rm KDE}(\{\bm{\theta}^{(t)}:Y^{(t)}(\varepsilon_{S_{max}})=1\}_{\bm{\theta}^{(t)\in\mathcal{D}}})$
\State $~~~~$ \textbf{Else if} Bin Center \textbf{then} set $\hat{\bm\theta}^{\rm(MAP)}$ to the center of the bin with largest estimated average posterior
\State \Return $\hat{\bm\theta}^{\rm(MAP)}$ 
\end{algorithmic}
\end{algorithm}

\begin{algorithm}[!h]
\caption{Thompson Sampling in multi-armed bandits with Bernoulli rewards}\label{alg:ts}
\spacingset{1.1}
\begin{algorithmic}[1]
\Require budget $T$, problem instance $\bm{\mu}$, $b\in(0,1)$ (if Top-Two)
\State Initialize $\alpha^{(0)}_{i}=\beta^{(0)}_i=1$ for $i=1,\ldots,K$.
\For{$t=1,\ldots,T$}
\State{$~\bm{\Pi}^{(t)}: = \bigotimes_{i=1}^K  \mathrm{Beta}\left(\alpha_i^{(t-1)}, \beta_i^{(t-1)}\right)$}
\State\label{tmp1}{ (exploration)} Sample $\bm\eta^{(t)} \sim \bm{\Pi}^{(t)}$
\State \label{tmp2}{ (exploitation)} Choose arm ${ I}=I^{(t)}  \gets \argmax_i \eta_i^{(t)}$
\State { \textbf{If} Top-Two:} 
\State { $~~~~~~~~$ Generate $Z\sim {Bernoulli}(b)$}
\State { $~~~~~~~~$ \textbf{If} $Z$ is not 1:}
\State \label{tmp3}$~~~~~~~~$ $~~~~~~~~$ Resample $\bm{\eta}\sim\bm{\Pi}^{(t)}$ until $\argmax_i\eta_i \neq I$ and set $I\gets \argmax_i\eta_i$
\State { (arm pulling)} { Observe reward $Y^{(t)}\sim \mathrm{Bernoulli}(\mu_{{ I} })$ }
\State { (information update)} Update $(\alpha_{{ I} }^{(t)}, \beta_{{ I} }^{(t)})\gets (\alpha_{{ I} }^{(t-1)}+Y^{(t)}, \beta_{{ I} }^{(t-1)}+1-Y^{(t)})$
\EndFor
\end{algorithmic}
\end{algorithm}

\end{document}